\documentclass[twoside]{article}
\def\vx{{\mathbf{x}}}
\def\vy{{\mathbf{y}}}
\def\vz{{\mathbf{z}}}
\def\vu{{\mathbf{u}}}
\def\vv{{\mathbf{v}}}
\def\vp{{\mathbf{p}}}
\def\vA{{\mathbf{A}}}
\def\vB{{\mathbf{B}}}
\def\vH{{\mathbf{H}}}

\def\vI{{\mathbf{I}}}

\def\vX{{\mathbf{X}}}
\def\vY{{\mathbf{Y}}}
\def\mA{{\mathbf{A}}}
\def\mU{{\mathbf{U}}}
\def\mG{{\mathbf{G}}}
\def\vmu{{\boldsymbol{\mu}}}
\def\vomega{{\mathbf{\omega}}}
\newcommand{\mD}{\mathbf{D}}
\def\mLambda{{\mathbf{\Lambda}}}

\newtheorem{innerthm}{Theorem}

\def\R{{\mathbb{R}}}
\def\vomega{{\boldsymbol{\omega}}}
\def\dvom{{\dot{\boldsymbol{\omega}}}}
\usepackage{amsmath,amsfonts,amssymb,amsthm, mathtools, hyperref}
\usepackage{tikz}
\usetikzlibrary{arrows.meta,calc}
\usepackage{mathtools}
\usepackage{xcolor}
\usepackage[ruled,vlined,linesnumbered,noend]{algorithm2e}
\usepackage[colorinlistoftodos]{todonotes}
\usepackage{placeins}

\newtheorem{corollary}{Corollary}
\newtheorem{lemma}{Lemma}
\newtheorem{theorem}{Theorem}
\theoremstyle{definition}
\newtheorem{definition}{Definition}[section]

\newtheorem*{lemmarep}{Lemma}
\newtheorem*{theoremrep}{Theorem}

\usepackage[accepted]{aistats2026}

\usepackage{xcolor}        
\definecolor{DarkGreen}{rgb}{0.1,0.5,0.1}
\definecolor{DarkRed}{rgb}{0.5,0.1,0.1}
\definecolor{DarkBlue}{rgb}{0.1,0.1,0.5}
\usepackage{hyperref}      
\hypersetup{
    unicode=false,          
    pdftoolbar=true,        
    pdfmenubar=true,        
    pdffitwindow=false,     
    pdfnewwindow=true,      
    colorlinks=true,      
    linkcolor=DarkBlue,          
    citecolor=DarkGreen,       
    filecolor=DarkRed,    
    urlcolor=DarkBlue,         
    pdftitle={},
    pdfauthor={},
}

\usepackage[inline, shortlabels]{enumitem}  
\usepackage{url}

\usepackage{subcaption}

\usepackage[round]{natbib}

\begin{document}
\newtheorem{proposition}{Proposition}
\twocolumn[

\aistatstitle{Frequency-Based Hyperparameter Selection in Games}

\aistatsauthor{ Aniket Sanyal \And Baraah A.M. Sidahmed\footnotemark[1] \And  Rebekka Burkholz\footnotemark[1] \And Tatjana Chavdarova }

\aistatsaddress{\hspace{-60pt}TU Munich\And  $^{1}$CISPA Helmholtz Center for Information Security \And \hspace{50pt}TU Wien } ]
\begin{abstract}

Learning in smooth games fundamentally differs from standard minimization due to rotational dynamics, which invalidate classical hyperparameter tuning strategies. Despite their practical importance, effective methods for tuning in games remain underexplored. A notable example is LookAhead (LA), which achieves strong empirical performance but introduces additional parameters that critically influence performance. We propose a principled approach to hyperparameter selection in games by leveraging frequency estimation of oscillatory dynamics. Specifically, we analyze oscillations both in continuous-time trajectories and through the spectrum of the discrete dynamics in the associated frequency-based space. Building on this analysis, we introduce \emph{Modal LookAhead (MoLA)}, an extension of LA that selects the hyperparameters adaptively to a given problem. We provide convergence guarantees and demonstrate in experiments that MoLA accelerates training in both purely rotational games and mixed regimes, all with minimal computational overhead.
\end{abstract}

\section{Introduction}\label{sec:intro}
Saddle-point optimization, and more generally \emph{Variational Inequalities (VIs)}~\citep{stampacchia1964formes,facchinei2003finite}, play a central role in modern machine learning. They underpin applications such as generative adversarial networks (GANs)~\citep{article}, adversarial training~\citep{goodfellow2015explainingharnessingadversarialexamples,madry2019deeplearningmodelsresistant}, and multi-agent reinforcement learning~\citep{shapley1953,littman1994,bertsekas2021rollout}, all of which can be viewed as multiplayer games seeking a Nash equilibrium.

However, classical optimizers like \emph{Gradient Descent (GD)} are ill-suited for games. Unlike minimization problems, where gradients point directly toward the solution, games exhibit \emph{rotational dynamics}: updates cycle around the equilibrium rather than converging to it. As a result, GD and similar methods diverge in these settings. While specialized algorithms~\citep[\textit{e.g.}][]{balduzzi2018mechanicsnplayerdifferentiablegames} can handle purely rotational games, they often perform poorly in settings that also contain a potential (minimization-like) component.
\emph{LookAhead (LA)}~\citep{zhang2019LookAheadoptimizerksteps} offers a simple and efficient alternative. Originally proposed as a wrapper for minimization, it was later adapted to games~\citep{chavdarova2021lamm}: by averaging the trajectory of a base optimizer over $k$ steps, LA can stabilize otherwise divergent dynamics such as GD, at essentially no extra cost. LA has achieved strong empirical results, including state-of-the-art performance in GAN training~\citep{chavdarova2021lamm} and recent use in competitive reinforcement learning~\citep{sidahmed2025}. Despite this practical success, LA still lacks convergence guarantees in fundamental settings such as monotone and Lipschitz VIs (see Section~\ref{sec:prelim}).

LA introduces two hyperparameters: the averaging frequency $k$ and the interpolation weight $\alpha\in(0,1)$. While many $(k,\alpha)$ choices outperform the base optimizer, finding the optimal combination remains difficult. Existing approaches typically tune one parameter at a time~\citep{ha2022LookAhead}, which only works in simple cases. Two central questions remain open:
\begin{center}
\emph{
(1) Does LookAhead converge for monotone and Lipschitz operators?  \\
(2) How should we select $(k,\alpha)$ to guarantee convergence and ideally speed it up?
}
\end{center}

Our work addresses these questions from a new angle. The key to our approach lies in that game dynamics are rotational because the Jacobian of the updates is generally non-symmetric. A simple example is the Bilinear Game $\min_x \max_y xy$, where the iterates of $(x,y)$ trace cycles around the equilibrium (the origin) instead of converging directly. In particular, such rotations 
\begin{center}
\vspace{-.8em}
    \emph{suggest a natural frequency-domain interpretation}. 
\end{center}
\vspace{-.8em}
Indeed, beyond studying iterates in time, one can analyze them in a \emph{dual space} where oscillations are represented by amplitude and phase. In the bilinear case, $x$ and $y$ each follow sinusoidal trajectories with a fixed phase shift. More generally, by applying the $z$-transform, the discrete-time dynamics can be decomposed into exponential modes. Each point in the $z$-plane corresponds to a possible exponential trajectory: the radius indicates growth or decay, and the angle corresponds to oscillation frequency. Poles in this space reveal the modes that the system supports.

This \emph{modal perspective} allows us to reason directly about the stability and convergence of LookAhead. We also show that it complements continuous-time, \emph{pole}-based analyses, which study stability by examining the location of poles (roots of the transfer function) in the frequency domain; see Section~\ref{sec:laplace-based_analysis}. Unlike our considered pole-based approach, modal analysis is exact at the algorithmic level: it avoids discretization errors and provides computable, finite-step admissibility regions for $(k,\alpha,\gamma)$, where $\gamma$ is the step size of the base optimizer in LA. In this work, we develop the modal framework, prove new convergence guarantees, and propose principled hyperparameter selection rules tailored to the rotational structure of a given game.

\paragraph{Contributions.}
\begin{enumerate}[leftmargin=*,itemsep=0em,topsep=0em]
    \item We introduce a novel framework for analyzing and selecting hyperparameters in game dynamics, motivated by the rotational (oscillatory) nature of such problems. This viewpoint is orthogonal to classical approaches developed for minimization.
    \item Building on this framework, we obtain two main algorithmic contributions:
    \begin{enumerate}
        \item We provide the first convergence guarantee for standard LookAhead (with fixed $k$ and $\alpha$) applied to the gap function of the average iterate, matching the known lower bound.
        \item We propose a hyperparameter selection principle based on maximizing system stability. This yields a new variant, \emph{Modal LookAhead} (MoLA), which selects $k$ and $\alpha$ once at initialization and keeps them fixed throughout.
    \end{enumerate}
    \item Through numerical experiments on bilinear and strongly convex--strongly concave games, we demonstrate that MoLA significantly outperforms other variational inequality methods, including LookAhead tuned over a large random hyperparameter search.
\end{enumerate}

\noindent\textbf{Additional insights.} 
Along the way, we also: 
\begin{itemize}[leftmargin=*,itemsep=0em,topsep=0em]
    \item clarify the role of hyperparameter selection in games via a modal analysis perspective, which connects naturally to classical operator-theoretic contractiveness. For certain classes, the modal view coincides with the spectral characterization of contractive operators (see Section~\ref{sec:mode}).
    \item relate modal stability analysis in the frequency domain (for tuning $k,\alpha$) to continuous-time dynamics using high-resolution differential equations (HRDEs). For Bilinear Games, this recovers exactly the conclusion of the discrete-time analysis.
    \item derive the $\mathcal{O}(\gamma)$-HRDE of LookAhead for general $(k,\alpha)$, extending prior work \citep{chavdarova2023hrdes} which considered only $k \in \{2,3\}$.
\end{itemize}

\paragraph{Related works.}
LookAhead (LA) was first introduced in minimization tasks, where it improved test accuracy~\citep{zhang2019LookAheadoptimizerksteps}. Later, it was extended to games, where its averaging mechanism was shown to effectively dampen rotational dynamics: in particular, it enables gradient descent, which would otherwise diverge, to converge in Bilinear Games~\citep{chavdarova2021lamm}.~\citet{chavdarova2021lamm} also showed how LA manipulates the spectral radius of the Jacobian of the base optimizer, and proved its convergence provided the latter converges.
Building on this, \citet{ha2022LookAhead} proved its local convergence in smooth games and established global convergence for Bilinear Games. \citet{chavdarova2023hrdes} observed that standard ordinary differential equation (ODE) derivations of VI methods collapse to the same form, and proposed refined continuous-time models of LA with $k=2,3$, termed \emph{high-resolution differential equations} (HRDEs). 
More recently, \citet{sidahmed2025} also used LA for multi-agent reinforcement learning. Conceptually, LA in games is also related to the classical \emph{Halpern Iteration}, or broadly \emph{anchoring} techniques~\citep{halpern1967,diakonikolas2020halpern,ryu2019ode}, which apply a similar averaging, but always anchor one point to the initial one.

Closest to our approach,
\citet{Anagnostides2021FrequencyDomainRO} introduced a frequency-domain framework to analyze the convergence of algorithms that incorporate historical information, such as past gradients. They employed a bilateral $Z$-transform to study the convergence of algorithms such as the \emph{Optimistic Gradient Descent}~\citep{popov1980}. It simplifies convergence analysis and establishes linear convergence under the co-coercive and monotone operators. 
Similar to their approach, which ensures convergence by restricting the poles to lie within the unit circle in the $Z$-plane, part of our work (Section~\ref{sec:laplace-based_analysis}) imposes conditions on the poles in the complex frequency domain instead.

\section{Preliminaries}\label{sec:prelim}

\noindent\textbf{Notation.}
Scalars are denoted by lowercase letters, vectors and matrices by boldface. 
The convolution operator is written as $*$ (see Appendix~\ref{app:background}). The complex numbers are written as $a + i\,b$, where $i  =\sqrt{-1}$.
We introduce only the essential definitions here; further details appear in Appendix~\ref{app:background}.

\begin{definition}[Variational Inequality (VI)]
Let $D \subseteq \mathbb{R}^d$ be nonempty, closed, and convex, and let $F\colon D \to \mathbb{R}^d$.
The \emph{variational inequality problem} $\mathrm{VI}(F,D)$ is to find $\vz^\star \in D$ such that
\vspace{-.5em}
\begin{equation}\label{eq:VI} \tag{VI}
  \langle F(\vz^\star),\, \vz - \vz^\star \rangle \;\ge\; 0 
  \qquad \text{for all } \vz \in D .
\end{equation}
\end{definition}

\begin{definition}[Monotone Operator]
Mapping $F\colon ( \mathbb{R}^d \supseteq D) \to \mathbb{R}^d$ 
is \emph{monotone} on $D$ if
\vspace{-.5em}
\begin{equation}\label{eq:monotone}\tag{Mo}
  \langle F(\vz) - F(\vz'),\, \vz - \vz' \rangle \;\ge\; 0 
  \qquad \text{for all } \vz, \vz' \in D \,.
\end{equation}
\end{definition}

\begin{definition}[Lipschitz Operator] Mapping $F\colon ( \mathbb{R}^d \supseteq D) \to \mathbb{R}^d$ is \emph{$L$-Lipschitz} on $D$, with $L \ge 0$ if
\vspace{-.5em}
\begin{equation}\label{eq:lipschitz}\tag{Lip}
  \|F(\vz) - F(\vz')\| \le L \,\|\vz - \vz'\| 
  \qquad \text{for all } \vz,\vz' {\in} D \,.
\end{equation}
\vspace{-.5em}
\end{definition}

\begin{definition}[Nonexpansive operator]
A mapping $F\colon ( \mathbb{R}^d \supseteq D) \to \mathbb{R}^d$ is \emph{nonexpansive} on $D$ if
\vspace{-.5em}
\begin{equation}\label{eq:nonexpansive} \tag{NExp}
    \|F(\vz)-F(\vz')\| \;\le\; \|\vz-\vz'\| \qquad \forall\,\vz,\vz' \in \mathbb{R}^d.
\end{equation}
\vspace{-.5em}
Equivalently, its Lipschitz constant satisfies $L\leq1$:
\vspace{-.3em}
\begin{equation}\label{eq:nonexpansive_lip}
    \sup_{\vz\neq \vz'}\frac{\|F(\vz)-F(\vz')\|}{\|\vz-\vz'\|} \;\le\; 1 \,.
\end{equation}
\vspace{-.5em}
\end{definition}

\begin{definition}[Fejér monotonicity~\citep{bauschke2011convex}]
Let $D \subseteq \mathbb{R}^d$ be nonempty. Sequence $(\vz_t)_{t\ge 0}\subset \mathbb{R}^d$ is \emph{Fejér monotone (w.r.t $D$)} if
\vspace{-.5em}
\begin{equation}\label{eq:fejer} \tag{FM}
    \|\vz_{t+1}-\vz'\| \;\le\; \|\vz_t - \vz'\| 
    \qquad \forall\, \vz' \in D,\ \forall\, t \ge 0 .
\end{equation}
\vspace{-.8em}
\end{definition}
\vspace{-.5em}
\begin{definition}[Saddle point]
A point $(\vx^\star, \vy^\star) \in D_x \times D_y$ is called a \emph{saddle point} of a convex--concave function $f \colon D_x \times D_y \to \mathbb{R}$ if 
\vspace{-.5em}
\begin{equation}
    f(\vx^\star, \vy) \;\le\; f(\vx^\star, \vy^\star) \;\le\; f(\vx, \vy^\star),
    \quad \forall \, \vx \in D_x, \; \vy \in D_y \,.
\vspace{-.5em}
\end{equation}
\end{definition}

\begin{definition}[Primal--dual gap of average iterate]
\label{def:pd-gap}
Let $(x^\star,y^\star)$ be a saddle point of $f\colon D_x \times D_y \to \mathbb{R}$. 
Given a sequence $\{(\vx^t,\vy^t)\}_{t=1}^k$, define the \emph{average iterate} as
\vspace{-.5em}
\begin{equation}
    \bar{\vx}^k = \tfrac{1}{k}\sum_{t=1}^k \vx^t \,, 
\qquad
\bar{\vy}^k = \tfrac{1}{k}\sum_{t=1}^k \vy^t \,.
\vspace{-.5em}
\end{equation}
The \emph{primal--dual gap} at $(\bar{\vx}^k,\bar{\vy}^k)$ is
\vspace{-.5em}
\begin{equation}\label{eq:pd_gap} \tag{$\mathcal{G}$}
    \mathcal{G}(\bar{\vx}^k,\bar{\vy}^k) 
    \;\equiv\; 
    \max_{\vy \in D_y} f(\bar{\vx}^k, \vy)
    \;-\;
    \min_{\vx \in D_x} f(\vx, \bar{\vy}^k) \,.
    \vspace{-.5em}
\end{equation}

\end{definition}
\begin{definition}[Real Schur decomposition]
\label{def:schur}
For any real square matrix $\mA \in \mathbb{R}^{n \times n}$, there exists an orthogonal matrix $\mathbf{Q} \in \mathbb{R}^{n \times n}$ such that
\[
    \mA \;=\; \mathbf{Q} \mathbf{T}  \mathbf{Q} ^\top,
\]
where $\mathbf{T}$ is a real block upper triangular matrix with $1\times 1$ blocks (real eigenvalues) and $2\times 2$ blocks (complex conjugate pairs). This factorization is called the \emph{real Schur decomposition}.
\end{definition}

We consider VIs defined by an operator $F \colon D \to \mathbb{R}^d$ that is monotone~\eqref{eq:monotone}, Lipschitz continuous~\eqref{eq:lipschitz}, and continuously differentiable ($C^1$), and show rates of \eqref{eq:pd_gap}.  
We refer to this class as \emph{$C^1$ monotone Lipschitz problems}.

We also consider the classical Bilinear Game (BG):
\vspace{-.5em}
\begin{equation}
    \tag{BG}
    \min_{\substack{ \vx \in D_x }} \quad \max_{\substack{\vy \in D_y}} \quad \vx^{\intercal} \mA \vy \,,
\label{eq:bilinear_game}
\end{equation}
\vspace{-.5em}
with $D_x\subseteq \R^{d_x}, D_y\subseteq \R^{d_y}$  $\mA\in \R^{d_x\times d_y}\,.$ 

For \ref{eq:bilinear_game}, we can compactly write the operator as $F(\vx,\vy) \equiv 
(\nabla_x (\vx^{\intercal} A \vy), \ \nabla_y (-\vx ^{\intercal} A\vy) )^\intercal = (A\vy, -A\vx)^\intercal \,.$

\paragraph{Optimization methods.}
\emph{Gradient descent} with step size $\gamma \in [0,1]$  for VIs is as follows:
\begin{equation} \tag{GD}\label{eq:gd}
    \vz_{t+1} = \vz_t - \gamma F(\vz_t) \,. 
\end{equation}
The \emph{LookAhead} algorithm~\citep{zhang2019LookAheadoptimizerksteps,chavdarova2021lamm} computes its update as a point on the line between the current iterate ($\vz_t$) and the iterate obtained after taking $k \geq 1$ steps with the base optimiser ($\tilde \vz_{t+k}$): 
\vspace{-.3em}
\begin{equation}
\tag{LA}
\vz_{t+1} \leftarrow \vz_t + \alpha (\tilde  \vz_{t+k} -  \vz_t), \quad \alpha \in [0,1]  
\,. \label{eq:LookAhead}
\vspace{-.3em}
\end{equation}
\noindent
In addition to the step size of the base optimizer (herein \ref{eq:gd}), LookAhead has two hyperparameters: 
\begin{itemize}[leftmargin=*,itemsep=0em,topsep=0em]
\item $k$---the number of steps of the base optimiser used to obtain the prediction $\smash{\tilde x_{t+k}}$; and
\item $\alpha$---averaging parameter that controls the step toward the predicted iterate  $\smash{\tilde x}$. The larger the closest, and $\alpha=1$~ \eqref{eq:LookAhead} is equivalent to running the base algorithm only (has no impact). 
\end{itemize}
Additional methods are deferred to Appendix~\ref{app:background}.

\paragraph{Laplace transform.}
The Laplace transform maps a time-domain function $\vx(t)$ to a complex frequency-domain function:  
\vspace{-.5em}
\begin{equation}
\label{eq:Laplace}\tag{LT}
\mathcal{L}\{\vx(t)\} = \mathbf{X}(s) = \int_{0}^{\infty} \vx(t) e^{-st} \, dt ,
\vspace{-.5em}
\end{equation}
where $s = \sigma + i\omega$ with $\sigma$ the exponential decay factor and $\omega$ the frequency.  
Unlike the Fourier transform, it accounts for growth/decay, making it suitable for analyzing unstable systems. It also converts differential equations into algebraic ones, which can be inverted via the \emph{inverse Laplace transform} (see App.~\ref{app:background}).

\paragraph{$Z$-transform.}
The $Z$-transform maps a discrete-time sequence $\vx[n]$ to a complex function:  
\vspace{-.5em}
\begin{equation}
\label{eq:Ztransform}\tag{ZT}
\mathcal{Z}\{\vx[n]\} = \mathbf{X}(\vz) = \sum_{n=0}^{\infty} \vx[n] \, \vz^{-n},
\vspace{-.5em}
\end{equation}
where $\vz \in \mathbb{C}$. The radius $|\vz|$ encodes growth or decay, and the angle $\arg(\vz)$ corresponds to oscillation frequency.  
As the discrete analogue of the Laplace transform, the $Z$-transform turns difference equations into algebraic ones, providing a natural tool for analyzing stability and dynamics of iterative algorithms. The sequence can be recovered via the \emph{inverse $Z$-transform}.

\paragraph{Modes.} 
\label{sec:mode}
The eigenvalues of the Jacobian of the discrete dynamics determine the rotational behavior in games.  
In the frequency domain, these eigenvalues are referred to as \emph{modes}. Since rotations are easier to analyze in this space, we use this terminology throughout.

Consider the linear iteration with update matrix $\mathbf{G} \in \mathbb{C}^{d \times d}$:
\vspace{-.5em}
\begin{equation}\label{eq:lti}
    \mathbf{x}_{t+1} = \mathbf{G}\,\mathbf{x}_t \,.
    \vspace{-.5em}
\end{equation}
A \emph{mode} is a complex number $\boldsymbol{\mu}$ with nonzero vector $\mathbf{v}$ such that
\vspace{-.5em}
\begin{equation}\label{eq:mode_def}
    \mathbf{G}\,\mathbf{v} = \boldsymbol{\mu}\,\mathbf{v}\,.
    \vspace{-.4em}
\end{equation}
Thus, the modes of \eqref{eq:lti} are the eigenvalues $\{\boldsymbol{\mu}_i\}$ of $\mathbf{G}$, with eigenvectors as mode shapes.  
Stability requires
\vspace{-.5em}
\begin{equation}\label{eq:stability}
    \rho(\mathbf{G}) \;\equiv\; \max_i |\boldsymbol{\mu}_i| \;<\; 1 \,.
    \vspace{-.5em}
\end{equation}

If the problem linearizes to operator $F$ with Jacobian $JF$, and LA (or extrapolation) induces
\vspace{-.5em}
\begin{equation}\label{eq:LookAhead_poly}
    \mathbf{G} = q(\mathbf{F}),
    \vspace{-.5em}
\end{equation}
for polynomial $q$, then the $z$-transform shows that poles (modes) are the eigenvalues of $G$.  
By spectral mapping, $\boldsymbol{\mu}_i = q(\boldsymbol{\lambda}_i)$, where $\{\boldsymbol{\lambda}_i\}$ are the eigenvalues of $JF$.  
In other words, LA maps each eigenvalue through $q$, effectively stretching/warping $\boldsymbol{\lambda}_i$.

\noindent\textbf{Modal vs. spectral.}  
When $\mathbf{G}$ (or $\mathbf{M}$) is diagonalizable over $\mathbb{C}$, modal and spectral spaces coincide.  
In particular, if $\mathbf{M}$ is normal ($\mathbf{M}\mathbf{M}^\ast=\mathbf{M}^\ast\mathbf{M}$), it admits an orthonormal eigendecomposition ($\mathbf{M}=\mathbf{U}\Lambda\mathbf{U}^\ast$).  
In such cases, each mode aligns with an eigenvector.  
Throughout, we assume $JF(\cdot)$ is full rank and normal, so modal and spectral decompositions coincide.

\section{Modal Stability Analysis: LA for Monotone Games}\label{sec:modal}

In this section, we show that variants of the \ref{eq:bilinear_game} contribute most towards the expansivity of the LookAhead (LA) operator $F^{\mathrm{LA}}$ for the class of $C^1$ $L$-smooth convex concave problems. 
Establishing nonexpansiveness (\ref{eq:nonexpansive}) of the LA operator ensures that its iterates are Fejér monotone (\ref{eq:fejer}) with respect to the solution set, 
which is the key ingredient in the convergence proof.

\begin{lemma}
\label{lem:bilinear-reduction}
Let $F\colon D\to\mathbb{R}^d$ be a $C^1$, monotone, and $L$-Lipschitz (\ref{eq:monotone}, \ref{eq:lipschitz}) operator. For a fixed $k\ge2$, $\alpha\in(0,1)$, $\gamma>0$, and any $\vu\neq \vv, \quad \vu, \vv \in \mathbb{R}^d $, we define the averaged Jacobian
$\vH(\vu,\vv)=\!\int_0^1 JF\!\big(\vv+\tau(\vu-\vv)\big)\,d\tau$, where $JF\colon D \to \mathbb{R}^{d \times d}$ is the Jacobian of the operator $F$.
Then, the supremum of $\big\|F^{\mathrm{LA}}_{k,\alpha}\big\|$ is attained at $2\times2$ skew blocks $H=\omega J$ with $J=\begin{psmallmatrix}0&1\\-1&0\end{psmallmatrix}$, where $\big\|F^{\mathrm{LA}}_{k,\alpha}\big\|$ is the worst-case Lipschitz constant of one LA step. Hence, nonexpansiveness for the $C^1$ monotone $L$-Lipschitz class is equivalent to checking the family $\{\omega \, JF:\ |\omega|\le L\}$.
\end{lemma}

\begin{proof}[Proof sketch]
$F(\vu)-F(\vv)=\vH(\vu,\vv)(\vu-\vv)$ with $\mathrm{sym}\,H\succeq0$, $\|\vH\|\le L$, so one LA step on $\Delta=\vu-\vv$ is $(1-\alpha)\vI+\alpha(\vI-\gamma \vH)^k$. Real-Schur decomposition (\ref{def:schur}) reduces to $1\times1$ real and $2\times2$ blocks $[\,\begin{smallmatrix}a&-b\\ b&a\end{smallmatrix}\,]$; the corresponding modal multiplier is $(1-\alpha)+\alpha(1-\gamma(a+ib))^k$. Increasing the symmetric part $a\ge0$ moves $1-\gamma(a+ib)$ toward the origin, decreasing the modulus; thus the worst case is $a=0$ with $|b|\le L$, i.e., $\vH=\omega \,JF$. Taking the supremum over $|\omega|\le L$ yields the stated characterization.
\end{proof}
For the remaining section, we consider that the modes of the Jacobian have no real part and hence contribute to the oscillations to the maximum extent. We now show that the hyperparameters of LA must be chosen very specifically to ensure that the overall operator is contractive. For this, we use stability analysis from classical dynamical systems theory: the \emph{dominant} mode (mode with the largest modulus) must lie within the unit ball (\ref{eq:stability}).

Let $c = \gamma \omega$, where $\omega \in (0, L$]. The LA operator transforms the purely imaginary mode $ic$ into 
\vspace{-.5em}
\begin{equation*}
\vmu_k(c;\alpha)\;:=\;(1-\alpha)+\alpha(1-ic)^{k} \,. 
\vspace{-.5em}
\end{equation*}
Next, we define the distance function
\vspace{-.5em}
\begin{equation*}
    g_k(c;\alpha)\;:=\;|\vmu_k(c;\alpha)|^2-1
    \vspace{-.5em}
\end{equation*}
and the computable threshold
\vspace{-.5em}
\begin{equation}
\label{eq:budget}
\Gamma_k^\star(\alpha)\;:=\;\sup\Bigl\{\Gamma\ge0:\ \sup_{c\in[0,\Gamma]} g_k(c;\alpha)\le 0\Bigr\} \,.
\vspace{-.5em}
\end{equation}

\begin{lemma}
    \label{lem:conv-cond}
    Let $F\colon D\to \mathbb{R}^d$ be $C^1$, monotone, and $L$-Lipschitz (\ref{eq:monotone}, \ref{eq:lipschitz}). Fix $k\ge2$, $\alpha\in(0,1)$, $\gamma>0$. 
Then the LA operator $F^{\mathrm{LA}}_{k,\alpha}(\vz)=(1-\alpha)\vz+\alpha(I-\gamma F)^k \vz$ satisfies
\vspace{-.5em}
\begin{equation*}
\|F^{\mathrm{LA}}_{k,\alpha}\| \;=\; \sup_{c\in[0,\gamma L]}|\vmu_k(c;\alpha)| \,,
    \vspace{-.5em}
\end{equation*}
where $\big\|F^{\mathrm{LA}}_{k,\alpha}\big\|$ is the worst-case Lipschitz constant of one LA step.
In particular,
\begin{enumerate}[itemsep=0em,topsep=0em] 
\item[(i)] If $\gamma L\le \Gamma_k^\star(\alpha)$, then $F^{\mathrm{LA}}_{k,\alpha}$ is nonexpansive.
\item[(ii)] If $\gamma L> \Gamma_k^\star(\alpha)$, there exists a monotone $L$-Lipschitz instance (a bilinear convex--concave game) for which $F^{\mathrm{LA}}_{k,\alpha}$ is expansive.
\item[(iii)] If $\alpha<1-\tfrac1k$, then $\Gamma_k^\star(\alpha)>0$.
\end{enumerate}
\end{lemma}
\emph{Proof sketch.} On purely rotational modes, the LA ensures that the Lipschitz modulus is $\sup_{c\in[0,\gamma L]}|\vmu_k(c;\alpha)|$.  
(i) Nonexpansiveness is equivalent to $g_k(c;\alpha)=|\vmu_k(c;\alpha)|^2-1\le0$. 
(ii) If the inequality fails at some $c_0$, one can construct a bilinear instance with frequency $c_0/\gamma$, where $|\vmu_k(c_0;\alpha)|>1$. Hence, the operator is expansive.  
(iii) A Taylor expansion near $c=0$ shows $g_k(c;\alpha)\approx(\alpha^2k^2-\alpha k(k-1))c^2$, which is negative for small $c>0$ when $\alpha<1-\tfrac1k$, ensuring a positive stability window.  \qedhere

\noindent\emph{Interpretation.}  Intuitively, $\Gamma_k^\star(\alpha)$ identifies the largest admissible steps for which the \emph{dominant mode} (i.e., the most oscillatory direction) remains contractive. Beyond this boundary, where $\gamma L > \Gamma_k^\star(\alpha)$, one can construct a game whose dominant mode escapes the unit disk, certifying divergence of LA on that instance. Thus, $\Gamma_k^\star(\alpha)$ operationally defines the \emph{maximum convergent} step size, and scaling it by $L$ provides a concrete hyperparameter selection technique for LA in monotone convex--concave problems.

\subsection{LA convergence}
\label{theo:la-conv}
\begin{theorem}{Let $F\colon D \to \mathbb{R}^d$ be a $C^1$, monotone, $L$- Lipschitz operator (\ref{eq:monotone}, \ref{eq:lipschitz}) with a set of fixed points $Z^{\star} \neq \varnothing$. Let $k \geq 2$ be an integer, $\alpha \in (0, 1- \frac{1}{k})$ and $\gamma > 0$ such that $\gamma \, L \leq \Gamma_k^{\star} (\alpha)$. LA ($k, \alpha$) iterate $\vz_{t+1} = F^{\mathrm{LA}} \, \vz_t$ with Gradient Descent($\gamma$) optimizer converges to a $\vz^\star \in Z^\star$. Furthermore, the rate of convergence of the average iterate (i.e. the rate at which the primal-dual gap shrinks (\ref{def:pd-gap})) is O($\frac{1}{T}$).}
\end{theorem}
\noindent\textbf{Proof Sketch.}
Lemma~\ref{lem:bilinear-reduction} reduces the worst-case dynamics to rotational modes, and Lemma~\ref{lem:conv-cond} shows that the condition 
$\gamma L \le \Gamma_k^\star(\alpha)$ guarantees nonexpansiveness (\ref{eq:nonexpansive}) of $F^{\mathrm{LA}}_{k,\alpha}$. 
Nonexpansiveness implies that the distance $\|\vz_{t}-\vp\|$ to any fixed point $\vp\in Z^\star$ is nonincreasing, 
so the iterates are Fejér monotone (\ref{eq:fejer}) with respect to $Z^\star$. 
Moreover, the averaging structure of LA forces successive differences $\|\vz_{t+1}-\vz_t\|$ to vanish, 
ensuring asymptotic regularity. 
Together these properties yield convergence of $\{\vz_t\}$ to some $\vz^\star\in Z^\star$ in finite dimensions. We leave the full proof as well as the proof for the convergence rate to the Appendix \ref{app:full-convergence-proof} and \ref{app:conv-rate-proof}.\qed

\subsection{MoLA Pseudocode}\label{sec:psuedocode}
Motivated by the above insights, we introduce \emph{Modal Lookahead} (MoLA), which chooses hyperparameters to maximize stability, as detailed next.

\noindent\textbf{Selecting hyperparameters.}
Given operator $F$ and anchor $\mathbf{z}_0$, Alg.~\ref{alg:choose-dominant-clear} linearizes the update map, computes the Jacobian eigenvalues, and maps them to one-step multipliers capturing modal amplification/rotation. We identify the dominant mode (largest magnitude), propose horizons $k$ aligned with its rotation, and cap $\alpha$ so no mode grows under $k$ steps plus averaging. Among feasible pairs, $(k^\star,\alpha^\star)$ minimizes worst-case amplification; if none exist, we fall back to $k_{\min}$ with a conservative $\alpha$.

\noindent\textbf{Lookahead (LA) optimization.}
Starting from $\mathbf{z}_0$, we query \textsc{ChooseModalParams} for $(k^\star,\alpha^\star)$ (re-queried optionally per cycle, fixed in stationary games). The method then runs $k^\star$ base steps $\mathbf{z}_{t+1}=\mathbf{z}_t-\gamma F(\mathbf{z}_t)$, followed by LA averaging~\eqref{eq:LookAhead} with $\alpha^\star$, and updates the anchor. This suppresses rotational components while preserving progress along contractive directions. Pseudocode is given in Alg.~\ref{alg:mola-general}.

\noindent\textbf{Eigenvalues.}
Eigenvalues (Line~\ref{line:3}, Alg.~\ref{alg:mola-general}) are computed with standard routines, e.g. \texttt{numpy.linalg.eigvals} on the local Jacobian. Equivalent functions exist in SciPy and PyTorch.

\begin{algorithm}[t]
\DontPrintSemicolon
\SetKw{KwRet}{return}
\caption{MoLA Pseudocode. }
\label{alg:mola-general}
\KwIn{Operator $F$ , base optimizer step size $\gamma>0$, iterations $T$, $k_{\min},k_{\max},\alpha\text{-grid}$, initial point $\vz_0 \in \mathbb{R}^d$}
\KwOut{Final iterate $\vz_T$}
Initialize $\vz\leftarrow \vz_0$\;
$\texttt{anchor}\leftarrow \vz$, $\texttt{step\_in\_cycle}\leftarrow 0$, $\texttt{cycle}\leftarrow 0$\;
$\Lambda \leftarrow \texttt{EstimateJacobianEigs}(F, \vz)$ \tcp*{eigs of $\nabla F(\vz)$ or local linearization}\label{line:3}
$(k,\alpha) \leftarrow \texttt{ChooseModalParams}(\Lambda, k_{\min},k_{\max},\alpha\text{-grid},\gamma)$\;
\For{$t \leftarrow 1$ \KwTo $T$}{
  $g \leftarrow F(\vz)$\;
  \tcp{\color{gray}Base first-order step (e.g., GD)}
  $\vz \leftarrow \vz - \gamma\, g$\;
  $\texttt{step\_in\_cycle} \leftarrow \texttt{step\_in\_cycle} + 1$\;
  \If{$\texttt{step\_in\_cycle} = k$}{
     \tcp{\color{gray}LA averaging}
     $\vz \leftarrow (1-\alpha)\,\texttt{anchor} + \alpha\, \vz$\;
     $\texttt{anchor} \leftarrow 
     \vz$; $\texttt{step\_in\_cycle} \leftarrow 0$; 
  }
}
\KwRet $\vz$\;
\end{algorithm}

\begin{algorithm}[!t]
\DontPrintSemicolon
\caption{ChooseModalParams; Alg.~\ref{alg:mola-general} calls}
\label{alg:choose-dominant-clear}
\SetKwInOut{KwIn}{Input}\SetKwInOut{KwOut}{Output}
\KwIn{Eigenvalues $\Lambda$ of $\nabla F$ (or local linearization); scalars $k_{\min},k_{\max}$, GD step size $\gamma$; $\alpha\_\text{grid}\subset(0,1)$}
\KwOut{Chosen $(k^\star,\alpha^\star)$}
\BlankLine
$T_{\mathrm{all}} \leftarrow 1 - \gamma \cdot \Lambda$ \tcp*{\color{gray}Eigen-values after one base step of GD}
\BlankLine
$i_{\mathrm{dom}} \leftarrow \arg\max_{i} \big|T_{\mathrm{all},i}\big|$\;
$T_{\mathrm{dom}} \leftarrow \{\,T_{\mathrm{all},i_{\mathrm{dom}}}\,\}$ \tcp*{\color{gray}Dominant mode}
\BlankLine
$(k^\star,\alpha^\star,\rho^\star) \leftarrow (k_{\min},\ 0.5,\ +\infty)$ 
\BlankLine
\For{each $\alpha \in \alpha\_\text{grid}$}{
  $\mathcal{K} \leftarrow \texttt{KCandidatesForAlpha}(T_{\mathrm{dom}},\,\alpha,\,k_{\min},\,k_{\max})$\;
  \If{$\mathcal{K}=\varnothing$}{\textbf{continue}}
  \For{each $k \in \mathcal{K}$}{
    $\alpha_{\max} \leftarrow \texttt{AlphaCap}(T_{\mathrm{dom}},\,k)$\;
    \If{$\alpha \le \alpha_{\max}$}{
      \tcp{\color{gray}Worst-mode spectral radius under LA $(k,\alpha)$}
      $\rho \leftarrow \max\limits_{\tau \in T_{\mathrm{dom}}} \left| (1-\alpha) + \alpha\, \tau^{\,k} \right|$\;
      \If{$\rho < \rho^\star$}{
        $(k^\star,\alpha^\star,\rho^\star) \leftarrow (k,\alpha,\rho)$\;
      }
    }
  }
}
\BlankLine
\If{$\rho^\star = +\infty$}{
\tcp{\color{gray}Fallback (no feasible pair)}
  $k^\star \leftarrow k_{\min}$\;
  $\alpha^\star \leftarrow \min\!\big\{0.5,\ \texttt{AlphaCap}(T_{\mathrm{dom}},\,k_{\min})\big\}$\;
}
\KwRet $(k^\star,\alpha^\star)$\;
\BlankLine\BlankLine
\SetKwProg{Fn}{Function}{:}{end}
\Fn{\texttt{AlphaCap}$(T_{\mathrm{stab}},\,k)$}{
  \tcp{Return $\max\{\alpha\in(0,1):\ \max_{\tau\in T_{\mathrm{stab}}}|(1-\alpha)+\alpha\,\tau^k|\le 1\}$}
  \KwRet $\alpha_{\max}(k)$
}
\BlankLine
\Fn{\texttt{KCandidatesForAlpha}$(T_{\mathrm{dom}},\,\alpha,\,k_{\min},\,k_{\max})$}{
  \tcp{Return integer horizons $k\in[k_{\min},k_{\max}]$ that are admissible for the dominant multiplier in $T_{\mathrm{dom}}$ (by the LookAhead-cycle geometry/derivation).}
  \KwRet $\mathcal{K}$
}
\end{algorithm}

\begin{figure}
    \begin{subfigure}[t]{.49\linewidth}
    \includegraphics[width=\linewidth]{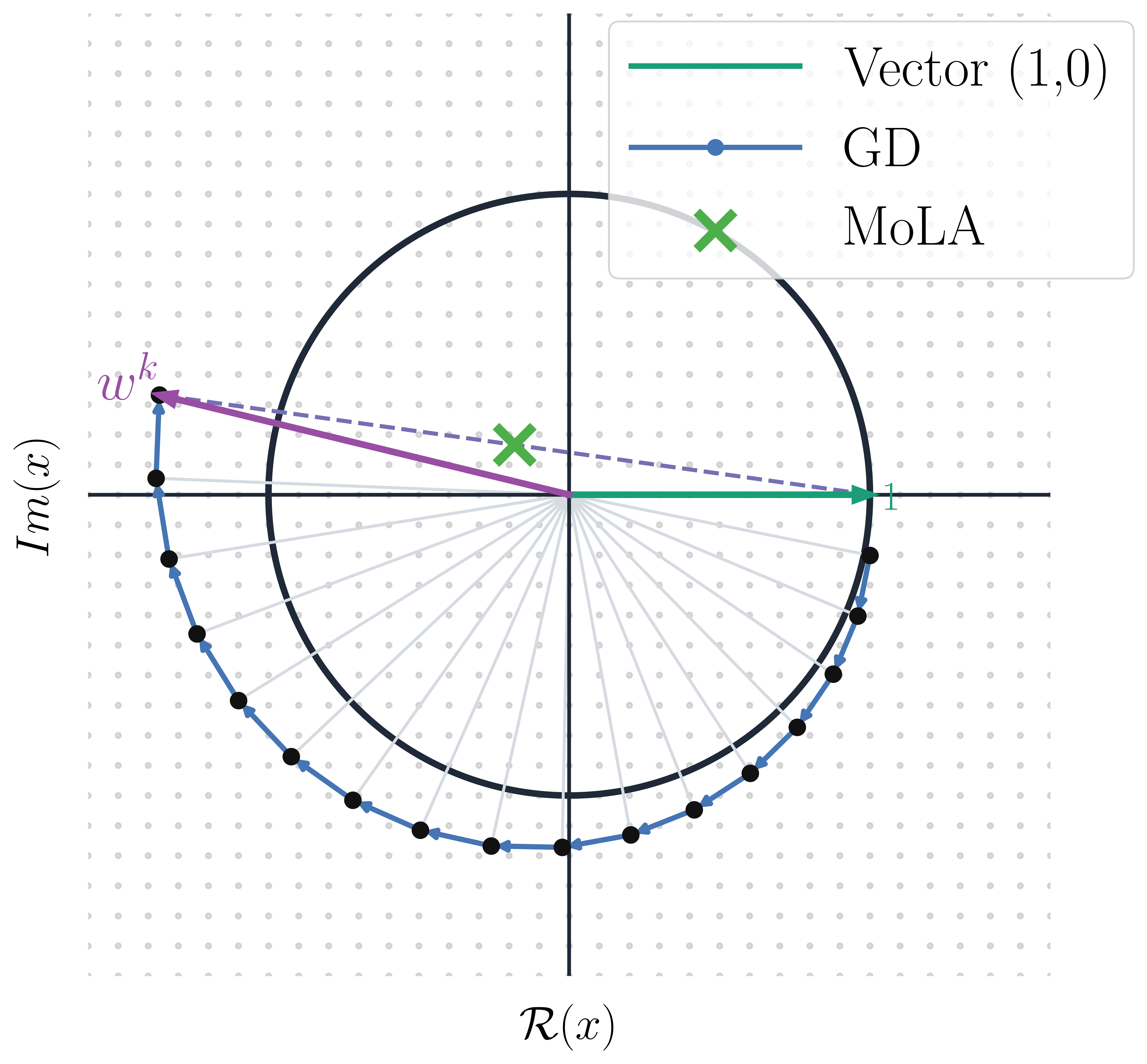}
    \vspace{-.7em}
    \caption{MoLA}
    \label{fig:mola_illustration1}
    \end{subfigure}
    \begin{subfigure}[t]{.49\linewidth}
    \includegraphics[width=\linewidth]{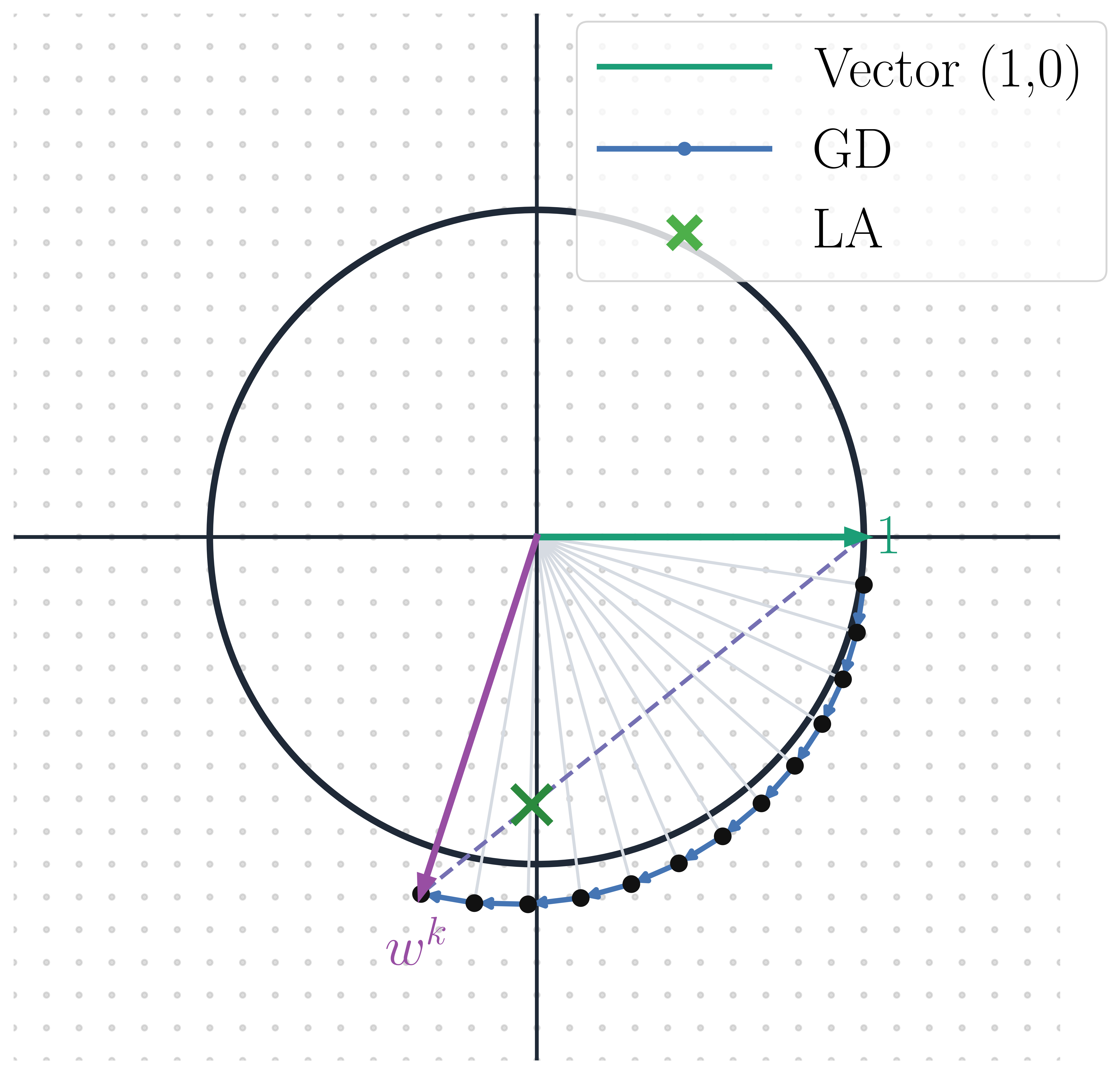}
    \vspace{-.7em}
    \caption{ LA}
    \label{fig:mola_illustration2}
    \end{subfigure}
    \vspace{-.7em}
    \caption{
    MoLA selects $(k,\alpha)$ so the average iterate is closer to the origin than for random choices. It sets $k$ so GD accumulates roughly a $\pi$ phase shift, then picks $\alpha$ to minimize the LA iterate’s distance to the origin.
    }
    \label{fig:mola_illustration}
    \vspace*{-0.5cm}
\end{figure}

\subsection{MoLA convergence}\label{sec:mola_theory}

\begin{theorem}
\label{theo:mola-conv}
    Let $F\colon D \to \mathbb{R}^d$ be a $C^1$, monotone, $L$- Lipschitz operator  with a set of fixed points $Z^{\star} \neq \varnothing$. Let $k \geq 2$ be an integer, $\alpha \in (0, 1- \frac{1}{k})$ and $\gamma > 0$ such that $\gamma \, L \leq \Gamma_k^{\star} (\alpha)$, where $\Gamma_k^{\star} (\alpha)$ is the maximum $\gamma L$ to ensure convergence as defined in \ref{eq:budget}.  Then MoLA converges to a fixed point $z^\star \in Z^\star$. The rate of convergence of the average iterate is $O(\frac{1}{T})$, improving on the baseline algorithm by a constant factor.
\end{theorem}
The detailed analysis is provided in Appendix \ref{sec:mola-rate}.

\emph{Interpretation.}
The rate constant in Theorem~\ref{theo:mola-conv} is $\displaystyle \frac{L}{2\,\alpha\,\Gamma_k^{\star}(\alpha)}$.
A \textbf{fixed} LA baseline that uses some $(k_0,\alpha_0)$ (and the maximal admissible $\gamma_0=\Gamma_{k_0}^\star(\alpha_0)/L$)
achieves the constant $\frac{L}{2\,\alpha_0\,\Gamma_{k_0}^\star(\alpha_0)}$. In contrast, \textbf{Modal LookAhead (MoLA)} \emph{chooses}
$(k,\alpha)$ to maximize the gain $\alpha\,\Gamma_k^\star(\alpha)$ and thus \emph{minimizes} the constant
$\frac{L}{2\,\alpha\,\Gamma_k^\star(\alpha)}$.
Therefore MoLA improves the $O(1/T)$ rate by the factor
\vspace{-.5em}
\[
\frac{\alpha^\star\,\Gamma_{k^\star}^\star(\alpha^\star)}{\alpha_0\,\Gamma_{k_0}^\star(\alpha_0)}\ \ge\ 1 \,,
\vspace{-.5em}
\]
with strict inequality unless the fixed baseline already selects a depth–wise maximizer of $\alpha\,\Gamma_k^\star(\alpha)$.
\paragraph{Mechanism Overview.}
Each complex mode can be written in terms of polar coordinates as $R \,e^{-i\theta}$, where $R$ represents the amplitude of the mode and $\theta$ the angle; refer to Figure~\ref{fig:mola_illustration}. GD rotates each mode and grows: after $k$ inner steps, the mode has
angle approximately $k\theta$ and amplitude approximately $R^{k}$ (using De Moivre's theorem).
The LA update then averages this rotating vector with the unit vector pointing
to $1$. Contraction is strongest when two conditions are nearly met:
\emph{(i)} the \textbf{phase} is in near opposition (angle difference close to $\pi$),
so the average cancels oscillation; and \emph{(ii)} the \textbf{amplitude} of the rotating vector
roughly matches the averaging weight, yielding a small average (target scale $\approx (1-\alpha)/\alpha$).

Because $k$ is an integer, the two targets cannot, in general, be satisfied exactly.
Our design rule therefore selects $k$ to be \emph{simultaneously close} to both:
align the phase near $\pi$ while keeping the amplitude near the target.
This balances the two effects, producing the smallest modal multiplier
and visibly damping the oscillations often observed in games. We provide detailed calculations in Appendix \ref{app:calc-mola}.

\section{Per-player Laplace Transform: Bilinear Game Case Study}\label{sec:laplace-based_analysis}

\citet{chavdarova2023hrdes} introduced the use of High Resolution Differential Equations (HRDEs) as a way to study games in continuous time. We extend the work on this direction, presenting an analysis of the continuous time dynamics in the frequency domain.

We first present the $\mathcal{O}(\gamma)$-HRDE of \ref{eq:LookAhead} with $(\alpha, k)$ kept as parameters.

We then propose the use of the Laplace transform to analyze convergence for the Bilinear Game (\ref{eq:bilinear_game}) in the frequency domain. Finally, we discuss the fundamental differences between the former modal approach and its overlapping conclusion.
The proofs of this section are deferred to Appendix~\ref{app:proofs_cont}.

\begin{theorem}    
\label{theo:3}
{General-$(\alpha,k)$  $\mathcal{O}(\gamma)$--HRDE of LA \& Convergence on \ref{eq:bilinear_game}.}

The $\mathcal{O}(\gamma)$-HRDE for~\eqref{eq:LookAhead} with \eqref{eq:gd} as a base optimizer is:
\begin{align}
\tag{LA-$\gamma$-HRDE}
\label{eq:la-hrde} 
\begin{split}
    & \ddot{\vz}(t) =  - \frac{2}{\gamma}\dot{\vz}(t) - \frac{2k\alpha}{\gamma} F(\vz(t)) \\ 
    & + k(k-1) \alpha \, JF(\vz(t)) \cdot F(\vz(t)), 
    \end{split}
\end{align}
where $\vz \in \mathbb{R}^{d}$ is the vector of all players, 
$k$ is the number of \eqref{eq:gd} steps  before the $\alpha$-averaging, $\gamma$ is the \eqref{eq:gd} step size, $F(\cdot)$ is the vector field, and $JF(\cdot)$ is the Jacobian of the vector field (i.e. $\nabla F$). 
\end{theorem}

\noindent
Using the above HRDE model, we next move to the frequency domain via the Laplace transform. 

\paragraph{LA for \ref{eq:bilinear_game}: Solution Trajectory \& Convergence}
We move to \eqref{eq:la-hrde} frequency dual via Laplace transform \eqref{eq:Laplace}, for \eqref{eq:bilinear_game}, then derive the solution trajectory through its inverse Laplace transform.

\begin{theorem}[Solution trajectories]
\label{theo:2}
Consider the~\ref{eq:bilinear_game} (with matrix $\mA$). Let $U$ be the orthogonal matrix from eigen decomposition of $\mA$ i.e., $\mA = \mU \mLambda \mU^\intercal$, where $\mLambda = \mathrm{diag}(\lambda_i)$. The trajectory of the individual players $\vx(t)$ and $\vy(t)$ of the~\eqref{eq:la-hrde} continuous time dynamics with parameters $(k, \alpha)$ and Gradient Descent with step size $\gamma$ as the base optimizer, are as follows:
\begin{align}
    \vx(t) &= -\frac{2k\alpha}{\gamma} (\mG \ast \vy)(t) + \mU \mD_x(t) \mU^\intercal \label{eq:trajectory_1} \tag{$x$-Sol} \\
    \vy(t) &= \frac{2k\alpha}{\gamma} (\mG * \vx)(t) + \mU \mD_y(t) \mU^\intercal \label{eq:trajectory_2} \tag{$y$-Sol}
\end{align}
where for a player $\mathbf{q} \in \{\vx, \vy\}$, $\mD_q(t)$ is a diagonal matrix whose $i$-th diagonal element is:
\vspace{-.5em}
\begin{align*}
    \big(\mD_q(t)\big)_{ii} = e^{-\frac{t}{\gamma}} \bigg[ \cosh(\omega_i t) &\mathbf{q}_i(0) \\
    &\hspace{-2.5em} + \frac{\sinh(\omega_i t)}{\omega_i} \Big( \dot{\mathbf{q}}_i(0) + \frac{\mathbf{q}_i(0)}{\gamma} \Big) \bigg] \,.
    \vspace{-.5em}
\end{align*}
Here, $*$ is the convolution operator, and we define
\vspace{-.5em}
\begin{align*}
    \omega_i &= \sqrt{\frac{1}{\gamma^2} - \alpha k(k-1) \lambda_i}, \\
    \mG(t) &= \mU \, \mathrm{diag}\left( \frac{e^{-t/\gamma} \sinh(\omega_i t)}{\omega_i} \right) \mU^\intercal \,,
    \vspace{-.5em}
\end{align*}
where $(\vx(0), \vy(0))$ and $(\dot{\vx}(0), \dot{\vy}(0))$ are the initial positions and momenta, respectively.
\end{theorem}

\noindent\emph{Interpretation.} 
Exponential terms drive convergence by damping oscillations, while periodic terms encode the equilibrium cycling characteristic of minimax problems. The dynamics are coupled, with $\vx(t)$ and $\vy(t)$ depending on each other, and the oscillation frequency determined by $\omega_0$.

For \eqref{eq:bilinear_game}, we derive a decoupled trajectory in Appendix \ref{app:proofs_cont} where $ \vx(t)$ becomes independent of $\vy(t)$, reducing the problem to pure minimization over $\vx$. 

\paragraph{GD \& LA Convergence Analysis}
In line with existing results~\citep{chavdarova2023hrdes}, pole-based stability analysis through Laplace transform of \eqref{eq:gda-hrde} confirms that \eqref{eq:gda-hrde} diverges on \eqref{eq:bilinear_game} for all $\gamma > 0$. Applying \eqref{eq:Laplace} to \eqref{eq:la-hrde} yields separate player-wise equations, which can be combined to obtain closed-form solution trajectories (see Appendix~\ref{app:proofs_cont}).

For \eqref{eq:bilinear_game}, substituting $\mathbf{Y}(s)$ into $\mathbf{X}(s)$ cancels the dependence on $\vy(t)$, so $\vx(t)$ depends on the opponent’s initial state, not its trajectory. This decoupling extends to affine variational inequalities.

\noindent
Importantly, \eqref{eq:Laplace} provides convergence conditions without requiring the initial values or solution trajectory of the corresponding HRDEs.
Convergence is determined by analyzing the poles of $\mathbf{X}(s)$. If the divergent pole (largest real part) lies in the negative half-plane, the trajectory $\vx(t)$ converges to $0$ as $t \to \infty$. For \ref{eq:LookAhead}, convergence for \ref{eq:bilinear_game} requires: 
\vspace{-.5em}
\begin{equation}\label{eq:cond_hamiltonian}\tag{BG-Cond}
    \alpha < \frac{k - 1}{k} \,. 
    \vspace{-.5em}
\end{equation}

\paragraph{Discussion.}
Both the \emph{modal-} and \emph{pole-}based stability analyses build on the same observation: rotational game dynamics give rise to oscillatory iterates. The modal view targets stability by shaping the dominant joint modes of the coupled system, whereas the pole-based view addresses stability by analyzing each player’s oscillation through their Laplace-domain transfer functions. Notably, as the case study above illustrates, both approaches arrive at the same convergence condition.

The continuous–time pole-based approach abstracts away discretization error effects, but offers greater interpretability and simpler analysis.
In contrast, the discrete-time modal approach is exact at the algorithmic level and provides explicit finite-step admissibility regions for \((k,\alpha,\gamma)\).

\section{Experiments}\label{sec:experiments}

\let\thefootnote\relax\footnotetext{Code implementation: \url{https://anonymous.4open.science/r/ModalLookAhead-D14B/}.}
We evaluate on two game setups:
A $d$-dimensional variant of \eqref{eq:bilinear_game}, where $\vA\in\mathbb{R}^{d\times d}$ has entries 
$\vA_{ij}\sim \mathcal{N}(0,\beta^2/d)$, and a structured convex–concave Quadratic Game (SC-SC) with controllable rotation and conditioning,
\[
f(x,y)=\tfrac{1}{2}\,\vx^\top(\eta_x\, \vI\,)\vx \;-\; \tfrac{1}{2}\,\vy^\top(\eta_y \, \vI\,)\vy \;+\; \vx^\top \vA \vy,
\]
where $\eta_x,\eta_y>0$ set the strong convexity/concavity in $x$ and $y$, and the bilinear coupling $A$ controls rotational strength. We construct $A=U\,\mathrm{diag}(\sigma)\,V^\top$, with $U,V\in\mathbb{R}^{d\times d}$, and $\sigma \in\mathbb{R}$. This parameterization allows systematic sweeps of rotation (via $\sigma$) relative to the contraction (via $\eta_x,\eta_y$). 
We experiment with dimension $d=100$ and stepsize $\gamma=0.01$.

We compare \emph{MoLA} with standard methods: \emph{Gradient Descent (GD)}, \emph{Extragradient (EG)}, \emph{Optimistic Gradient Descent (OGD)}, and \emph{LookAhead (LA)} with randomly chosen $k$. Further experimental details are provided in Appendix~\ref{app:exp_setup}.  
Each algorithm is run for $T$ base iterations. After every update, we record both the Euclidean distance to equilibrium $\mathrm{dist}(x_t,y_t)=\|(x_t,y_t)-(0,0)\|_2
$

and the cumulative CPU time under a fixed budget.

\begin{figure}
    \centering
    \includegraphics[width=.95\linewidth]{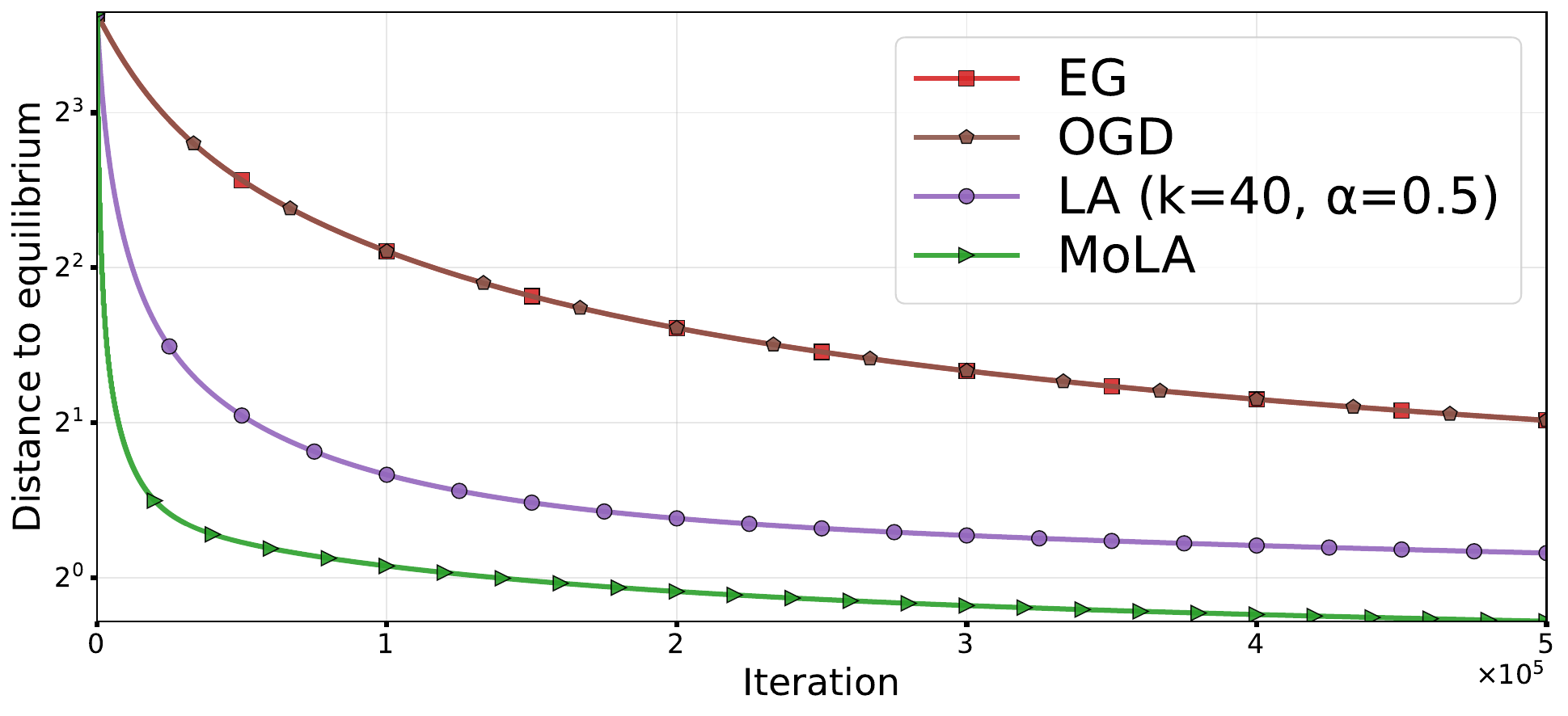}
    \vspace{-.7em}
    \caption{Convergence to equilibrium of \emph{EG, OGDA, LA, and MoLA} in BG with $d=100, \gamma=0.01$. GD is omitted since it diverges away.}
    \label{fig:dist_main}
    \vspace*{-0.2cm}
\end{figure}

\begin{figure}
    \centering
    \includegraphics[width=.95\linewidth]{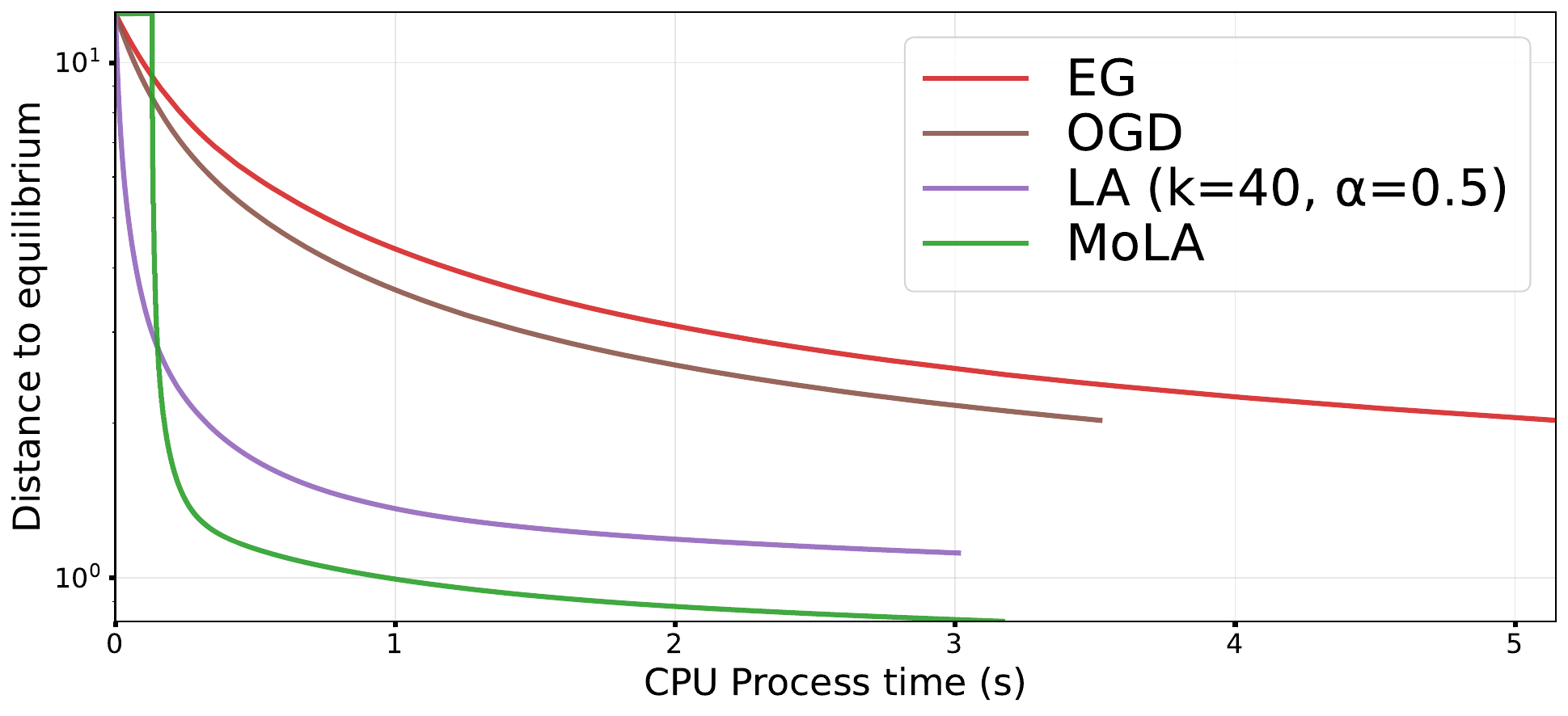}
    \vspace{-.7em}
    \caption{Convergence against CPU time of \emph{EG, OGDA, LA, and MoLA} in BG with $d=100, \gamma=0.01$.}
    \label{fig:conv_speed_main}
    \vspace*{-0.5cm}
\end{figure}

\paragraph{Results.}
Figures~\ref{fig:dist_main} and \ref{fig:conv_speed_main} summarize results for \eqref{eq:bilinear_game}.  
MoLA consistently stabilizes the rotational dynamics and achieves the fastest reduction in distance, particularly in the early stages. While LA improves over GD, its performance is sensitive to the choice of $(k,\alpha)$. In contrast, MoLA eliminates this sensitivity by selecting hyperparameters based on the problem’s eigenvalues. Overall, MoLA attains the smallest distance to equilibrium and the best time-to-accuracy trade-off, aligning with our theoretical insights.

Figure~\ref{fig:scsc_rot} reports results on an SC--SC game with $\eta_x=\eta_y=0.1$ and $\sigma\in[0.7,0.9]$, representing a strongly rotational regime. MoLA consistently attains the smallest distance to equilibrium and the fastest decrease per iteration, outperforming all baselines by a substantial margin while maintaining stable trajectories. For completeness, Figure~\ref{fig:scsc_pot} (Appendix) shows a balanced potential--rotational instance ($\eta_x=\eta_y=0.5$ and $\sigma = 0.5$), where MoLA again delivers the best performance, indicating robustness across regimes.

\begin{figure}
    \centering
    \includegraphics[width=.95\linewidth]{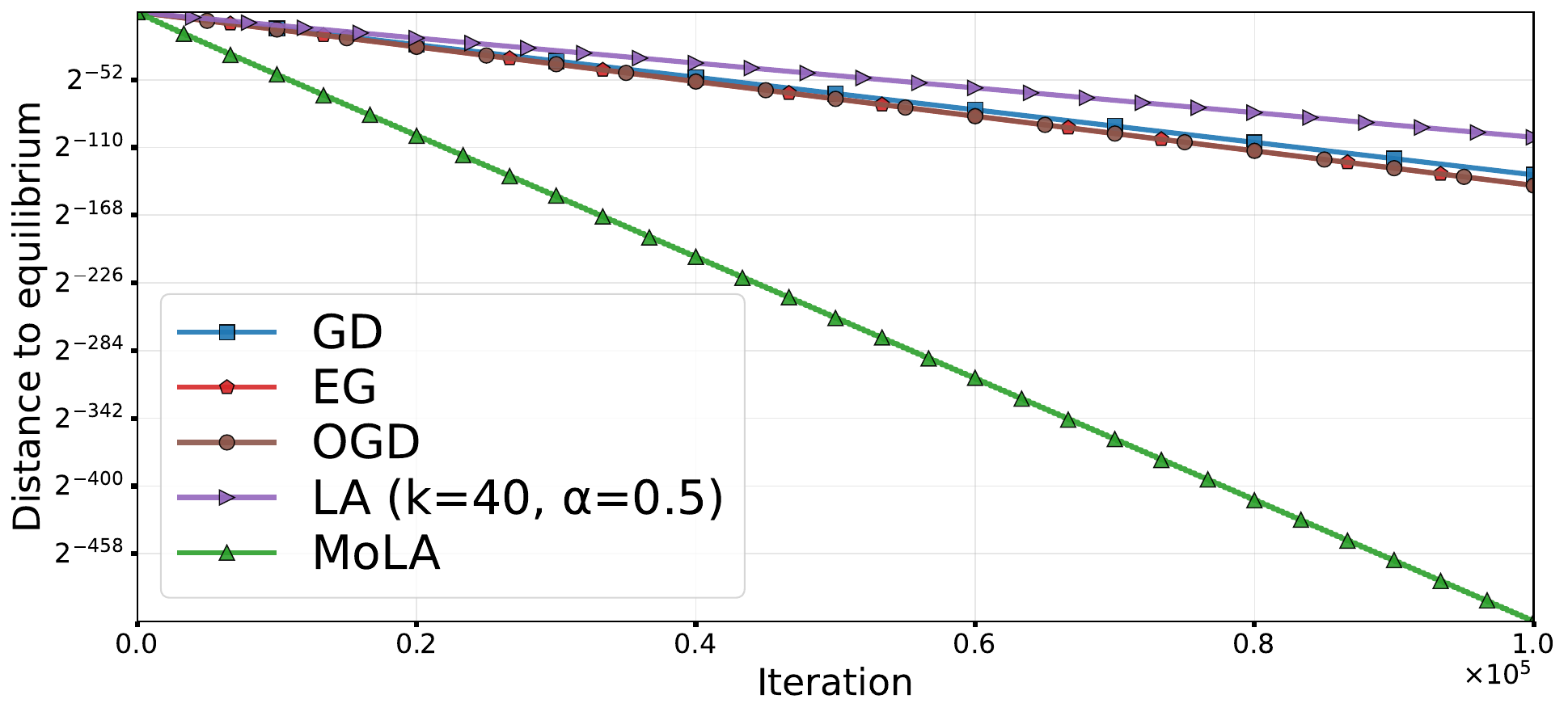}
    \vspace{-.7em}
    \caption{Distance to equilibrium vs. iterations for \emph{GD, EG, OGDA, LA, and MoLA} in a more rotational setting of SC-SC game with $d=100, \gamma=0.01$. The x-axis reports iteration count; the y-axis reports the Euclidean distance to the Nash equilibrium. 
}
    \label{fig:scsc_rot}
\end{figure}

\paragraph{Ablation over rotation.}
We examine how MoLA’s optimal hyperparameters vary with the rotation factor $\beta \in [0,1]$ in the Quadratic Game (QG):  
\vspace{-.5em}
\begin{equation}
    \tag{QG}
    \min_{\mathbf{x} \in \mathbb{R}^d}\max_{\mathbf{y} \in \mathbb{R}^d}\;
    (1-\beta)\,\mathbf{x}^\top \mathbf{x} + \beta\,\mathbf{x}^\top \mathbf{A}\mathbf{y} - (1-\beta)\,\mathbf{y}^\top \mathbf{y}.
\label{eq:guad_game}
\vspace{-.5em}
\end{equation}
For each $\beta$, we compute the Jacobian spectrum and apply \textsc{ChooseModalParams} to obtain $(k^\star,\alpha^\star)$.  
Figures~\ref{fig:k_values_rot} (\&  Appendix .~\ref{app:experiments}) show trends:  
when $\beta \approx 0$ (mostly potential), MoLA chooses a large horizon and $\alpha^\star \approx 1$, effectively disabling averaging.  As $\beta$ grows (more rotational), it selects smaller $k^\star$ (frequent anchoring) and $\alpha^\star < 1$, damping oscillations.

\begin{figure}
    \centering \includegraphics[width=0.65\linewidth]{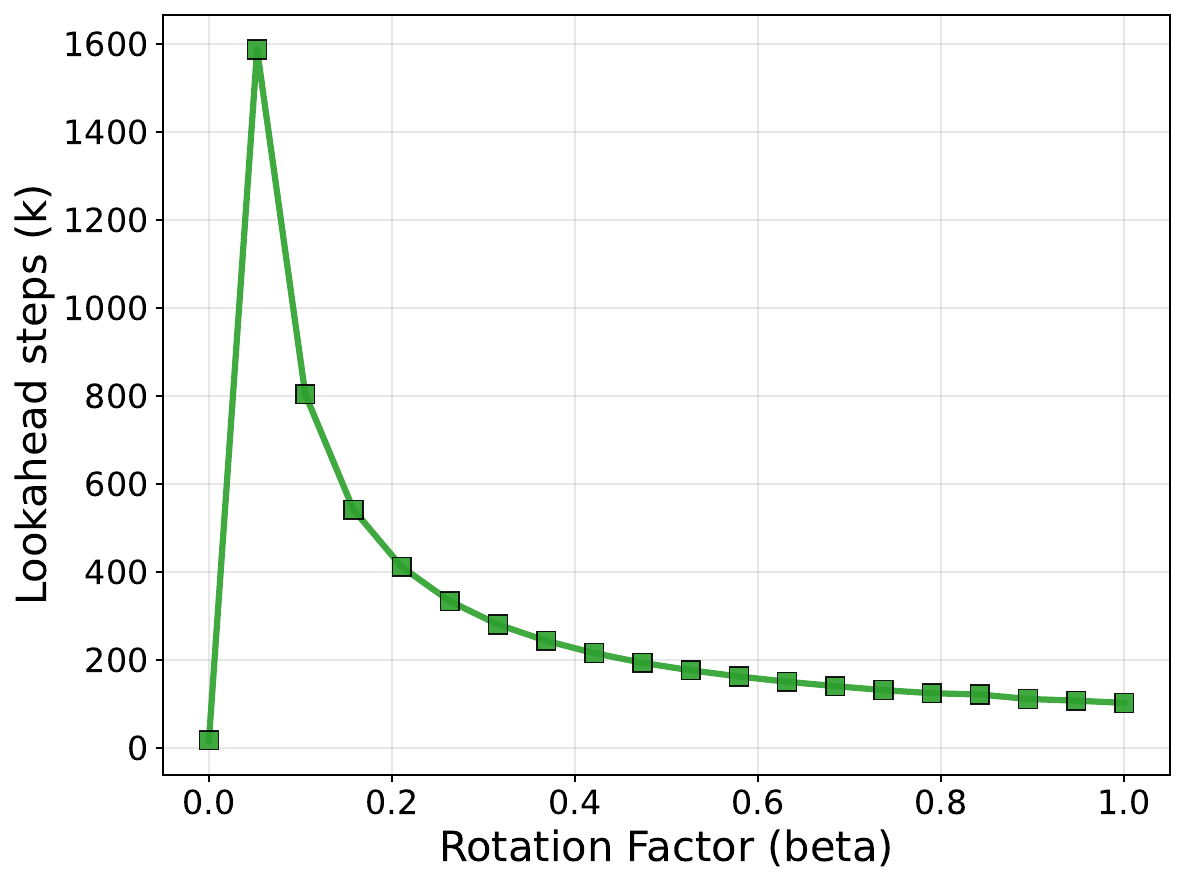}
    \vspace{-.7em}
    \caption{\textbf{Optimal LA horizon ($k$ vs. rotation factor $\beta$ in Quadratic Game.} Lower $\beta$ (more potential/convex–concave) favors larger horizons, while higher $\beta$ (more rotational;  $\beta \rightarrow 1$) 
    favors smaller $k$.}
    \label{fig:k_values_rot}
    \vspace*{-0.5cm}
\end{figure}

\section{Discussion}\label{sec:conclusion}
We studied Lookahead (LA) in game optimization through a modal perspective, motivated by the rotational nature of variational inequality dynamics. Our analysis provides convergence guarantees for standard LA and introduces \emph{MoLA}, a principled hyperparameter selection rule that maximizes stability via the dominant modes.

MoLA is motivated by the local modal picture: near equilibrium, game dynamics decompose into complex modes that evolve like rotating phasors, where each gradient step changes both phase and radius. After $k$ inner steps, LookAhead averages the rotated phasor with the anchor at $1$, and contraction is maximized when the accumulated rotation is close to $\pi$ so the average cancels oscillations. Because $k$ is integer-valued, one cannot in general simultaneously achieve ideal phase cancellation and amplitude matching, so MoLA adaptively selects $(k,\alpha)$ to best damp the dominant mode while keeping all relevant modes within the contraction region. This phase-alignment viewpoint explains the empirical stabilization: MoLA chooses the most permissive stable parameters that maximize oscillation attenuation (cf.\ Figures~\ref{fig:mola_illustration} and~\ref{fig:traj_illustration}; Appendix~\ref{app:calc-mola}).

Per iteration, MoLA retains the one-gradient-per-step structure of LookAhead: a cycle of length $k$ uses $k$ gradient evaluations, unlike extra-gradient methods which require two evaluations per update. Its additional overhead is a lightweight spectral estimate of the dominant Jacobian mode (e.g., a few power iterations), performed only once per LookAhead cycle and easily amortized by updating intermittently. In practice, this yields a favorable trade-off: the eigen-estimation cost of the Jacobian is typically outweighed by the ability to select the most permissive stable hyperparameters, reducing the total number of gradient evaluations needed to reach a target accuracy in strongly rotational regimes.

Experiments show that MoLA consistently stabilizes oscillatory dynamics and achieves faster convergence than classical baselines.
More broadly, our work underscores the value of frequency-domain viewpoints for games. Modal analysis gives an exact, discrete-time characterization that complements continuous-time pole-based approaches and offers practical guidance for hyperparameter tuning.

Future work includes developing efficient spectral approximations to scale MoLA, periodically re-estimating $(k,\alpha)$ in large-scale settings, and testing MoLA in applications such as GANs and multi-agent reinforcement learning.
In summary, we establish a modal framework for analyzing and tuning LA in games, bridging theory with practice and enabling scalable, adaptive extensions.
\FloatBarrier

\section*{Acknowledgements}
TC was partly supported by the Vienna Science and Technology Fund (WWTF) [Grant ID: 10.47379/VRG25019] and the FAIR (Future Artificial Intelligence Research) project, funded by the NextGenerationEU program within the PNRR-PE-AI scheme (M4C2, Investment 1.3, Line on Artificial Intelligence).
RB gratefully acknowledges the Gauss Centre for Supercomputing e.V. (www.gauss-centre.eu) for funding this project by providing computing time on the GCS Supercomputer JUWELS~\citep{JUWELS} at Jülich Supercomputing Centre (JSC). 
RB is also grateful for funding from the European Research Council (ERC) under the Horizon Europe Framework Programme (HORIZON) for proposal number 101116395 SPARSE-ML.

\bibliographystyle{plainnat}  
\bibliography{refs}

\clearpage
\section*{Checklist}

The checklist follows the references. For each question, choose your answer from the three possible options: Yes, No, Not Applicable.  You are encouraged to include a justification to your answer, either by referencing the appropriate section of your paper or providing a brief inline description (1-2 sentences). 
Please do not modify the questions.  Note that the Checklist section does not count towards the page limit. Not including the checklist in the first submission won't result in desk rejection, although in such case we will ask you to upload it during the author response period and include it in camera ready (if accepted).

\begin{enumerate}

  \item For all models and algorithms presented, check if you include:
  \begin{enumerate}
    \item A clear description of the mathematical setting, assumptions, algorithm, and/or model. [Yes. A clear description of the preliminaries has been provided in Section \ref{sec:prelim}.]
    \item An analysis of the properties and complexity (time, space, sample size) of any algorithm. [Yes]
    \item (Optional) Anonymized source code, with specification of all dependencies, including external libraries. [Yes]
  \end{enumerate}

  \item For any theoretical claim, check if you include:
  \begin{enumerate}
    \item Statements of the full set of assumptions of all theoretical results. [Yes]
    \item Complete proofs of all theoretical results. [Yes]
    \item Clear explanations of any assumptions. [Yes]     
  \end{enumerate}

  \item For all figures and tables that present empirical results, check if you include:
  \begin{enumerate}
    \item The code, data, and instructions needed to reproduce the main experimental results (either in the supplemental material or as a URL). [Yes]
    \item All the training details (e.g., data splits, hyperparameters, how they were chosen). [Yes]
    \item A clear definition of the specific measure or statistics and error bars (e.g., with respect to the random seed after running experiments multiple times). [Yes]
    \item A description of the computing infrastructure used. (e.g., type of GPUs, internal cluster, or cloud provider). [Not Applicable]
  \end{enumerate}

  \item If you are using existing assets (e.g., code, data, models) or curating/releasing new assets, check if you include:
  \begin{enumerate}
    \item Citations of the creator If your work uses existing assets. [Yes]
    \item The license information of the assets, if applicable. [Not Applicable]
    \item New assets either in the supplemental material or as a URL, if applicable. [Not Applicable]
    \item Information about consent from data providers/curators. [Not Applicable]
    \item Discussion of sensible content if applicable, e.g., personally identifiable information or offensive content. [Not Applicable]
  \end{enumerate}

  \item If you used crowdsourcing or conducted research with human subjects, check if you include:
  \begin{enumerate}
    \item The full text of instructions given to participants and screenshots. [Not Applicable]
    \item Descriptions of potential participant risks, with links to Institutional Review Board (IRB) approvals if applicable. [Not Applicable]
    \item The estimated hourly wage paid to participants and the total amount spent on participant compensation. [Not Applicable]
  \end{enumerate}

\end{enumerate}

\clearpage
\appendix
\thispagestyle{empty}

\onecolumn
\aistatstitle{Supplementary Material: \\ Frequency-Based Hyperparameter Selection in Games}

\section{Extended Preliminaries}\label{app:background}

In this section, we present an extended background on transform- and operator-based tools, such as the convolution, Laplace, and 
$Z$-transforms—that underpin our spectral and modal analysis of game dynamics. We also provide an extended overview of some of the gradient methods that perform well for VIs, and discuss the measure used to quantify the convergence rate.

\subsection{Foundations of Transforms}

This section formally introduces the convolution operator and reviews key properties of the Laplace transform—along with its inverse—that underpin the results in Section~\ref{sec:laplace-based_analysis} and Theorem~\ref{theo:2}. We also include an extended discussion of the $Z$-transform, which serves as the basis for our discrete-time analysis and the proposed MoLA method (Section~\ref{sec:modal}).

\paragraph{Convolution operator.}
The convolution operator plays a fundamental role in control theory, as it characterizes systems whose outputs at a given time depend on the accumulated influence of past inputs.
Let \( f(t) \) and \( g(t) \) be functions defined for \( t \geq 0 \), \textit{i.e.}, \( f(t)\colon [0, \infty) \to \mathbb{R} \) and \( g(t)\colon [0, \infty) \to \mathbb{R} \). Their convolution is defined as:
\begin{equation}
\label{convolution} \tag{CONV}
    (f * g)(t) = \int_{0}^{t} f(\tau) g(t - \tau) \, d\tau, \quad t \geq 0 \,.
\end{equation}

\paragraph{Laplace transform.}
For a function $\vx(t)$, the Laplace transform is defined as 
\begin{equation}
\tag{LT}
\mathcal{L}\{\vx(t)\} = \mathbf{X}(s) = \int_{0}^{\infty} \vx(t) e^{-st} \, dt
\end{equation}  
where $t$ denotes the time and $s \in \mathbb{C}$ is the complex frequency variable.

The inverse Laplace transform recovers the original time-domain signal from its frequency-domain representation. It is given by the Bromwich integral (also known as the inverse Laplace integral):
\begin{equation}
\label{inverse-laplace}\tag{iLT}
    \mathcal{L}^{-1}\{\mathbf{X}(s)\} = \vx(t) = \frac{1}{2\pi i} \int_{c-i\infty}^{c+i\infty} \mathbf{X}(s) e^{st} \, ds \,,
\end{equation}
where the contour of integration is taken along a vertical line in the complex $s$-plane in the region of convergence (ROC) of $\mathbf{X}(s)$. In other words, the path of integration runs \emph{parallel to the imaginary axis}, ensuring that the integral converges and uniquely defines the inverse transform.

A key property of the Laplace transform is that it converts convolution in the time domain into multiplication in the frequency domain.
If \( \mathcal{L}\{f(t)\} = F(s) \) and \( \mathcal{L}\{g(t)\} = G(s) \), then:
\[
\mathcal{L}\{(f * g)(t)\} = F(s) G(s) \,.
\]

Thus, the convolution of two signals in time corresponds to the product of their transforms in the frequency domain. Conversely, the inverse Laplace transform of a product yields a convolution:
\[
\mathcal{L}^{-1}\{F(s) G(s)\} = (f * g)(t) = \int_{0}^{t} f(\tau) g(t - \tau) \, d\tau  \,.
\]

Hence, the product of two functions in the frequency domain corresponds to the convolution of their inverse Laplace transforms in the time domain.
This duality underpins many analytical tools in systems and control theory, simplifying the study of linear time-invariant (LTI) dynamics.

\paragraph{$Z$-transform.}
For a discrete-time signal \(x[k]\) defined for \(k\ge 0\), the (one-sided) $Z$-transform is defined as
\begin{equation}
\label{z-def}\tag{Z}
\mathcal{Z}\{x[k]\} = X(z) = \sum_{k=0}^{\infty} x[k]\, z^{-k} \,,
\end{equation}
where \(z \in \mathbb{C}\). The series converges on a ROC, an annulus \(\{z: r_1<|z|<r_2\}\) determined by the growth or the decay of \(x[k]\).

The inverse $Z$-transform (for any \(r\) in the ROC) is given by the contour integral
\begin{equation}
\label{z-inv}\tag{iZ}
\mathcal{Z}^{-1}\{X(z)\} = x[k] = \frac{1}{2\pi i}\oint_{|z|=r} X(z)\, z^{k-1}\,dz \,.
\end{equation}

For the unilateral discrete convolution
\[
(x * y)[k] \;=\; \sum_{n=0}^{k} x[n]\,y[k-n], \qquad k\ge 0,
\]
we have
\begin{equation}
\label{z-conv}\tag{Z-Conv}
\mathcal{Z}\{(x*y)[k]\} = X(z)\,Y(z) \,,
\end{equation}
with ROC containing at least the intersection of the individual ROCs (with standard caveats near poles).

Evaluating $X(z)$ on the unit circle gives the discrete-time Fourier transform (DTFT):
\begin{equation}
\label{dtft}\tag{DTFT}
X(e^{i\omega}) = \sum_{k=0}^\infty x[k]\,e^{-i\omega k} = X(z)\big|_{z=e^{i\omega}} \,,
\end{equation}
whenever the unit circle lies in the ROC. 
For a causal LTI system with transfer function \(H(z)=\frac{B(z)}{A(z)}\), \emph{bounded input bounded output (BIBO)} stability holds if and only if all poles lie strictly inside the unit circle \(|z|<1\).

\paragraph{Relation between Laplace and $Z$-domain stability.}
With sampling period \(\Delta>0\) and \(z = e^{s\Delta}\),
\begin{equation}
\label{bilinear-map}\tag{S2Z}
\text{continuous-time } s\text{-plane} \;\longleftrightarrow\; \text{discrete-time } z\text{-plane},
\end{equation}
so left-half-plane stability (\(\Re(s)<0\)) maps inside the unit disk (\(|z|<1\)). This correspondence underlies the discretization and pole–zero mapping techniques commonly used in control and optimization.

\paragraph{Dominant mode.}
We define the \emph{dominant mode}, which is the eigenvalue with the largest modulus:
\begin{equation}\label{eq:dom_mode}\tag{$\lambda_{ \mathrm{dom}}$ }
    \lambda_{\mathrm{dom}} \colon= \arg\max_{\lambda_i \in \sigma(J)} |\lambda_i| \,,
\end{equation}
This eigenvalue represents the strongest oscillatory component of the system, with its magnitude determining the maximal radius of rotation in the complex plane.

\subsection{VI Background}

This section provides additional background on variational inequality (VI) methods and convergence metrics used throughout our analysis. We begin by revisiting two foundational operator-based algorithms---Extragradient (EG) and Optimistic Gradient Descent (OGD)---which extend gradient descent by incorporating predictive or optimistic updates to handle non-monotone or saddle-point game structures. We then introduce the \emph{primal-dual gap} and its \emph{restricted variant}, which serve as standard measures of suboptimality in convex–concave optimization problems.

\subsubsection{Additional VI Methods}\label{app:vi_mds}

\paragraph{Extragradient (EG).} The Extragradient method~\citep{korpelevich1976extragradient} augments gradient descent with a one-step extrapolated update. Specifically, it first computes an intermediate prediction (extrapolated iterate) via Eq.~\eqref{eq:gd}:
$\vz_{t+\frac{1}{2}} \!=\! 
 \vz_t - \eta F(\vz_t) \,, 
$
and then uses the operator evaluated at this predicted point to perform the actual update:
\begin{equation} \label{eq:extragradient} \tag{EG}
\vz_{t+1}=\vz_t-\gamma\,F(\vz_{t+\frac{1}{2}})    \,.
\end{equation}
This two-step structure enables EG to achieve convergence even in settings where standard gradient descent diverges, such as in bilinear games~\citep{korpelevich1976extragradient}.

\paragraph{Optimistic Gradient Descent (OGD).}
The Optimistic Gradient Descent method~\citep{popov1980} leverages \emph{optimism} by extrapolating the next gradient from the two most recent evaluations. Its update rule is:
\begin{equation} \label{eq:ogda} \tag{OGD}
    \vz_{t+1} = 
    \vz_{t} - 2\gamma F(\vz_t) + \gamma F(\vz_{t-1}) 
    \,.
\end{equation}
This formulation can be interpreted as incorporating a correction term that anticipates the operator’s future behavior, often leading to improved stability and faster convergence in game optimization.

\subsection{Convergence Measure}\label{app:gap_function}

In convex–concave optimization problems, the suboptimality of a candidate solution $(\vx, \vy)$ is commonly measured using the \emph{primal–dual gap}, defined as 
\[
\mathcal{G}(\vx, \vy) := \max_{\vy' \in \mathcal{Y}} \ell(\vx, \vy') - \min_{\vx' \in \mathcal{X}} \ell(\vx', \vy) \,.
\]
By definition, \(  \mathcal{G}(\vx, \vy) \geq 0 \) for all \(  \vx, \vy \in \mathcal{X} \times \mathcal{Y} \), and equality holds if and only if $(\vx, \vy)$ is a saddle point of \( \ell \).

However, when the feasible sets  $\mathcal{X}$ and $\mathcal{Y}$ are unbounded, the primal-dual gap becomes infinite except at saddle points, limiting its practical use. To address this issue, a \emph{restricted primal-dual gap} is often employed. Given bounded subsets $\mathcal{D}_x \subset \mathcal{X}$ and $\mathcal{D}_y \subset \mathcal{Y}$, it is defiend as
\[
\mathcal{G}_{\mathcal{D}_x \times \mathcal{D}_y}(\vx, \vy) := \max_{\vy' \in \mathcal{D}_y} \ell(\vx, \vy') - \min_{\vx' \in \mathcal{D}_x} \ell(\vx', \vy) \,.
\]

The restricted version satisfies two key properties: 
\begin{enumerate}[label=(\roman*),font=\itshape]
    \item If $(\vx^*, \vy^*) \in \mathcal{D}_x \times \mathcal{D}_y$, then $\mathcal{G}_{\mathcal{D}_x \times \mathcal{D}_y}(\vx, \vy) \ge 0$ for any $(\vx, \vy) \in \mathcal{X} \times \mathcal{Y}$;  
    \item If $(\vx, \vy)$ is in the interior of $\mathcal{D}_x \times \mathcal{D}_y$, then $\mathcal{G}_{\mathcal{D}_x \times \mathcal{D}_y}(\vx, \vy) = 0$ if and only if $(\vx, \vy)$ is a saddle point. 
\end{enumerate}
By selecting $\mathcal{D}_x \times \mathcal{D}_y$ sufficiently large to contain this saddle point $(\vx^\star, \vy^\star)$, the restricted primal–dual gap provides a well-defined and interpretable measure of convergence in unbounded domains.

\clearpage

\section{Missing Proofs of Theorems in Discrete Time}\label{app:proofs_disc}

This section presents the detailed proofs and derivations of Lemmas~\ref{lem:bilinear-reduction} and \ref{lem:conv-cond} and Theorems~\ref {theo:la-conv} and \ref{theo:mola-conv} from Section~\S\ref{sec:modal}. We begin by restating the setup and notation used in the main text, followed by the formal statements and complete proofs of the results.

\paragraph{Setup.}
 Let $D \subset \mathbb{R}^d, D_x \
\subset \mathbb{R}^d, D_y \subset \mathbb{R}^d$ be bounded sets. The players are vectors $\vx \in D_x \subseteq \mathbb{R}^{d_x}, \, \vy\in D_y \subseteq \mathbb{R}^{d_y}$. Without loss of generality, let $d_x + d_y = d$ and hence $D_x \,\times \, D_y \subseteq D$. A mapping $F\colon D\to\mathbb{R}^d$ is \emph{monotone} if $\langle F(\vu)-F(\vv),\,\vu-\vv\rangle\ge 0$ for all $\vu,\vv$, and \emph{$L$-Lipschitz} if $\|F(\vu)-F(\vv)\|\le L\|\vu-\vv\|$ for all $\vu,\vv$ (from definitions \ref{eq:monotone} and \ref{eq:lipschitz}). We fix $\gamma>0$ (GD stepsize), a positive integer $k\ge 2$ (GD steps), and an \emph{averaging} parameter $\alpha\in(0,1]$. We define
\[
F^{\mathrm{GD}}:=\vI-\gamma F,\qquad
F^{\mathrm{LA}}_{k,\alpha}:=(1-\alpha)
\vI+\alpha\big(F^{\mathrm{GD}}\big)^{k}.
\]
We assume the solution set $Z^\star:=\{\vz\in D:\ F(\vz)=0\}$ is nonempty. Note that $\vz^\star\in Z^\star$ iff $F^{\mathrm{GD}}(\vz^\star)=\vz^\star$ iff $F^{\mathrm{LA}}_{k,\alpha}(\vz^\star)=\vz^\star$. 

\textbf{Notation.} Throughout, the norm symbol $\|\cdot\|$ is used to represent the Euclidean ($L_2$) norm when applied to vectors, and the corresponding induced spectral norm $\|\cdot\|_{\mathrm{op}}$ when applied to matrices.
For complex numbers $w \in \mathbb{C}$, the real part is denoted by $\Re (w)$ and the imaginary part by $\Im (w)$.

\subsection{Proof of Lemma 1}
\begin{lemmarep}[Restatement of Lemma ~\ref{lem:bilinear-reduction}]
 Let $F\colon D\to\mathbb{R}^d$ be a $C^1$, monotone, and $L$-Lipschitz (\ref{eq:monotone}, \ref{eq:lipschitz}) operator. For a fixed $k\ge2$, $\alpha\in(0,1)$, $\gamma>0$, and any $\vu\neq \vv, \quad \vu \in D_x, \vv \in D_y$, we define the averaged Jacobian
$\vH(\vu,\vv)=\!\int_0^1 JF\!\big(\vv+\tau(\vu-\vv)\big)\,d\tau$, where $JF\colon D \to \mathbb{R}^{d \times d}$ is the Jacobian of the operator $F$.
Then, the supremum of $\big\|F^{\mathrm{LA}}_{k,\alpha}\big\|$ is attained at $2\times2$ skew blocks $H=\omega J$ with $J=\begin{psmallmatrix}0&1\\-1&0\end{psmallmatrix}$. Hence, nonexpansiveness for the $C^1$ monotone $L$-Lipschitz class is equivalent to checking the family $\{\omega \, JF:\ |\omega|\le L\}$.
\end{lemmarep}
\begin{proof}
Fix $\vu,\vv\in D$ and set $\Delta:=\vu-\vv$. By the fundamental theorem of calculus along the segment $[\vv,\vu]$,
\[
F(\vu)-F(\vv)\;=\;\Big(\int_0^1 JF\,\!\big(\vv+\tau(\vu-\vv)\big)\,d\tau\Big)(\vu-\vv)\;=\;\vH(\vu,\vv)\,\Delta \,.
\]
Since $F$ is monotone and $L$-Lipschitz, we have $\operatorname{sym}JF(\vz)\succeq 0$, where $\operatorname{sym} \vA := \tfrac{1}{2}(\vA+ \vA^\top)$ denotes the symmetric part of the matrix $\vA$, and $\|JF(\vz)\|\le L$ for all $\vz$. It is straightforward to see that:  $\operatorname{sym}\vH(\vu,\vv)\succeq 0$ and $\|\vH(\vu,\vv)\|\le L$. Hence differences of one LA step satisfy
\[
F^{\text{LA}}_{k,\alpha}(\vu)-F^{\text{LA}}_{k,\alpha}(\vv)
\;=\;\Big[(1-\alpha)\vI+\alpha\,(\vI-\gamma \vH(\vu,\vv))^k\Big]\Delta
\]
so the local Lipschitz modulus at $(\vu,\vv)$ is $\|P(\vH(\vu,\vv))\|$ for the matrix polynomial $P(\vX):=(1-\alpha)\vI+\alpha\,(\vI-\gamma \vX)^k$. We refer to $P(\cdot)$ as the \emph{modal contraction factor}.

Let $\vH:=\vH(\vu,\vv)$ be fixed with $\operatorname{sym}\vH\succeq 0$ and $\|\vH\|\le L$. By the real Schur decomposition (\textbf{Def.} \ref{def:schur}) there exists an orthogonal $\mathbf{Q}$ such that $\mathbf{Q}^\top \vH \mathbf{Q}$ is real block upper-triangular with diagonal blocks either $1\times 1$ scalars $a\in\mathbb{R}$ or $2\times 2$ blocks $\vB(a,b):=\begin{psmallmatrix}a&-b\\ b&a\end{psmallmatrix}$; moreover $\operatorname{sym}\vH\succeq 0$ forces $a\ge 0$. Orthogonal similarity preserves operator norm and polynomials respect similarity, so
\[
\|P(\vH)\|=\|P(\mathbf{Q}^\top \vH \mathbf{Q})\|=\max\{\;\|P(a)\|,\ \|P(\vB(a,b))\|\ \text{over all blocks}\;\}.
\]
On a $1\times 1$ block we obtain the scalar modal contraction factor
\[
m(a)\;=\;\big|(1-\alpha)+\alpha(1-\gamma a)^k\big|\,.
\]
On a $2\times 2$ block with eigenvalues $a\pm ib$ the modal multiplier equals
\[
m(a,b)\;=\;\big|(1-\alpha)+\alpha(1-\gamma(a+ib))^k\big|\,.
\]
For fixed $b$, increasing $a\ge 0$ moves $1-\gamma(a+ib)$ horizontally toward the origin, strictly decreasing its modulus
$\big|\,1-\gamma(a+ib)\,\big|$. By the triangle inequality,
\[
m(a,b)\;=\;\big|(1-\alpha)+\alpha z^k\big|\ \le\ (1-\alpha)+\alpha|z|^k,\qquad z:=1-\gamma(a+ib),
\]
and the right-hand side is strictly decreasing in $a$; hence the maximum over admissible $a\ge 0$ is attained at $a=0$. Therefore, the worst-case $2\times 2$ blocks have the purely skew form
\[
\vB(0,b)=b\begin{bmatrix}0&-1\\[2pt]1&0\end{bmatrix}=\omega JF\,,\qquad \omega=b,\ \ |\omega|\le \|\vB(0,b)\| \le \|\vH\|\le L\,.
\]
For $1\times 1$ blocks, the worst case is also attained at $a=0$, where $m(0)=|(1-\alpha)+\alpha|=1$, so such blocks cannot exceed the contribution of the rotational modes. Taking the supremum over all admissible $\vH$ thus reduces to the family $\{\,\omega JF:\ |\omega|\le L\,\}$, and we obtain
\[
{\|F^{\text{LA}}_{k,\alpha}\|}_L\;=\;\sup_{|\omega|\le L}\ \Big|(1-\alpha)+\alpha(1-i\gamma\omega)^k\Big|\,.
\]
In particular, $F^{\text{LA}}_{k,\alpha}$ is nonexpansive on the full $C^1$ monotone $L$-Lipschitz class if and only if
$\big|(1-\alpha)+\alpha(1-i\gamma\omega)^k\big|\le 1$ for every $|\omega|\le L$, i.e., if and only if the inequality holds on the reduced $2\times 2$ skew family. This proves the claim.
\end{proof}

\subsection{Proof of Lemma 2}
\label{app:proof-lemma-2}
\begin{lemmarep}[Restatement of Lemma ~\ref{lem:conv-cond}]
    Let $F\colon D \to \mathbb{R}^d$ be $C^1$, monotone, and $L$-Lipschitz (\ref{eq:monotone}, \ref{eq:lipschitz}). Fix $k\ge2$, $\alpha\in(0,1)$, $\gamma>0$. 
Then the LA operator $F^{\mathrm{LA}}_{k,\alpha}(\vz)=(1-\alpha)\vz+\alpha(I-\gamma F)^k \vz$ satisfies
\begin{equation*}
\|F^{\mathrm{LA}}_{k,\alpha}\| \;=\; \sup_{c\in[0,\gamma L]}|\vmu_k(c;\alpha)| \,.
    \vspace{-.5em}
\end{equation*}
In particular,
\begin{enumerate}[itemsep=0em,topsep=0em]
\item[(i)] If $\gamma L\le \Gamma_k^\star(\alpha)$, then $F^{\mathrm{LA}}_{k,\alpha}$ is nonexpansive.
\item[(ii)] If $\gamma L> \Gamma_k^\star(\alpha)$, there exists a monotone $L$-Lipschitz instance (a bilinear convex--concave game) for which $F^{\mathrm{LA}}_{k,\alpha}$ is expansive.
\item[(iii)] If $\alpha<1-\tfrac1k$, then $\Gamma_k^\star(\alpha)>0$.
\end{enumerate}
\end{lemmarep}
\begin{proof}
By Lemma~1 and the real Schur decomposition, any real matrix $\vH$ with $\operatorname{sym}\vH\succeq 0$ and $\|\vH\|\le L$ admits an orthogonal change of basis that decomposes $\mathbb{R}^d$ into invariant subspaces of two kinds: (i) $1$D blocks corresponding to nonnegative real eigenvalues (pure dilations), and (ii) $2$D blocks on which $\vH$ is orthogonally similar to $\vB(a,b)=\bigl[\begin{smallmatrix}a&-b\\ b&a\end{smallmatrix}\bigr]$, i.e., a dilation–rotation with spectrum $a\pm ib$. Monotonicity enforces $a\ge 0$, and the worst case for the LookAhead polynomial occurs at $a=0$, whence each $2$D block is orthogonally similar to $\omega\,JF$ with $JF=\bigl[\begin{smallmatrix}0&-1\\ 1&0\end{smallmatrix}\bigr]$ and $\omega\in[0,L]$. Lemma~1 further reduces the Lipschitz modulus of one LA step to the supremum over these blocks of the operator norm of
\[
\mathbf{P}(\vH)\;=\;(1-\alpha)\,\vI+\alpha\,(\vI-\gamma\vH)^k.
\]
The $1$D blocks are nonexpansive and cannot dominate, while on a rotational $2$D block we may set $c:=\gamma\omega\in[0,\gamma L]$ and write the complex multiplier as
\[
\vmu_k(c;\alpha)\;:=\;(1-\alpha)+\alpha(1-ic)^k,
\qquad
\big\|F^{\mathrm{LA}}_{k,\alpha}\big\|\;=\;\sup_{c\in[0,\gamma L]} \big|\vmu_k(c;\alpha)\big|.
\]
To analyze $\big|\vmu_k(c;\alpha)\big|$, we write the complex numbers in their polar form :  $1-ic=re^{-i\theta}$ with $r=\sqrt{1+c^2}$ and $\theta=\arctan(c)\in[0,\tfrac{\pi}{2}]$, so
\[
(1-ic)^k=r^k e^{-ik\theta},
\quad
\big|\vmu_k(c;\alpha)\big|^2
=(1-\alpha)^2+\alpha^2 r^{2k}+2\alpha(1-\alpha)r^k\cos(k\theta).
\]
Equivalently, using the binomial series (about $c=0$) for the even and real parts,
\[
(1+c^2)^k
=1+k c^2+\tfrac{k(k-1)}{2}c^4+R_1(c),
\qquad
\Re(1-ic)^k
=1-\tfrac{k(k-1)}{2}c^2+\tfrac{k(k-1)(k-2)(k-3)}{24}c^4+R_2(c),
\]
we obtain
\[
\big|\vmu_k(c;\alpha)\big|^2-1
=\underbrace{\big(\alpha^2 k-\alpha(1-\alpha)k(k-1)\big)}_{=:A_2(\alpha,k)}c^2
+\underbrace{\Big(\alpha^2\tfrac{k(k-1)}{2}+\alpha(1-\alpha)\tfrac{k(k-1)(k-2)(k-3)}{12}\Big)}_{=:A_4(\alpha,k)}c^4
+\alpha^2 R_1(c)+2\alpha(1-\alpha)R_2(c).
\]
Standard integral-remainder bounds for binomial series yield
\[
|R_1(c)|\le \tfrac{k(k-1)(k-2)}{6}\,c^6(1+c^2)^{k-3},
\qquad
|R_2(c)|\le \tfrac{k(k-1)(k-2)}{6}\,c^6(1+c^2)^{k-3},
\]
and therefore the conservative estimate
\[
\big|\vmu_k(c;\alpha)\big|^2-1
\;\le\; A_2(\alpha,k)c^2+A_4(\alpha,k)c^4+A_6(\alpha,k)c^6(1+c^2)^{k-3},
\qquad
A_6(\alpha,k):=\Big(2\alpha-\alpha^2\Big)\frac{k(k-1)(k-2)}{6}.
\]
Define $g_k(c;\alpha):=\big|\vmu_k(c;\alpha)\big|^2-1$ and the budget
\[
\Gamma_k^\star(\alpha):=\sup\Big\{\Gamma\ge 0:\ \sup_{c\in[0,\Gamma]} g_k(c;\alpha)\le 0\Big\}.
\]
First, if $\gamma L\le \Gamma_k^\star(\alpha)$, then $\sup_{c\in[0,\gamma L]}|\vmu_k(c;\alpha)|\le 1$ and hence $\|F^{\mathrm{LA}}_{k,\alpha}\|\le 1$ (nonexpansiveness). Second, if $\gamma L>\Gamma_k^\star(\alpha)$, there exists $c_0\in(0,\gamma L]$ with $|\vmu_k(c_0;\alpha)|>1$; choosing $\omega_0=c_0/\gamma$ and a rotational block $\vH=\omega_0JF$ certifies an expansive instance. Third, near $c=0$ the sign is controlled by the quadratic coefficient $A_2(\alpha,k)$:
\begin{equation}
\label{eq:app-conv-cond}
A_2(\alpha,k)\le 0
\quad\Longleftrightarrow\quad
\alpha k\le k-1
\quad\Longleftrightarrow\quad
\alpha\le 1-\tfrac{1}{k}\,.    
\end{equation}

Thus, if $\alpha<1-\tfrac{1}{k}$, then $g_k(c;\alpha)\le 0$ for all $c$ in some interval $[0,\varepsilon]$, and by continuity $\Gamma_k^\star(\alpha)\ge \varepsilon>0$. Consequently,
\[
\big\|F^{\mathrm{LA}}_{k,\alpha}\big\|
=\sup_{c\in[0,\gamma L]} \big|\vmu_k(c;\alpha)\big|,
\quad
\|F^{\mathrm{LA}}_{k,\alpha}\|\le 1\ \text{whenever}\ \gamma L\le \Gamma_k^\star(\alpha),
\quad
\text{and}\ \Gamma_k^\star(\alpha)>0\ \text{for}\ \alpha<1-\tfrac{1}{k},
\]
which establishes the claims of Lemma~2.
\end{proof}

\subsection{Convergence of Fixed-$k$ LookAhead for $C^1$ \emph{Monotone Lipschitz} Games}

\label{app:full-convergence-proof}
\begin{theoremrep}[Restatement of Theorem \ref{theo:la-conv}]
    Let $F\colon D \to \mathbb{R}^d$ be a $C^1$, monotone, $L$- Lipschitz operator (\ref{eq:monotone}, \ref{eq:lipschitz}) with a set of fixed points $Z^{\star} \neq \varnothing$. Let $k \geq 2$ be an integer, $\alpha \in (0, 1- \frac{1}{k})$ and $\gamma > 0$ such that $\gamma \, L \leq \Gamma_k^{\star} (\alpha)$. LA ($k, \alpha$) iterate $\vz_{t+1} = F^{\mathrm{LA}} \, \vz_t$ with Gradient Descent($\gamma$) optimizer converges to a $\vz^\star \in Z^\star$. Furthermore, the rate of convergence of the average iterate (i.e. the rate at which the primal-dual gap shrinks (\ref{def:pd-gap})) is O($\frac{1}{T}$).
\end{theoremrep}

We present the proof of convergence in this section, and the proof for the convergence rate in the next section.

\begin{proof}
From Lemma~1, for any $\vu\in D$ and any zero $\vz^\star$ of $F$,
\[
\big\|F^{\mathrm{LA}}_{k,\alpha}(\vu)-\vz^\star\big\|
=\big\|P \, \!\big(\vH(\vu,\vz^\star)\big)\,(\vu-\vz^\star)\big\|
\le \big\|P\,\!\big(\vH(\vu,\vz^\star)\big)\big\|\,\|\vu-\vz^\star\|.
\]
Lemma~2 gives the uniform bound
\[
\big\|P(\vH)\big\|\;\le\;\rho
\quad\text{for all }\ \vH\ \text{with}\ \operatorname{sym}\vH\succeq 0,\ \|\vH\|\le L,
\qquad
\rho:=\sup_{c\in[0,\gamma L]}|\vmu_k(c;\alpha)|.
\]
Therefore, for the iterates $\vz_{t+1}=F^{\mathrm{LA}}_{k,\alpha}(\vz_t)$ we have the one-step inequality
\begin{equation}\label{eq:LA-contraction}
\|\vz_{t+1}-\vz^\star\|\ \le\ \rho\,\|\vz_t-\vz^\star\| \quad\text{for all } t\ge 0.
\end{equation}
If $\gamma L<\Gamma_k^\star(\alpha)$, then $\rho<1$ by Lemma~2. Iterating \eqref{eq:LA-contraction} yields the global linear rate
\[
\|\vz_t-\vz^\star\|\ \le\ \rho^{\,t}\,\|\vz_0-\vz^\star\|\ \xrightarrow[t\to\infty]{}\ 0.
\]
Moreover, uniqueness of the limit follows from the contraction: if $\vz_1,\vz_2$ are both fixed points of $F^{\mathrm{LA}}_{k,\alpha}$, then
\[
\|\vz_1-\vz_2\|=\big\|F^{\mathrm{LA}}_{k,\alpha}(\vz_1)-F^{\mathrm{LA}}_{k,\alpha}(\vz_2)\big\|
\le \rho\,\|\vz_1-\vz_2\|
\]
forces $\vz_1=\vz_2$ since $\rho<1$. Finally, any fixed point of $F^{\mathrm{LA}}_{k,\alpha}$ is a zero of $F$ (and conversely): if $F(\vz^\star)=\mathbf{0}$ then $(\vI-\gamma F)^k\vz^\star=\vz^\star$ so $F^{\mathrm{LA}}_{k,\alpha}(\vz^\star)=\vz^\star$; if $F^{\mathrm{LA}}_{k,\alpha}(\vz^\star)=\vz^\star$, then by Lemma~1 applied with $(\vu,\vv)=(\vz^\star,\vz^\star)$ and by the definition of $P$, we must have $F(\vz^\star)=\mathbf{0}$.

\end{proof}

\subsection{Convergence rate of  $O(1/T)$ for the average iterate for LookAhead}
\label{app:conv-rate-proof}
Let $f\colon D_x\times D_y\to \mathbb{R}$ be convex in $\vx$ and concave in $\vy$, and let 
\[
F(\vz) \;=\; \begin{bmatrix}\nabla_x f(\vx,\vy) \\ -\nabla_y f(\vx,\vy)\end{bmatrix}, 
\qquad \vz=(\vx,\vy),
\]
be the associated saddle operator, that is $C^1$ monotone and $L-$Lipschitz. 
Let the zero set $Z^\star=\{\vp:\,F(\vp)=0\}$ be nonempty. Assume the parameters $(\gamma,k,\alpha)$ are chosen such that Lemma 2 holds (i.e. $F^{\mathrm{LA}}$ is nonexpansive) and shares fixed points with $F$. Then for the iterates $(\vz_t)$ generated by
\[
\vz_{t+1} \;=\; F^{\mathrm{LA}}(\vz_t),
\]
the ergodic average $\bar \vz_T=\tfrac{1}{T}\sum_{t=0}^{T-1}\vz_t$ satisfies
\[
\mathcal{G}(\bar \vz_T)\;:=\;\sup_{\vy'} f(\bar \vx_T,\vy') - \inf_{\vx'} f(\vx',\bar \vy_T) \;\le\; \frac{\|\vz_0-\vp\|^2}{2\alpha\gamma T},
\qquad \forall \vp\in Z^\star.
\]
where $\mathcal{G}(\bar \vz_T)$ is the primal-dual gap of the average iterate. Hence, the primal–dual gap decays as $O(1/T)$.

\textbf{Note.} We use the \emph{restricted} primal-dual gap here, since $D_x$ and $D_y$ are bounded sets. In some cases where the feasible sets are unbounded, the primal-dual gap can be infinite (except at the saddle points), and thus cannot be used as a convergence measure.

\begin{proof}[Proof of convergence rate (Theorem 1) (continued)]
Fix any $\vp\in Z^\star$. For the first $k$ steps of the base optimizer,
\begin{align}
\|\vz_t^{(j+1)}-\vp\|^2
&= \|\vz_t^{(j)}-\vp\|^2 - 2\gamma\,\langle F(\vz_t^{(j)}),\,\vz_t^{(j)}-\vp\rangle + \gamma^2\|F(\vz_t^{(j)})\|^2,\label{eq:inner-one}\\
2\gamma\sum_{j=0}^{k-1}\!\langle F(\vz_t^{(j)}),\,\vz_t^{(j)}-\vp\rangle
&= \|\vz_t-\vp\|^2 - \|\vz_t^{(k)}-\vp\|^2 + \gamma^2\sum_{j=0}^{k-1}\!\|F(\vz_t^{(j)})\|^2.
\label{eq:inner-sum-equality}
\end{align}
For the averaging step,
\begin{align}
\|\vz_{t+1}-\vp\|^2
&= \big\|(1-\alpha)(\vz_t-\vp)+\alpha(\vz_t^{(k)}-\vp)\big\|^2 \nonumber\\
&= (1-\alpha)\|\vz_t-\vp\|^2+\alpha\|\vz_t^{(k)}-\vp\|^2-\alpha(1-\alpha)\|\vz_t^{(k)}-\vz_t\|^2.
\label{eq:outer-exact}
\end{align}
Subtracting \eqref{eq:outer-exact} from $\|\vz_t-\vp\|^2$ yields
\begin{equation}\label{eq:Vdrop}
\|\vz_t-\vp\|^2-\|\vz_{t+1}-\vp\|^2
= \alpha\big(\|\vz_t-\vp\|^2-\|\vz_t^{(k)}-\vp\|^2\big) + \alpha(1-\alpha)\|\vz_t^{(k)}-\vz_t\|^2.
\end{equation}

From \eqref{eq:inner-sum-equality} and \eqref{eq:Vdrop},
\begin{align}
\|\vz_t-\vp\|^2-\|\vz_{t+1}-\vp\|^2
&= \alpha\!\left(2\gamma\sum_{j=0}^{k-1}\!\langle F(\vz_t^{(j)}),\vz_t^{(j)}-\vp\rangle
- \gamma^2\sum_{j=0}^{k-1}\!\|F(\vz_t^{(j)})\|^2\right)
+ \alpha(1-\alpha)\|\vz_t^{(k)}-\vz_t\|^2.\nonumber
\end{align}
Equivalently,
\begin{equation}\label{eq:key-equality}
2\alpha\gamma\sum_{j=0}^{k-1}\!\langle F(\vz_t^{(j)}),\vz_t^{(j)}-\vp\rangle
= \big(\|\vz_t-\vp\|^2-\|\vz_{t+1}-\vp\|^2\big)
+ \alpha\gamma^2\sum_{j=0}^{k-1}\!\|F(\vz_t^{(j)})\|^2
- \alpha(1-\alpha)\|\vz_t^{(k)}-\vz_t\|^2.
\end{equation}
Define the inner-sum vector and its norm:
\[
\mathbf{S}_t := \sum_{j=0}^{k-1}F(\vz_t^{(j)}),\qquad
\boldsymbol{\Delta}_t := \vz_t^{(k)}-\vz_t = \gamma\,\mathbf{S}_t,\qquad
\|\boldsymbol{\Delta}_t\|^2=\gamma^2\|\mathbf{S}_t\|^2.
\]
Then the last two terms in \eqref{eq:key-equality} combine as
\begin{equation}\label{eq:combo}
\alpha\gamma^2\sum_{j=0}^{k-1}\!\|F(\vz_t^{(j)})\|^2
- \alpha(1-\alpha)\|\boldsymbol{\Delta}_t\|^2
= \alpha\gamma^2\left(\sum_{j=0}^{k-1}\!\|F(\vz_t^{(j)})\|^2 - (1-\alpha)\|\mathbf{S}_t\|^2\right).
\end{equation}

By Cauchy–Schwarz,
\begin{equation}\label{eq:CS}
\|\mathbf{S}_t\|^2=\Big\|\sum_{j=0}^{k-1}\!F(\vz_t^{(j)})\Big\|^2
\ \le\ k\sum_{j=0}^{k-1}\!\|F(\vz_t^{(j)})\|^2
\qquad\Longleftrightarrow\qquad
\sum_{j=0}^{k-1}\!\|F(\vz_t^{(j)})\|^2 \ \ge\ \frac{1}{k}\,\|\mathbf{S}_t\|^2.
\end{equation}
Insert \eqref{eq:CS} into \eqref{eq:combo}:
\begin{equation}\label{eq:combo-sign}
\alpha\gamma^2\sum_{j=0}^{k-1}\!\|F(\vz_t^{(j)})\|^2
- \alpha(1-\alpha)\|\boldsymbol{\Delta}_t\|^2
\ \ge\ \alpha\gamma^2\Big(\frac{1}{k}-(1-\alpha)\Big)\|\mathbf{S}_t\|^2.
\end{equation}
Therefore, if
\begin{equation}\label{eq:alpha-threshold}
\alpha\ \le\ 1-\frac{1}{k}
\end{equation}
the right-hand side of \eqref{eq:combo-sign} is nonpositive, and \emph{the entire combination} in \eqref{eq:combo} is $\le 0$. This holds since we assume that the parameters $(\gamma, k , \alpha)$ are chosen in accordance to Lemma 2 (for details, see \eqref{eq:app-conv-cond} in the proof of Lemma 2 (Appendix ~\S\ref{app:proof-lemma-2})). In particular, using \eqref{eq:key-equality} and \eqref{eq:combo}–\eqref{eq:combo-sign},
\begin{equation}\label{eq:per-outer-final}
2\alpha\gamma\sum_{j=0}^{k-1}\!\langle F(\vz_t^{(j)}),\vz_t^{(j)}-\vp\rangle
\ \le\ \|\vz_t-\vp\|^2-\|\vz_{t+1}-\vp\|^2.
\end{equation}

Summing \eqref{eq:per-outer-final} over $t=0,\dots,T-1$ gives
\begin{equation}\label{eq:telescope}
2\alpha\gamma\sum_{t=0}^{T-1}\sum_{j=0}^{k-1}\!\langle F(\vz_t^{(j)}),\vz_t^{(j)}-\vp\rangle
\ \le\ \|\vz_0-\vp\|^2-\|\vz_T-\vp\|^2\ \le\ \|\vz_0-\vp\|^2.
\end{equation}
By monotonicity of $F$, $0\le \langle F(\vz_t),\vz_t-\vp\rangle \le \sum_j\langle F(\vz_t^{(j)}),\vz_t^{(j)}-\vp\rangle$. Thus from \eqref{eq:telescope},
\[
\frac{1}{T}\sum_{t=0}^{T-1}\langle F(\vz_t),\vz_t-\vp\rangle\ \le\ \frac{\|\vz_0-\vp\|^2}{2\alpha\gamma\,T}.
\]
By convexity in $\vx$ and concavity in $\vy$,
\[
f(\bar\vx_T,\vy^\star)-f(\vx^\star,\bar\vy_T)\ \le\ \frac{1}{T}\sum_{t=0}^{T-1}\big(f(\vx_t,\vy^\star)-f(\vx^\star,\vy_t)\big)
= \frac{1}{T}\sum_{t=0}^{T-1}\langle F(\vz_t),\vz_t-\vp\rangle,
\]
and we conclude
\begin{equation}\label{eq:final-gap}
f(\bar\vx_T,\vy^\star)-f(\vx^\star,\bar\vy_T)\ \le\ \frac{\|\vz_0-\vp\|^2}{2\,\alpha\,\gamma\,T}\qquad(\forall\,\vp\in Z^\star)
\end{equation}
under the explicit algebraic condition $\alpha\le 1-\tfrac{1}{k}$. Taking the supremum of both sides of \eqref{eq:final-gap} over $\vx \in D_x$ and $\vy \in D_y$ yields the upper bound on the restricted primal-dual gap (Appendix ~\ref{app:gap_function}). 
\end{proof}

\emph{Discussion:}
Together, the two lemmas and the convergence theorem show that LookAhead with a GD base is best understood as a \emph{modal filter} acting on the Jacobian spectrum. Lemma~1 shows that, for the entire class of \(C^1\) monotone \(L\)-Lipschitz operators, the worst case is realized by \(2\times2\) purely rotational (skew) blocks—i.e., oscillatory modes—so global nonexpansiveness is equivalent to controlling a single complex multiplier on those blocks. Instead of reasoning about an arbitrary high-dimensional, possibly non-normal operator, one can design \((k,\alpha,\gamma)\) against a one-parameter family \((1-\alpha)+\alpha(1-i\gamma\omega)^k\) that captures the hard instances. Lemma~2 then promotes this view to a crisp step-size budget \(\Gamma_k^\star(\alpha)\): if \(\gamma L\) lies below this budget, every mode is contractive, and if not, a bilinear convex--concave game witnesses failure—so the condition is both sufficient and essentially tight.

The theorem turns the nonexpansiveness of the LookAhead map into global convergence and gives an \(O(1/T)\) ergodic rate for the primal--dual gap, highlighting a clean separation of roles: \(k\) and \(\alpha\) shape the frequency response (how aggressively oscillatory modes are damped), while \(\gamma\) trades stability margin against speed within the allowed budget. Conceptually, LookAhead does not alter the solution set but reweights time for each mode—attenuating rotations without over-damping the non-oscillatory directions—thereby enlarging the admissible range of \(\gamma\) relative to plain GD while preserving fixed points. Practically, this yields a design rule: pick \(\alpha \le 1-\tfrac{1}{k}\) to guarantee a nontrivial stability margin, choose \(k\) to target the dominant oscillatory band, and set \(\gamma\) up to the certified \(\Gamma_k^\star(\alpha)/L\). We use these guiding principles to design Modal LookAhead (MoLA), such that we can choose the best set of hyperparameters with minimal computational overhead.

We next present a couple of corollaries that can be derived from \ref{lem:conv-cond}, which highlight the entangled behavior of the hyperparameters of LA, and how one influences the other. 

\begin{corollary}[Behavior of $k$ vs $\alpha$]\label{cor:alpha-vs-k}
Fix an inner stepsize budget $\Gamma:=\gamma L>0$. Consider the class–uniform nonexpansiveness envelope
\[
\alpha \;\le\; \frac{2}{\,1+(1+\Gamma^2)^{k/2}\,}
\]
Define $\alpha_{\max}(\Gamma,k):=\dfrac{2}{1+(1+\Gamma^2)^{k/2}}$. Then for every fixed $\Gamma>0$ the map $k\mapsto \alpha_{\max}(\Gamma,k)$ is strictly decreasing on $\{2,3,\dots\}$ and 
\[
\lim_{k\to\infty}\alpha_{\max}(\Gamma,k)=0.
\]
In particular, holding $\Gamma=\gamma L$ fixed and increasing the depth $k$ necessarily tightens the admissible range of $\alpha$.
\end{corollary}

\begin{proof}
Set $g(k):=(1+\Gamma^2)^{k/2}=e^{(\frac{k}{2})\log(1+\Gamma^2)}$. Since $\log(1+\Gamma^2)>0$, $g(k)$ is strictly increasing in $k$. Hence $k\mapsto \frac{1}{1+g(k)}$ is strictly decreasing, and so is $\alpha_{\max}(\Gamma,k)=\frac{2}{1+g(k)}$. As $k\to\infty$, $g(k)\to\infty$, whence $\alpha_{\max}(\Gamma,k)\to 0$.
\end{proof}

\begin{corollary}[Behavior of $k$ vs $\gamma$]\label{cor:gamma-vs-k}
Fix $\alpha\in(0,1)$. The class–uniform envelope implies the necessary bound
\[
\alpha \;\le\; \frac{2}{\,1+(1+\Gamma^2)^{k/2}\,}
\qquad\Longleftrightarrow\qquad
\Gamma \;\le\; \Gamma_{\max}(\alpha,k)
:= \sqrt{\Big(\frac{2-\alpha}{\alpha}\Big)^{\!2/k}-1}
\]
Then for every fixed $\alpha\in(0,1)$ the map $k\mapsto \Gamma_{\max}(\alpha,k)$ is strictly decreasing on $\{2,3,\dots\}$ and satisfies the asymptotics
\[
\Gamma_{\max}(\alpha,k)\;=\; \sqrt{\frac{2}{k}\,\log\!\Big(\frac{2-\alpha}{\alpha}\Big)}\;\cdot\,(1+o(1))
\qquad (k\to\infty).
\]
Consequently, the maximal admissible stepsize obeys
\[
\gamma_{\max}(\alpha,k)=\frac{\Gamma_{\max}(\alpha,k)}{L}
\;=\; \Theta\!\Big(\frac{1}{\sqrt{k}}\Big)\qquad (k\to\infty).
\]
\end{corollary}

\begin{proof}
Let $c_\alpha:=\frac{2-\alpha}{\alpha}>1$. Then
\[
\Gamma_{\max}(\alpha,k)
= \sqrt{c_\alpha^{\,2/k}-1}
= \sqrt{\exp\!\Big(\frac{2\log c_\alpha}{k}\Big)-1}\,.
\]
Since $x\mapsto e^{x}$ is strictly increasing, $k\mapsto c_\alpha^{\,2/k}$ is strictly decreasing with $\lim_{k\to\infty} c_\alpha^{\,2/k}=1$, hence $k\mapsto \Gamma_{\max}(\alpha,k)$ is strictly decreasing with limit $0$. The expansion $e^{x}=1+x+o(x)$ as $x\downarrow 0$ yields
\[
c_\alpha^{\,2/k}-1
= \exp\!\Big(\frac{2\log c_\alpha}{k}\Big)-1
= \frac{2\log c_\alpha}{k}+o\!\Big(\frac{1}{k}\Big)\,,
\]
so
\[
\Gamma_{\max}(\alpha,k)
= \sqrt{\frac{2\log c_\alpha}{k}}\,(1+o(1))
= \sqrt{\frac{2}{k}\,\log\!\Big(\frac{2-\alpha}{\alpha}\Big)}\,(1+o(1))\,.
\]
Dividing by $L$ gives the stepsize statement.
\end{proof}
\subsection{Convergence Rate of $O(\frac{1}{T})$ for the average iterate of Modal LookAhead}
\label{sec:mola-rate}
\label{thm:mola_rate}

We restate the full Theorem 2 as follows: 
\begin{theoremrep}[Full statement of Theorem \ref{theo:2}]\label{thm:mola_opt}
Let $f\colon D_x\times D_y\to \mathbb{R}$ be convex in $\vx$ and concave in $\vy$, and
\[
F(\vz)=\begin{bmatrix}\nabla_x f(\vx,\vy)\\ -\nabla_y f(\vx,\vy)\end{bmatrix},\qquad \vz=(\vx,\vy).
\]
 By convexity-concavity of $f$, $F\colon D \to \mathbb{R}^d$ is a monotone operator. Assume $F$ is $L$-smooth, and the saddle set $Z^\star:=\{\vp:\ F(\vp)=\mathbf{0}\}$ is nonempty.
For $k\ge 2$ and $\alpha\in(0,1)$, let
\[
\vmu_k(c;\alpha):=(1-\alpha)+\alpha(1-ic)^{k},\qquad
\Gamma_k^\star(\alpha):=\sup\Big\{\Gamma\ge 0:\ \sup_{c\in[0,\Gamma]}|\vmu_k(c;\alpha)|\le 1\Big\}.
\]
Let the MoLA hyperparameters be chosen by
\[
(k^\star,\alpha^\star)\in\arg\max_{k\ge 2,\ \alpha\in(0,1-1/k]}\ \alpha\,\Gamma_k^\star(\alpha),
\qquad
\gamma^\star\ :=\ \frac{\Gamma_{k^\star}^\star(\alpha^\star)}{L}.
\]
Then the MoLA iterates $\vz_{t+1}=F^{\mathrm{LA}}_{k^\star,\alpha^\star}(\vz_t)$ and the average $\bar\vz_T:=\frac{1}{T}\sum_{t=0}^{T-1}\vz_t=(\bar\vx_T,\bar\vy_T)$ satisfy, for every $\vp=(\vx^\star,\vy^\star)\in Z^\star$,
\begin{equation}\label{eq:mola-main-bound}
f(\bar\vx_T,\vy^\star)-f(\vx^\star,\bar\vy_T)\ \le\ \frac{1}{\,2\,\alpha^\star\,\gamma^\star\,T\,}\,\|\vz_0-\vp\|^2
\ =\ \frac{L}{\,2\,\alpha^\star\,\Gamma_{k^\star}^\star(\alpha^\star)\,}\cdot\frac{\|\vz_0-\vp\|^2}{T} \,\,.
\end{equation}
Moreover, for any baseline $(k_0,\alpha_0)$ with $\alpha_0\in(0,1-1/k_0]$ and stepsize $\gamma_0=\Gamma_{k_0}^\star(\alpha_0)/L$, MoLA attains a better constant:
\begin{equation}\label{eq:mola-improvement}
\frac{L}{2\,\alpha^\star\,\Gamma_{k^\star}^\star(\alpha^\star)}
\ \le\
\frac{L}{2\,\alpha_0\,\Gamma_{k_0}^\star(\alpha_0)}
\end{equation}
\end{theoremrep}

\begin{proof}
Fix any $\vp\in Z^\star$. With $(k^\star,\alpha^\star,\gamma^\star)$ as stated, set $F^{\mathrm{GD}}:=\vI-\gamma^\star F$ and $F^{\mathrm{LA}}:=F^{\mathrm{LA}}_{k^\star,\alpha^\star}$. 
Exactly as in the proof of Theorem~\ref{thm:mola_rate}, define the inner chain
\[
\vz_t^{(0)}:=\vz_t,\qquad \vz_t^{(j+1)}:=F^{\mathrm{GD}}(\vz_t^{(j)})=\vz_t^{(j)}-\gamma^\star F(\vz_t^{(j)}),\quad j=0,\dots,k^\star\!-\!1,
\]
so that $\vz_{t+1}=(1-\alpha^\star)\vz_t+\alpha^\star \vz_t^{(k^\star)}$. The \emph{same} identities as in \eqref{eq:inner-one}–\eqref{eq:key-equality} hold verbatim with $(k,\alpha,\gamma)$ replaced by $(k^\star,\alpha^\star,\gamma^\star)$:
\begin{align*}
2\alpha^\star\gamma^\star\sum_{j=0}^{k^\star-1}\!\langle F(\vz_t^{(j)}),\,\vz_t^{(j)}-\vp\rangle
&=\ \big(\|\vz_t-\vp\|^2-\|\vz_{t+1}-\vp\|^2\big)\\
&\quad+\ \alpha^\star(\gamma^\star)^2\sum_{j=0}^{k^\star-1}\!\|F(\vz_t^{(j)})\|^2\ -\ \alpha^\star(1-\alpha^\star)\|\vz_t^{(k^\star)}-\vz_t\|^2.
\end{align*}
Let $\mathbf{S}_t:=\sum_{j=0}^{k^\star-1}F(\vz_t^{(j)})$ and $\boldsymbol{\Delta}_t:=\vz_t^{(k^\star)}-\vz_t=\gamma^\star\mathbf{S}_t$. 
By Cauchy–Schwarz, $\sum_{j=0}^{k^\star-1}\|F(\vz_t^{(j)})\|^2\ge \tfrac{1}{k^\star}\|\mathbf{S}_t\|^2$, hence for $\alpha^\star\le 1-\tfrac{1}{k^\star}$ the remainder is nonpositive and we obtain the one-step descent
\begin{equation}\label{eq:mola-descent}
2\alpha^\star\gamma^\star\sum_{j=0}^{k^\star-1}\!\langle F(\vz_t^{(j)}),\,\vz_t^{(j)}-\vp\rangle
\ \le\ \|\vz_t-\vp\|^2-\|\vz_{t+1}-\vp\|^2.
\end{equation}
Summing \eqref{eq:mola-descent} over $t=0,\dots,T-1$ and using monotonicity of $F$ as in Theorem~\ref{thm:mola_rate} yields
\[
\frac{1}{T}\sum_{t=0}^{T-1}\langle F(\vz_t),\vz_t-\vp\rangle\ \le\ \frac{\|\vz_0-\vp\|^2}{2\,\alpha^\star\,\gamma^\star\,T}.
\]
By convexity–concavity of $f$,
\[
f(\bar\vx_T,\vy^\star)-f(\vx^\star,\bar\vy_T)
\ \le\ \frac{1}{T}\sum_{t=0}^{T-1}\langle F(\vz_t),\vz_t-\vp\rangle
\ \le\ \frac{\|\vz_0-\vp\|^2}{2\,\alpha^\star\,\gamma^\star\,T}.
\]
Finally, with the MoLA stepsize choice $\gamma^\star=\Gamma_{k^\star}^\star(\alpha^\star)/L$, this becomes \eqref{eq:mola-main-bound}.Taking the supremum of both sides of \eqref{eq:final-gap} over $\vx \in D_x$ and $\vy \in D_y$ yields the upper bound on the restricted primal-dual gap (Appendix ~\ref{app:gap_function}). 

For the improvement claim \eqref{eq:mola-improvement}, observe that Theorem~\ref{thm:mola_rate} applied to any admissible $(k,\alpha,\gamma=\Gamma_k^\star(\alpha)/L)$ yields the constant $L/(2\,\alpha\,\Gamma_k^\star(\alpha))$. 
By the MoLA selection rule, $\alpha^\star\,\Gamma_{k^\star}^\star(\alpha^\star)\ge \alpha_0\,\Gamma_{k_0}^\star(\alpha_0)$ for any baseline $(k_0,\alpha_0)$, giving the stated domination of constants, with strict inequality unless $(k_0,\alpha_0)$ is itself a maximizer.
\end{proof}

\subsection{Modal Geometry and Exclusion of Non-convergent Rotations}
\label{sec:modal-geom}
Our convergence analysis uses the fact that the modulus of the dominant mode lies within the unit ball. Our choice of $k$ and $\alpha$ is aligned to maximize contraction along this modal direction. However, this might push certain other modes into regions where the averaging cannot pull the mode back inside the unit ball. We define the problem formally, and show that for our constraints on the hyperparameters, modes never get rotated such that they enter these ``non-converging regions'' and our LookAhead operator uniformly contracts all modes within the unit ball. We illustrate this in Fig.~ \ref{fig:modes_illustration} how MoLA chooses hyperparameters that avoid this phenomenon.

After $k$ inner steps of Gradient Descent Ascent, a mode $c$ maps to
$z_k(c):=(F^{\mathrm{GD}}(c))^k$, whose real part is
\[
\phi_k(c)\ :=\ \Re\!\bigl((F^{\mathrm{GD}}(c))^k\bigr).
\]

\begin{lemma}
\label{lem:phiforms}
For every integer $k\ge2$ and $c\ge0$,
\[
\phi_k(c)=\sum_{j=0}^{\lfloor k/2\rfloor}(-1)^j\binom{k}{2j}c^{2j}
=(1+c^2)^{k/2}\cos\!\bigl(k\arctan c\bigr).
\]
\end{lemma}

\begin{proof}
By the binomial theorem,
\[
(1-ic)^k=\sum_{m=0}^k \binom{k}{m}(-i)^m c^m.
\]
Only even $m=2j$ contribute to the real part, since $(-i)^{2j}=(-1)^j$ is real. Thus
\[
\Re\bigl((F^{\mathrm{GD}}(c))^k\bigr)=\sum_{j=0}^{\lfloor k/2\rfloor}(-1)^j\binom{k}{2j}c^{2j}.
\]
For the trigonometric form, write $1-ic=r e^{i\theta}$ with $r=\sqrt{1+c^2}$ and $\theta=-\arctan c$. Then
\[
(1-ic)^k=(1+c^2)^{k/2} e^{ik\theta},\quad
\Re(\cdot)=(1+c^2)^{k/2}\cos(k\theta)=(1+c^2)^{k/2}\cos(k\arctan c).
\]
\end{proof}

\begin{lemma}
\label{lem:taylor}
For $k\ge2$,
\[
\phi_k(0)=1,\qquad \phi_k''(0)=-k(k-1)<0.
\]
\end{lemma}

\begin{proof}
From Lemma~\ref{lem:phiforms}, $\phi_k(c)=\sum_{j=0}^{\lfloor k/2\rfloor}(-1)^j\binom{k}{2j}c^{2j}$. At $c=0$, only $j=0$ contributes, giving $\phi_k(0)=1$. Differentiating twice,
\[
\phi_k''(c)=\sum_{j=1}^{\lfloor k/2\rfloor}(-1)^j\binom{k}{2j}(2j)(2j-1)c^{2j-2}.
\]
At $c=0$, only the $j=1$ term survives:
\[
\phi_k''(0)=(-1)^1\binom{k}{2}(2)(1)=-2\binom{k}{2}=-k(k-1).
\]
\end{proof}

\begin{proposition}[Geometric exclusion of the impermeable half-plane]
\label{prop:exclusion}
Let the averaging-impermeable half-plane $\mathcal{H}:=\{z\in\mathbb{C}:\Re z\ge1\}$.
For each $k\ge2$, set
\[
C_k:=\sup\{r>0:\ \phi_k(c)<1\ \text{for all}\ c\in(0,r]\}\in(0,\infty].
\]
Then for every stepsize $\gamma$ with $\gamma L\le C_k$, one has $\Re((F^{\mathrm{GD}}(c))^k)<1$ for all modes $c\in[0,\gamma L]$, hence $z_k(c)\notin\mathcal{H}$ for any mode.
\end{proposition}

\begin{proof}
By Lemma~\ref{lem:taylor}, $c=0$ is a strict local maximum of $\phi_k$ with value $1$, so there exists $\varepsilon>0$ with $\phi_k(c)<1$ for $c\in(0,\varepsilon]$. The definition of $C_k$ ensures $\phi_k(c)<1$ on $(0,\Gamma]$ whenever $\Gamma=\gamma L\le C_k$. Thus, no mode is rotated into $\mathcal{H}$.
\end{proof}

\subsection{Strongly Convex - Strongly Concave Games}
In this section, we establish the convergence of LookAhead for strongly convex–strongly concave (SC–SC) games, a structured subclass of monotone variational problems. This theoretical analysis complements our empirical results, which demonstrate that MoLA consistently outperforms competing methods on both SC–SC and bilinear game settings.

Let us consider the following class of strongly-convex, strongly-concave saddle point problems with bilinear coupling -
\begin{equation}
\label{eq:games-class}
\tag{SC-SC}
\min_{\substack{\vx \in D_x }} \max_{\substack{\vy \in D_y}} \, \mathcal{L}(\vx,\vy)\, = \,P(\vx) + \langle A\vx, \vy\rangle \,- Q(\vy)   
\end{equation}
where we assume that $P\colon D_x \to \mathbb{R}$ is $\mu_P$-strongly convex and $L_P$ smooth and $Q\colon D_y \to \mathbb{R}$ is $\mu_Q$-strongly convex and $L_Q$ smooth, and $A \in D_x \times D_y \subseteq D$ is the fixed bilinear coupling matrix. 

The associated vector field $F\colon D \rightarrow \mathbb{R}^{d}$ and its Jacobian $JF\colon D \rightarrow \mathbb{R}^{d \times d}$ are as follows 
\begin{equation}
    F = \begin{bmatrix}
        \nabla_\vx \,\mathcal{L} \\
        - \nabla_\vy \,\mathcal{L}
    \end{bmatrix} \, = \, \begin{bmatrix}
        \nabla P(\vx) + \vA^{\top} \vy \\
        -\nabla G(\vy) + \vA\vx
    \end{bmatrix} \quad \quad \text{and} \quad JF = \begin{bmatrix}
        \nabla_\vx^2 \,\mathcal{L} && \nabla_\vx\nabla_\vy \mathcal{L} \\ 
        -\nabla_\vy\nabla_\vx\mathcal{L} && -\nabla_\vy^2 \,\mathcal{L}
    \end{bmatrix} \, = \, \begin{bmatrix}
        \nabla^2 P(\vx) && \vA^\top \\ 
        -\vA && \nabla^2 G(\vy)
    \end{bmatrix}
\end{equation}
\begin{proposition}
\label{prop:mono}
    The operator $F$ is $\mu$-strongly monotone and $L$-Lipschitz, where $\mu = \min\, \{ \mu_P ,\,\mu_Q\}$ and $L = \sqrt{\max \, \{L_{P}^2, L_Q^2 \} + \| 
    \vA \|^2}$
\end{proposition}
\begin{proof} Let $\vz_1 \, , \vz_2 \in D$ be any two arbitrary points in the domain, where $\vz_1=(\vx_1,\vy_1)$ and $\vz_2=(\vx_2,\vy_2)$, and let $\Delta \vz:=\vz_1-\vz_2=(\Delta \vx,\Delta \vy)$.

\paragraph{Strong monotonicity.}
By the definition of $F$,
\[
\langle F(\vz_1)-F(\vz_2),\,\Delta \vz\rangle
=\langle \nabla P(\vx_1)-\nabla P(\vx_2),\,\Delta \vx\rangle
-\langle \nabla Q(\vy_1)-\nabla Q(\vy_2),\,\Delta \vy\rangle
+\langle A^\top\Delta \vy,\,\Delta \vx\rangle+\langle A\Delta \vx,\,\Delta \vy\rangle.
\]
The bilinear cross terms cancel since
$\langle A^\top\Delta \vy,\,\Delta \vx\rangle+\langle A\Delta \vx,\,\Delta \vy\rangle
=\Delta \vy^\top A\Delta \vx+\Delta \vx^\top A^\top\Delta \vy=0$.
By $\mu_P$- and $\mu_Q$-strong convexity,
\[
\langle \nabla P(\vx_1)-\nabla P(\vx_2),\,\Delta \vx\rangle\ge \mu_P\|\Delta \vx\|^2,\qquad
\langle \nabla Q(\vy_1)-\nabla Q(\vy_2),\,\Delta \vy\rangle\ge \mu_Q\|\Delta \vy\|^2.
\]
Hence
\[
\langle F(\vz_1)-F(\vz_2),\,\Delta \vz\rangle
\;\ge\; \mu_P\|\Delta \vx\|^2+\mu_Q\|\Delta \vy\|^2
\;\ge\; \min\{\mu_P,\mu_Q\}\,\big(\|\Delta \vx\|^2+\|\Delta \vy\|^2\big)
\;=\; \mu\,\|\Delta \vz\|^2.
\]

\paragraph{Lipschitz continuity.}
Using the triangle inequality,
\begin{align*}
\|F(\vz_1)-F(\vz_2)\|^2
&\le \|\nabla P(\vx_1)-\nabla P(\vx_2)\|^2
   +\|\nabla Q(\vy_1)-\nabla Q(\vy_2)\|^2
   +\|A^\top\Delta \vy\|^2+\|A\Delta \vx\|^2 \\
&\le L_P^2\|\Delta \vx\|^2+L_Q^2\|\Delta \vy\|^2+\|A\|^2\|\Delta \vy\|^2+\|A\|^2\|\Delta \vx\|^2 \\
&\le \big(\max\{L_P^2,L_Q^2\}+\|A\|^2\big)\,\big(\|\Delta \vx\|^2+\|\Delta \vy\|^2\big)
= L^2\,\|\Delta \vz\|^2,
\end{align*}
which yields $\|F(\vz_1)-F(\vz_2)\|\le L\,\|\Delta \vz\|$. Thus $F$ is $\mu$-strongly monotone and $L$-Lipschitz.
\end{proof}

We first present the proof of convergence of the LookAhead discrete dynamics for the problem. We then present the analysis in the discrete complex plane using the $Z$-transform, and finally, the design of the algorithm that allows us to tune our hyperparameters near-optimally.

\subsection{Convergence of LookAhead for Strongly-Convex Strongly-Concave Games}
\label{sec:scsc-conv}
In this subsection, we prove that for a sufficiently small stepsize $\gamma$ of the base optimizer (in our case ,Gradient Descent), the LookAhead dynamics converges to the equilibrium of the \ref{eq:games-class}.

\begin{proposition}
    If $0<\gamma < \frac{2\mu}{L^2}$, then the LookAhead map $F^{\mathrm{LA}}_{k,\alpha}$ converges to a unique fixed point, which is the game equilibrium.
\end{proposition}
\begin{proof}
Let $\vz_1,  \vz_2 \in D$ be any two arbitrary points in the domain. Let $\Delta = \vz_1 - \vz_2$ and $g = F(\vz_1) - F(\vz_2)$. Then 
\begin{align}
    F^{\mathrm{GD}}(\vz_1) - F^{\mathrm{GD}}(\vz_2) = (\vz_1 - \gamma F(\vz_1)) - (\vz_2 - \gamma F(\vz_2)) = \Delta-\gamma g \,.
\end{align}

Squaring both sides gives
\begin{align*}
    \|F^{\mathrm{GD}}(\vz_1) - F^{\mathrm{GD}}(\vz_2)\|^2 &= \|\Delta-\gamma g\|^2 = \|\Delta\|^2 - 2\gamma \langle \Delta, g \rangle + \gamma^2 \|g\|^2 \\
    &\leq \|\Delta\|^2  -2\gamma\mu \|\Delta\|^2 + \gamma^2 L^2 \|\Delta\|^2  \quad \quad \quad  \quad \textit{(by Proposition ~\ref{prop:mono})}\\
    & = (1-2\gamma \mu + \gamma^2 L^2 )\, \|\Delta\|^2 \,.
\end{align*}

If $\gamma < \frac{2\mu}{L^2}$, the $(1-2\gamma \mu + \gamma^2 L^2 ) \in (0,1)$, hence $F^{\mathrm{GD}}$ is a contractive operator
\begin{align*}
    \| F^{\mathrm{GD}}(\vz_1) - F^{\mathrm{GD}}(\vz_2)\| \leq (1-2\gamma \mu + \gamma^2 L^2 ) \|\vz_1 - \vz_2\| \,.
\end{align*}

By induction, after $k$ steps of the base optimizer
\begin{align*}
    \|(F^{\mathrm{GD}})^k (\vz_1) -  (F^{\mathrm{GD}})^k (\vz_2)\| \leq (1-2\gamma \mu + \gamma^2 L^2 )^k \|\vz_1 - \vz_2\| \,.
\end{align*}

The $k$-th step of LookAhead performs the following averaging 
\begin{align}
     F^{\mathrm{LA}}(\vz) = (1-\alpha)\cdot \mathbf{1} \,+ \, \alpha \, ( F^{\mathrm{GD}})^k (\vz) \,,
\end{align}
where $\mathbf{1}$ is the vector of ones of dimension $d$. Hence
\begin{align*}
    \| F^{\mathrm{LA}}(\vz_1) -  F^{\mathrm{LA}}(\vz_2)\| 
    &\leq ( 1-\alpha) \|\vz_1 - \vz_2 \| 
        + \alpha \, \| ( F^{\mathrm{GD}})^k(\vz_1) - ( F^{\mathrm{GD}})^k(\vz_2)\| \\
    &\leq 
    \underbrace{\big((1-\alpha)\cdot \mathbf{1} 
    + \alpha (1 - 2\gamma \mu + \gamma^2 L^2)^k\big)}_{\displaystyle \rho}
    \,\|\vz_1 - \vz_2 \| \,.
\end{align*}

Since $\alpha \in (0,1)$ and  $(1-2\gamma \mu + \gamma^2 L^2 ) \in (0,1)$, hence $\rho < 1$. Therefore, $ F^{\mathrm{LA}}(\cdot)$ is a global contraction on $D$. By Banach's fixed point theorem, $ F^{\mathrm{LA}}(\cdot)$ has a unique fixed point $\vz^{*}$, and the iteration $\vz_{t+1} =  F^{\mathrm{LA}}(\vz_t)$ converges to $\vz^*$ with a linear rate $\rho$. Since $ F^{\mathrm{LA}}$ and $ F^{\mathrm{GD}}$ have the same fixed point by definition, $ F^{\mathrm{LA}}$ converges to the global equilibrium of the game $\vz^*$, where $F(\vz^*) = 0$.
\end{proof}

\subsection{Discrete Frequency Domain Analysis of LookAhead}
In this section, we motivate the use of $Z$-transform (~\ref{eq:Ztransform}) as a competent tool to analyze the discrete-time dynamics of LookAhead. The $Z$-transform provides a natural bridge between time-domain recursion and frequency-domain stability, enabling us to characterize the poles of the update operator and hence its discrete-time stability.

We consider the update step of Gradient Descent Ascent as: $\vz_{t+1} \leftarrow F^{\mathrm{GD}} \, \vz_t$, where $T$ is the vector field defined in ~Section\ref{sec:scsc-conv}, with initialization at $\vz_0$. This is a linear time-invariant (LTI) system. To analyze the stability of this system, we examine the dynamics along the eigen directions of $F$.

Let $(\lambda, v)$ be an eigen pair of $F$. We project the state onto the basis of the eigenvectors to analyze the dynamics of each component separately. Let $u_t := \langle v, z_t \rangle$ be a scalar sequence, which represents the projection of $z_t$ onto the eigen direction $v$. Using the GDA update rule, we get
\begin{align}
\label{eq:app-$Z$-transform}
    u_{t+1} = \langle v, z_{t+1} \rangle = F^{\mathrm{GD}} \,\langle v, z_t \rangle  = F^{\mathrm{GD}} \,u_t \,.
\end{align}

Let $U(z)$ be the unilateral $Z$-transform of the sequence $u_t$. Then the $Z$-transform of \ref{eq:app-$Z$-transform} is given by 
\begin{align}
    &z \, U(z) -u_0 = F^{\mathrm{GD}}\, U(z)\\
    &\implies U(z) = \frac{u_0}{z-F^{\mathrm{GD}}}  \,.
\end{align}
Hence, the pole for the inner GDA optimizer lies at $z= F^{\mathrm{GD}}$.
Projecting the LookAhead step onto the eigen directions yields 
\begin{align}
   U(z) = \frac{u_0}{z - \big( (1-\alpha) + \alpha \, (F^{\mathrm{GD}})^k\big)}   \,.
\end{align}
Hence the poles of the LookAhead update step lie at $z = (1-\alpha) + \alpha \, (F^{\mathrm{GD}})^k$

Discrete stability of the LookAhead iterate is \emph{equivalent} to all the poles lying strictly inside the unit disk:
\[
|\mu| \;<\; 1 \,.
\]

Let $w = (F^{\mathrm{GD}})^k$ where $w \in \mathbb{C}^d$. Let $\Delta:=w-1$. Then $\mu =1+\alpha\Delta$, and
\begin{align*}
|\mu|^2
&= (1+\alpha\Delta)(1+\alpha\overline{\Delta})
= 1 + 2\alpha\,\Re\Delta + \alpha^2|\Delta|^2 \,.
\end{align*}
The stability inequality $|\mu|<1$ is therefore
\[
1 + 2\alpha\,\Re\Delta + \alpha^2|\Delta|^2 \;<\; 1
\;\;\Longleftrightarrow\;\;
\alpha\Big(2\,\Re\Delta + \alpha\,|\Delta|^2\Big) \;<\; 0\,.
\]
For $\alpha>0$ this holds iff
\[
2\,\Re\Delta + \alpha\,|\Delta|^2 \;<\; 0
\quad\Longleftrightarrow\quad
\alpha \;<\; -\,\frac{2\,\Re\Delta}{|\Delta|^2}.
\]
Because $\Re\Delta=\Re(w-1)=\Re w - 1$, a positive bound exists iff $\Re w<1$, and the \emph{exact} critical averaging is
\begin{equation}
\label{eq:alpha-max}
\;
\alpha_{\max}
= \frac{2\big(1-\Re w\big)}{|\,1-w\,|^2}
= \frac{2\big(1-\Re((F^{\mathrm{GD}})^{\,k})\big)}{\,|1-(F^{\mathrm{GD}})^{\,k}|^2\,} \,.
\;
\end{equation}
Thus, the modes are Schur-stable precisely when
\begin{equation}
    0 \;<\; \alpha \;<\; \min\{\,1,\ \alpha_{\max}\,\} \,.
\end{equation}

\paragraph{Geometric Interpretation.}
Let $\lambda \in \sigma(JF)$. We rewrite the GDA iterate in polar coordinates
\[
T \;=\; R\,e^{-i\theta}, \qquad
R:=|T|=\sqrt{\big(1-\gamma\,\Re\lambda\big)^2+\big(\gamma\,\Im\lambda\big)^2},\quad
\theta:=\arg(T)=\operatorname{atan2}(\gamma\,\Im\lambda,\ 1-\gamma\,\Re\lambda).
\]
Then $w=(F^{\mathrm{GD}})^{\,k}=R^{k}e^{-ik\theta}$, so
\[
\Re w = R^{k}\cos(k\theta),
\qquad
|1-w|^2 = 1+R^{2k}-2R^{k}\cos(k\theta).
\]
Substituting into the formula above 
\[
\alpha_{\max}
= \frac{2\big(1-R^{k}\cos(k\theta)\big)}{\,1+R^{2k}-2R^{k}\cos(k\theta)\,}.
\]

Since stability must hold for every eigenmode, the LookAhead step is Schur-stable iff
\[
0<\alpha \;<\; \alpha_{\max}^{\mathrm{all}}(k)
\;:=\;\min_{\lambda\in\sigma(J)}\ \alpha_{\max}
\]

\subsection{Algorithm Design for optimal choice of hyperparameters}
\label{app:calc-mola}

In this subsection, we discuss the mechanism of Modal LookAhead in detail. In the previous subsection, we showed that the LookAhead step with averaging parameter $\alpha\in(0,1]$ has poles at
\[
\mu(k) \;=\; (1-\alpha) + \alpha w \;=\; (1-\alpha)+ \alpha R^k e^{-i k \theta},
\]
The \emph{modal contraction factor} is
\[
\rho(k) \;=\; |\mu(k)| 
= \sqrt{(1-\alpha)^2 + 2\alpha(1-\alpha) R^k \cos(k\theta) + \alpha^2 R^{2k}}.
\]

We minimize the modal contraction factor to dampen the oscillations usually observed in games.

We rewrite the dynamics in terms of the \emph{amplitude} $s:=R^k$ and the \emph{phase} $\phi:=k\theta$. In these coordinates,
\[
\rho^2(k) \equiv \Phi(s,\phi) \;=\; (1-\alpha)^2 + 2\alpha(1-\alpha)\,s\cos\phi + \alpha^2 s^2,
\]
subject to the coupling constraint $\phi = (\theta/\ln R)\,\ln s$, since both $s$ and $\phi$ derive from the same integer~$k$.

We observe the following: 
First, for any fixed $s>0$, the contraction $\Phi(s,\phi)$ is minimized when $\cos\phi=-1$, i.e.\ when $\phi=\pi \ (\mathrm{mod}\,2\pi)$, yielding
\[
\Phi_{\min}(s) \;=\; (1-\alpha-\alpha s)^2.
\]
Second, at $\phi=\pi$, this quadratic is minimized by choosing
\[
s^\star = \tfrac{1-\alpha}{\alpha} \,,
\]
which would cancel the mode exactly ($\mu=0$) if $s$ and $\phi$ could be chosen independently.

Since the amplitude and the phase are co-dependent on $k$, it is hard to choose both optimally simultaneously. Therefore, we choose to select $k$ so that $k\theta$ lands as close as possible to the ``phase targets'' $\pi+2\pi m$, while at the same time making $R^k$ close to the amplitude target $s^\star$. Explicitly, the candidate indices are
\[
k_{\mathrm{phase}}(m) \;=\; \frac{\pi(2m+1)}{\theta}, \qquad 
k_{\mathrm{amp}} \;=\; \frac{\ln((1-\alpha)/\alpha)}{\ln R} \,.
\]
We select
\[
m^\star := \arg\min_{m\ge 0}\ \big|\,k_{\mathrm{phase}}(m) - k_{\mathrm{amp}}\,\big|, 
\qquad k^\circ := k_{\mathrm{phase}}(m^\star) \,,
\]
and then take
\[
k^\star \;\in\; \{\lfloor k^\circ \rfloor,\, \lceil k^\circ \rceil\} \,,
\]
choosing the integer that minimizes the exact $\rho(k)$. We illustrate the significance of this choice of $k$ by plotting the trajectory between MoLA and LookAhead with randomly chosen $k$ and $\alpha$ in Figure \ref{fig:traj_illustration}.

This rule balances phase alignment with amplitude matching. Phase alignment ensures that oscillatory terms contribute negatively, minimizing the cross term, while amplitude matching approximately satisfies $\alpha R^k \approx 1-\alpha$, which drives the contraction factor toward zero. In the quadratic case, $\lambda$ has a closed form and $R,\theta$ are explicit. For general strongly convex–strongly concave games, one may estimate $(R,\theta)$ from local curvature and coupling bounds, and the same principle applies.

\clearpage
\section{Missing Proofs of Theorems in Continuous Time}\label{app:proofs_cont}
In this section, we introduce high-resolution differential equations (HRDEs) as a continuous-time lens on the algorithm. We then derive the complete HRDE for LookAhead. Finally, we set up a Laplace-transform framework for stability analysis that connects the continuous model to discrete parameter choices.

\textbf{Notation} For this section, we assume that the coefficient matrix $\vA$ of \ref{eq:bilinear_game} is positive semi-definite. This is a standard assumption, and the analysis can be easily extended to other matrices.
\subsection{Showing Divergence of \ref{eq:gda-hrde} for \ref{eq:bilinear_game} Through \eqref{eq:Laplace}}\label{app:gda_convergence_analysis}

 As shown in~\citep{chavdarova2023hrdes}, Gradient decent's $\mathcal{O}(\gamma)$-HRDE~\citep{chavdarova2023hrdes} is:
\begin{equation}
\tag{GD-HRDE}
    \ddot{\vz}(t)     = - \frac{2}{\gamma}\cdot \dot{\vz}(t) - \frac{2}{\gamma} \cdot F(\vz(t)) \,,
\label{eq:gda-hrde}
\end{equation}
where $\vz$ represents the vector of players $\vz(t) \triangleq (\vx(t) , \vy(t))^\intercal\,.$
$\vz(t) = \begin{bmatrix}
    \vx(t) \\
    \vy(t)
\end{bmatrix}$ 
and $F(\cdot)$ is the operator defined in \eqref{eq:VI}.

Using the continuous time representation, the following are the joint variable, the gradient field, and the Jacobian of the parameterized Bilinear Game \ref{eq:bilinear_game}:
\begin{equation}
    \vz(t) = \begin{bmatrix}
        \vx(t) \\
        \vy(t)
    \end{bmatrix} \,, 
\end{equation}
\begin{equation}
    F(\vz(t)) = \begin{bmatrix}
     \vA\vy(t) \\
     - \vA\vx(t)\end{bmatrix} \,, \quad \text{and }  \quad  JF(\vz(t))  = \begin{bmatrix}
        0 && \vA  \\
        -\vA && 0
    \end{bmatrix} \,.
\end{equation}
The  \ref{eq:gda-hrde} written separately for the two players is
\begin{equation}
\label{eq:hrde-gda-x}
    \ddot{\vx}(t) =  - \frac{2}{\gamma} \cdot \vA\dot{\vx}(t) - \frac{2}{\gamma} \cdot \vy (t)   \,, 
\end{equation}
\begin{equation}
\label{eq:hrde-gda-y}
    \ddot{\vy}(t) = - \frac{2}{\gamma} \cdot \vA \dot{\vy}(t) + \frac{2}{\gamma} \cdot \vx(t)   \,.
\end{equation}

Taking the Laplace transform of the \eqref{eq:hrde-gda-x} and \eqref{eq:hrde-gda-y} yields:
\begin{equation}
       \mathbf{X}(s) =  -\dfrac{2 }{ \gamma (s^2 \vI+ \frac{2\vA s}{\gamma})} \mathbf{Y}(s) + \dfrac{\dot{\vx}(0)+ \frac{\vA}{\gamma} \vx(0)}{s^2 I + \frac{2\vA s}{\gamma} } + \dfrac{(s + \frac{\vA}{\gamma})\vx(0)}{ s^2 + \frac{2\vA s}{\gamma} } \,,
        \label{eq:freq_gda_X}
\end{equation}
\begin{equation}
         \mathbf{Y}(s) = -\dfrac{2}{ \gamma (s^2 \vI+ \frac{2\vA s}{\gamma})} \mathbf{X}(s) + \dfrac{\dot{\vy}(0)+ \frac{\vA}{\gamma} \vy(0)}{s^2 + \frac{2\vA s}{\gamma} } + \dfrac{(s + \frac{\vA}{\gamma})\vy(0)}{ s^2 + \frac{2\vA s}{\gamma} } \,.
        \label{eq:freq__gda_Y} 
\end{equation}
Here, $s \in \mathbb{C}$ represents the complex frequency, $\mathbf{X}(s) =  \int_{0}^{\infty} x(t) e^{-st} \, dt $, $\mathbf{Y}(s) =  \int_{0}^{\infty} \vy(t) e^{-st} \, dt$ are the transfer functions of the trajectories $\vx(t)$ and $\vy(t)$. We assume that $\vx(t)$ and $\vy(t)$ are defined for $t>0$ \textit{i.e.}, a unilateral Laplace transform.

 Since $\vX(s)$ and $\vY(s)$ are linear equations, we compute the combined transform function for $\vX(s)$. The poles of the combined transform function are $s \in \mathbf{C}$ such that
 \begin{equation}
 \label{eq:gd_condition_roots}
    det(s^4 \vI + \frac{2s^3}{\gamma} (\vI+\vA) + \frac{4s^2}{\gamma^2} \vA - \frac{4}{\gamma^2}\vI) = 0 \,.
 \end{equation}

 We use Routh-Hurwitz criterion to analyze the stability of \ref{eq:gd_condition_roots}. Let $\lambda_i \in \Lambda$ be an eigenvalue of \(\mathbf{A}\).
 
 Since the determinant is zero when any eigenvalue of the matrix polynomial is zero, consider an eigenvalue \( \lambda_i \) of \(\mathbf{A}\). The scalar polynomial corresponding to \(\lambda_i\) is:
\begin{equation}
p_i(s) = s^4 + \frac{2 s^3}{\gamma} (1 + \lambda_i) + \frac{4 s^2}{\gamma^2} \lambda_i - \frac{4}{\gamma^2} = 0 \,.
\end{equation}

We rewrite \(p_i(s)\) as
\begin{equation}
\label{eq:general-polynomial}
p_i(s) = s^4 + a_3 s^3 + a_2 s^2 + a_1 s + a_0 \,,
\end{equation}
where
\begin{equation}
a_3 = \frac{2}{\gamma} (1 + \lambda_i), \quad
a_2 = \frac{4}{\gamma^2} \lambda_i, \quad
a_1 = 0, \quad
a_0 = -\frac{4}{\gamma^2} \,.
\end{equation}

The Routh array for \ref{eq:general-polynomial} is:

\[
\begin{array}{c|cc}
s^4 & 1 & a_2 \\
s^3 & a_3 & a_1 \\
s^2 & b_1 & a_0 \\
s^1 & c_1 & 0 \\
s^0 & a_0 & 0
\end{array} \,,
\]
where
\begin{equation}
b_1 = \frac{a_3 a_2 - 1 \cdot a_1}{a_3} = \frac{a_3 a_2}{a_3} = a_2 \,,
\end{equation}
(since \(a_1=0\)),
\begin{equation}
c_1 = \frac{b_1 a_1 - a_3 a_0}{b_1} = \frac{0 - a_3 a_0}{a_2} = -\frac{a_3 a_0}{a_2} \,.
\end{equation}

The Routh-Hurwitz stability conditions are as follows
\begin{equation}
1 > 0, \quad
a_3 = \frac{2}{\gamma} (1 + \lambda_i) > 0, \quad
b_1 = a_2 = \frac{4}{\gamma^2} \lambda_i > 0 \,,
\end{equation}
\begin{equation}
c_1 = -\frac{a_3 a_0}{a_2} = -\frac{\frac{2}{\gamma} (1 + \lambda_i) \cdot \left(-\frac{4}{\gamma^2}\right)}{\frac{4}{\gamma^2} \lambda_i} = \frac{2 (1 + \lambda_i)}{\gamma} \cdot \frac{1}{\lambda_i} > 0 \,.
\end{equation}

Finally,
\begin{equation}
a_0 = -\frac{4}{\gamma^2} > 0 \,,
\end{equation}
and,
\begin{equation}
a_0 = -\frac{4}{\gamma^2} > 0 \implies -\frac{4}{\gamma^2} > 0 \implies \frac{1}{\gamma^2} < 0 \,,
\end{equation}
which is false. Hence, the system is unstable, and GD diverges for \ref{eq:bilinear_game}, which affirms the behavior of GD for this game.

\subsection{\texorpdfstring{General $\mathcal{O}(\gamma)$--HRDE for LookAhead}{General O(gamma)--HRDE for LookAhead}}

\label{app:la_hrde_derivation}
In this section, we prove Theorem~\ref {theo:3}, i.e., we derive the $O(\gamma)$ HRDE for LookAhead as presented in ~\ref{eq:la-hrde}. This differential equation is a continuous time representation of the discrete LookAhead dynamics (\ref{eq:LookAhead}). By transitioning to the continuous domain, we can leverage the well-established theory of differential equations to analyze the behavior of the algorithm and gain deeper insights into its underlying structure.

We follow the procedure of~\citet{chavdarova2023hrdes}, where we first rewrite the methods in the following general form:
\begin{equation} \tag{GF} \label{eq:general_form}
  \frac{\vz_{n+1}-\vz_n}{\gamma} \!=\! \mathcal{U}(\vz_{n+k}, \dots, \vz_0) \,.
\end{equation}
We introduce the Ansatz $\vz_n \approx \vz(n \cdot \delta)$,  for some smooth
curve $\vz(t)$ defined for $t \ge 0$. A Taylor expansion gives:
\begin{align*}
    \vz_{n+1} &\!\approx\! \vz((n+1) \delta ) = \vz(n\delta) + \dot{\vz} (n\delta)\delta + \frac{1}{2} \ddot{\vz}(n \delta) \delta^2 + \dots \,.
\end{align*}
Thus,  when deriving $\mathcal{O}(\gamma)$--HRDEs, for the nominator of the left-hand side of \ref{eq:general_form} for all the methods considered in this work, we have:
\begin{equation} \tag{LHS-GF}\label{eq:time-taylor} 
    \vz_{n+1} - \vz_{n} \!\approx\!  \dot{\vz} (n\delta)\delta + \frac{1}{2} \ddot{\vz} (n \delta) \delta^2 + \mathcal{O}(\delta^3) \,.
\end{equation}

\paragraph{HRDEs for LA-GD.}
The predicted iterates for $k=1, \dots, 4$, given \ref{eq:gd} as a base optimizer are as follows:
\begin{align*}
    \tilde\vz_{n+1} &= \vz_n - \gamma F(\vz_n)
    \tag{$\tilde\vz_{n+1}^{\text{GDA}}$}\label{eq:predicted_gda_1}\\
    \tilde\vz_{n+2} &= \tilde\vz_{n+1} - \gamma F(\tilde\vz_{n+1}) = \vz_n - \gamma F(\vz_n) - \gamma F(\vz_n - \gamma F (\vz_n) )
    \tag{$\tilde\vz_{n+2}^{\text{GDA}}$}\label{eq:predicted_gda_2}\\
    \tilde\vz_{n+3} &= \tilde\vz_{n+2} - \gamma F(\tilde\vz_{n+2}) \notag\\
    & = \vz_n - \gamma F(\vz_n) - \gamma F(\vz_n - \gamma F (\vz_n) ) 
    - \gamma F\big(
    \vz_n - \gamma F(\vz_n) - \gamma F(\vz_n - \gamma F (\vz_n) ) 
    \big)
    \tag{$\tilde\vz_{n+3}^{\text{GDA}}$}\label{eq:predicted_gda_3}\\
    \tilde\vz_{n+4} &= \tilde \vz_{n+3} - \gamma F(\tilde \vz_{n+3}) \notag\\
    &= \underbrace{
    \vz_n - \gamma F(\vz_n) - \gamma F(\vz_n - \gamma F (\vz_n) ) 
    - \gamma F\big(
    \vz_n - \gamma F(\vz_n) - \gamma F(\vz_n - \gamma F (\vz_n) ) 
    \big)
    }_{\tilde \vz_{n+3}} \notag\\
    & - \gamma F \Big(
    \underbrace{
    \vz_n - \gamma F(\vz_n) - \gamma F(\vz_n - \gamma F (\vz_n) ) 
    - \gamma F\big(
    \vz_n - \gamma F(\vz_n) - \gamma F(\vz_n - \gamma F (\vz_n) ) 
    \big)
    }_{\tilde \vz_{n+3}}
    \Big)
    \tag{$\tilde\vz_{n+4}^{\text{GDA}}$}\label{eq:predicted_gda_4}
\end{align*}

\paragraph{~\ref{eq:LookAhead}2-GDA.}
The iterates of~\ref{eq:LookAhead}2-GDA are obtained as follows:
\begin{align*}
    \vz_{n+1} = \vz_n + \alpha (\tilde\vz_{n+2} - \tilde\vz_n) = \vz_n + \alpha \Big(
    -\gamma F(\vz_n) - \gamma F\big(
    \vz_n - \gamma F(\vz_n)
    \big)
    \Big)\,.
\end{align*}
Using~\eqref{eq:time-taylor}, we get (where $\delta$ and $\gamma$ are the step sizes in time and parameter space, resp.):
\begin{align*}
    \frac{\dot \vz (n\delta) + \frac{1}{2} \delta^2 \ddot{\vz} (n) +\mathcal{O} (\delta^3) }{\gamma}
    = \alpha \Big( - 2F(\vz (n\delta)) + \gamma J(\vz(n\delta))F(\vz(n\delta))
    \Big) \,.
\end{align*}
Setting $\delta \!=\! \gamma$ and keeping the $\mathcal{O}(\gamma)$ terms yields:
\begin{align*}
\dot\vz(t) + \frac{\gamma}{2}\ddot\vz(t) = - 2\alpha F(\vz(t)) + \alpha \gamma \,JF (\vz(t)) \,F(\vz(t)) \,.
\end{align*}

Writing it in phase-space representation, yeilds:
\begin{equation}\tag{LA2-GDA-HRDE}\label{la2-gda-hrde}
\begin{aligned}
    \dot\vz(t) &= \vomega (t)\\
    \dot\vomega(t) &= -\frac{2}{\gamma} \vomega(t) - \frac{4\alpha}{\gamma} F(\vz(t)) + 2\alpha \,JF(\vz(t)) \,F(\vz(t)) \,.
\end{aligned}
\end{equation}

\paragraph{\ref{eq:LookAhead}3-GDA.}
For~\ref{eq:LookAhead}3-GDA using \eqref{eq:predicted_gda_3} we have:
\begin{equation} \label{eq:la3_gda_interm1}
\begin{aligned}
    \vz_{n+1} &= \vz_n + \alpha(\tilde \vz_{n+3} - \vz_n) \\
    &= \vz_n + \alpha\gamma\big[
    - F(\vz_n) - F\big(\vz_n - \gamma F(\vz_n)\big)
    - \underbrace{F\big( \vz_n - \gamma F(\vz_n) - \gamma F(\vz_n - \gamma F(\vz_n)) \big)}_{(\star)}
    \big]
\end{aligned}
\end{equation}
Similarly, by doing TE in coordinate space for ($\star$),  we get:
\begin{align*}
    F\big( \vz_n - \gamma F(\vz_n) - \gamma F(\vz_n - \gamma F(\vz_n)) \big) 
    &= F(\vz_n) - \gamma \,JF (\vz_n)\,F(\vz_n) - \gamma \,JF(\vz_n)\, F\big(\vz_n - \gamma F(\vz_n) \big) + \mathcal{O}(\gamma ^2) \\
    &= F(\vz_n) - 2\gamma \,JF (\vz_n)\,F(\vz_n)  + \mathcal{O}(\gamma ^2) \,,
\end{align*}
wherein the second row we do additional TE of the last term in the preceding row.

Thus using \ref{eq:time-taylor} as well as replacing the above in~\eqref{eq:la3_gda_interm1} we have: 
\begin{align*}
\frac{\dot{\vz}(n\delta)\delta \!+\! \frac{1}{2} \ddot{\vz}(n\delta)\delta^2  + \mathcal{O}(\delta^3)}{\gamma} 
&=
\alpha \Big\{
- 3 F(\vz(n\delta)) + 3\gamma \,JF(\vz(n\delta)) \,F(\vz(n\delta)) + \mathcal{O} (\gamma^2)
  \Big\} \,.
\end{align*}

Setting $\delta \!=\! \gamma$ and keeping the $\mathcal{O}(\gamma)$ terms yields: 
\begin{align*}
\dot\vz(t) + \frac{\gamma}{2}\ddot\vz(t) = - 3\alpha F(\vz(t)) + 3 \alpha \gamma J (\vz(t)) F(\vz(t)) \,,
\end{align*}

Rewriting the above in phase-space gives:
\begin{equation}\tag{LA3-GDA-HRDE}\label{eq:la3-gda_hrde3}
\begin{split}
  \dot{\vz}(t) & = \vomega(t) \\
  \dvom(t)     & = - \frac{2}{\gamma} \vomega(t) - \frac{6\alpha}{\gamma} F(\vz(t)) + 6 \alpha \cdot J(\vz(t)) \cdot F(\vz(t))  \,.
\end{split} 
\end{equation}

\paragraph{~\ref{eq:LookAhead}4-GDA.}
For~\ref{eq:LookAhead}4-GDA, replacing with \eqref{eq:predicted_gda_4} we get:
\begin{equation} 
\begin{aligned}\label{eq:la4_gda_interm1}
    \vz_{n+1} &= \vz_n + \alpha(\tilde \vz_{n+4} - \vz_n) \\
    &= \vz_n + \alpha\gamma\big[\underbrace{
    - F(\vz_n) - F\big(\vz_n - \gamma F(\vz_n)\big)
    - F\big( \vz_n - \gamma F(\vz_n) - \gamma F(\vz_n - \gamma F(\vz_n)) \big)}_{\text{same as for LA3-GDA}}\\
    &- \underbrace{F \Bigg(\vz_n
    - F(\vz_n) - F\big(\vz_n - \gamma F(\vz_n)\big)
    - F\big( \vz_n - \gamma F(\vz_n) - \gamma F(\vz_n - \gamma F(\vz_n)) \big)
    \Bigg)}_{(\star)}
    \big]
\end{aligned}
\end{equation}

Similarly as above for LA3-GDA, by performing consecutive TE in  parameter space for $(\star)$ we have:
\begin{align*}
(\star)= F(\vz_n) - 3\gamma \,JF(\vz_n)\, F(\vz_n) + \mathcal{O} (\gamma^2)\,.
\end{align*}

Thus, using \ref{eq:time-taylor} as well as replacing the above in~\eqref{eq:la4_gda_interm1} we have: 
\begin{align*}
\frac{\dot{\vz}(n\delta)\delta \!+\! \frac{1}{2} \ddot{\vz}(n\delta)\delta^2  + \mathcal{O}(\delta^3)}{\gamma} 
&=
\alpha \Big\{
- 4 F(\vz(n\delta)) + 6\gamma \,JF\, (\vz(n\delta)) F(\vz(n\delta)) + \mathcal{O} (\gamma^2)
\Big\} \,.
\end{align*}

Setting $\delta \!=\! \gamma$ and keeping the $\mathcal{O}(\gamma)$ terms yields: 
\begin{align*}
\dot\vz(t) + \frac{\gamma}{2}\ddot\vz(t) = - 4\alpha F(\vz(t)) + 6 \alpha \gamma\, JF (\vz(t))\, F(\vz(t)) \,.
\end{align*}

Rewriting the above in phase-space gives:
\begin{equation}\tag{LA4-GDA-HRDE}\label{eq:la4-gda_hrde3}
\begin{split}
  \dot{\vz}(t) & = \vomega(t) \\
  \dvom(t)     & = - \frac{2}{\gamma} \vomega(t) - \frac{8\alpha}{\gamma} F(\vz(t)) + 12 \alpha \cdot \,JF(\vz(t)\, \cdot F(\vz(t))  \,.
\end{split} 
\end{equation}

\subsection{Summary: HRDEs of \ref{eq:LookAhead}k-GDA}

Table~\ref{tab:la_gda_hrdes_summary} summarizes the obtained HRDEs for LookAhead-Minmax using GDA as a base optimizer, and generalizes it to any $k$.

\begin{table}[htb]
    \centering
    \begin{tabular}{r|l}
         GDA & 
         $\dot\vz(t) + \frac{\gamma}{2}\ddot\vz(t) = -  F(\vz(t))$  \\
         LA2-GDA & $\dot\vz(t) + \frac{\gamma}{2}\ddot\vz(t) = - 2\alpha \cdot F(\vz(t)) + \alpha \gamma \cdot \,JF\, (\vz(t))  \cdot F(\vz(t))$ \\
         LA3-GDA & $\dot\vz(t) + \frac{\gamma}{2}\ddot\vz(t) = - 3\alpha  \cdot F(\vz(t)) + 3 \alpha \gamma  \cdot \,JF\, (\vz(t))  \cdot F(\vz(t))$\\
         LA4-GDA & 
         $\dot\vz(t) + \frac{\gamma}{2}\ddot\vz(t) = - 4\alpha  \cdot F(\vz(t)) + 6 \alpha \gamma  \cdot \,JF\, (\vz(t)) \cdot  F(\vz(t))$\\
         LA5-GDA & 
         $\dot\vz(t) + \frac{\gamma}{2}\ddot\vz(t) = - 5\alpha  \cdot F(\vz(t)) + 10 \alpha \gamma  \cdot \,JF\, (\vz(t)) \cdot F(\vz(t))$ \\
         \multicolumn{1}{c}{$\vdots$} & \multicolumn{1}{c}{$\vdots$} \\
         LA$k$-GDA & 
         $\dot\vz(t) + \frac{\gamma}{2}\ddot\vz(t) = - k \alpha  \cdot F(\vz(t)) + (\sum_{i=1}^{k-1} i)\cdot \alpha \gamma  \cdot \,JF(\vz(t)) \cdot  F(\vz(t))$\\[.5em]
    \end{tabular}
    \caption{Summary: HRDEs of LAk-GD.}
    \label{tab:la_gda_hrdes_summary}
\end{table}

\subsection{Analysis of LookAhead Dynamics}
We first prove general convergence of the LookAhead dynamics for the \ref{eq:bilinear_game} in the discrete domain. We write the phase-space representation of the differential equation, then enforce convergence on the eigenvalues of the coefficient matrix using the Routh-Hurwitz criterion, as introduced in ~Section \ref{sec:prelim}.
\label{sec:theorem1-proof}

\subsection{Convergence Proof of LookAhead for Bilinear Game (Time Domain)} 

In this section, we prove convergence of the ~\ref{eq:la-hrde} for the \ref{eq:bilinear_game} problem.

\begin{theorem}
    The $O(\gamma)$ HRDE of LookAhead($k, \alpha$) (\ref{eq:la-hrde}) converges to the Nash equilibrium of the bilinear game \ref{eq:bilinear_game} for any step-size $\gamma$.
\end{theorem}

\begin{proof}
For the LA-kGD optimizer, we have the following differential equation
\begin{align*}
    \dot{x}(t) = \omega(t)
\end{align*}
\begin{align*}    
    \dot{\omega}(t) = - \frac{2}{\gamma} \omega(t) - \frac{2k}{\gamma} F(\vz(t)) + \frac{k(k -1)}{\gamma } \alpha \gamma \, JF(\vz(t)) \cdot F(\vz(t)) \,.
\end{align*}

By denoting $\dot{x}(t) = \omega_x(t)$ and $\dot{y}(t) = \omega_y(t)$ we get the following

\begin{align*}
    \begin{bmatrix}
        \dot{x}(t) \\ 
        \dot{y}(t) \\ 
        \dot{\omega}_x(t) \\ 
        \dot{\omega}_y(t) 
    \end{bmatrix} = \underbrace{\begin{bmatrix}
        0 && 0 && \mathbf{I} && 0 \\
        0 && 0 &&  0 && \mathbf{I} \\
        - k (k -1) \alpha \vA^2 && -\dfrac{2k \alpha \vA}{ \gamma} && - \frac{2}{\gamma} \mathbf{I} && 0 \\
        \dfrac{2k \alpha \vA}{\gamma} && - k(k -1) \alpha \vA^2 && 0 && -\frac{2}{\gamma} \mathbf{I} 
    \end{bmatrix}}_{\triangleq \hspace{3pt} \mathcal{C}_{LA-kGD}} \cdot \begin{bmatrix}
        x(t) \\
        y(t) \\
        \omega_x(t) \\
        \omega_y(t)
    \end{bmatrix}
\end{align*}
To obtain the eigenvalues $\lambda \in \mathbf{C}$ of $\mathbf{C}_{LA-kGD}$, we have

$\det(\mathcal{C}_{LA-kGD} - \lambda\mathbf{I})$
\begin{align*}
      &= \det\Bigg( 
    \begin{bmatrix}
         -\lambda\mathbf{I} & 0 & \mathbf{I} & 0 \\
        0 & -\lambda \mathbf{I} &  0 & \mathbf{I} \\
        - k (k -1) \alpha \vA^2 & -\dfrac{2k \alpha \vA}{ \gamma} & - (\frac{2}{\gamma} +\lambda) \mathbf{I} & 0 \\
        \dfrac{2k \alpha \vA}{ \gamma} & - k(k -1) \alpha \vA^2 & 0 & -(\frac{2}{\gamma } + \lambda)  \mathbf{I}
    \end{bmatrix}  
    \Bigg)  \\
    &=  \det \Bigg( 
    \begin{bmatrix}
        \lambda(\frac{2}{\gamma} + \lambda) & 0\\
        0 & \lambda(\frac{2}{\gamma} + \lambda)
    \end{bmatrix} 
    - \underbrace{\begin{bmatrix}
        - k (k -1) \alpha \vA^2& -\dfrac{2k \alpha \vA}{\gamma} \\ 
         \dfrac{2k \alpha \vA}{ \gamma} & - k(k -1) \alpha \vA^2
    \end{bmatrix}}_{\triangleq \mathbf{D}} 
    \Bigg) .
\end{align*}

Let $\mu = \mu_1 + i \mu_2  \in \mathbb{C}$ denote the eigenvalues of $\mathbf{D}$. The characteristic equation becomes
\[
\lambda\left( \frac{2}{\gamma} + \lambda \right) - \mu = 0.
\]
Let $\beta = \frac{2}{\gamma}$. Using the generalized Hurwitz theorem for polynomials with complex coefficients, we construct the following generalized Hurwitz array:

\[
\begin{array}{c|ccc}
\lambda^2 & 1 & 0 & \mu_1 \\
\lambda^1 & \beta & \mu_2 & 0 \\
 & -\mu_2 & \beta \mu_2 & 0 \\
   \lambda^0      & -\mu_2^2 - \beta^2 \mu_1 & 0 & 0 \\
\end{array}
\]

The sign of the last row determines the stability of the polynomial. Since $\beta > 0$, the system is stable if and only if:
\begin{equation}
    \label{eq:stability-equation}
    \mu_1 < -\frac{1}{\beta^2} \mu_2^2 \,.
\end{equation}

Thus, it suffices to show:
\begin{equation}
\label{eq:stability-condition}
\Re(\mu(\vz)) < -\frac{1}{\beta^2} \left( \Im(\mu(\vz)) \right)^2 \,.
\end{equation}

We have that: 
\begin{align*}
    \mu(\vz) = \bar{\vz}^T \, \mathbf{D}\,  \vz &= [\bar{\vx}^T \quad \bar{\vy}^T]\begin{bmatrix}
        - k (k -1) \alpha \vA^2 && -\dfrac{2k \alpha \vA}{ \gamma} \\ 
         \dfrac{2k \alpha \vA}{ \gamma} && - k(k -1) \alpha \vA^2
    \end{bmatrix}\begin{bmatrix}
        \vx \\
        \vy
    \end{bmatrix} \\
    &= -\alpha \, k(k -1) ( \|\vA \vx\|_2^2+ \|\vA \vy\|_2^2) + \alpha \, k  (\frac{\bar{\vy}^T \vA \vx}{ \gamma} - \frac{\bar{\vx}^T\vA \vy}{\gamma}) \,
    \\
    &=\underbrace{-\alpha \, k(k -1) ( \|\vA \vx\|_2^2 + \|\vA \vy\|_2^2)}_{\Re(\mu(\vz))} + \underbrace{\alpha \, k ( \beta \, \Im(\bar{\vx}^\top \vA \vy))}_{\Im(\mu(\vz)))} \cdot i \,.
\end{align*}

where the last equations follows from the fact that $\bar{\vx}^T\vA \vy$ is the complex conjugate of $\bar{\vy}^T \vA \vx$, hence $\bar{\vy}^T \vA \vx - \bar{\vx}^T\vA \vy = 2\Im(\bar{\vx}^T \vA \vy) \cdot i $. Hence, from \ref{eq:stability-condition}, we need to show that 
\begin{align*}
    &-k(k-1) \alpha (\|\vA \vx\|_2^2 \, + \, \|\vA \vy\|_2^2) < -\frac{1}{\beta^2} k^2 \alpha^2 \big(\beta \, \Im(\| \bar{\vx}^\top \vA \vy\| )\big)^2\\
    \implies & \quad  \|\vA \vx\|_2^2 \, + \, \|\vA \vy\|_2^2 > \frac{\alpha k}{k-1}  \,  \| \bar{\vx}^\top \vA \vy\|^2
    \end{align*}

It is straightforward to see that $\frac{k}{k-1 } = 1 + \frac{1}{k-1} < 2 \implies \frac{\alpha k}{k-1} < 2$.

Hence we show that
\begin{align*}
     \|\vA \vx\|_2^2 \, + \, \|\vA \vy\|_2^2 > 2  \,  \| \bar{\vx}^\top \vA \vy\|^2
\end{align*}

Consider the case $\|\vx\|_2 \leq \|\vy\|_2 \leq 1$. We have two sub-cases.

\medskip
\textbf{Case 1:} $\|\vA^\top \vx\|_2^2 \leq \|\vA\vy\|_2^2$.  
We can set $\|\vA\vy\|_2 = \|\vy\|_2$, and we have:
\begin{align}
\|\vA^\top \vx\|_2^2 + \|\vA^\top \vy\|_2^2 
&= \|\vA^\top \vx\|_2^2 + \|\vy\|_2^2 \nonumber \\
&\geq \tfrac{1}{2}\big(\|\vA^\top \vx\|_2^2 + \|\vy\|_2^2\big)^2 \nonumber \\
&\geq 2\|\vA^\top \vx\|_2^2 \, \|\vy\|_2^2 \nonumber \\
&\geq 2|\bar{\vx}^\top \vA\vy|^2, \label{eq:case1}
\end{align}
where the last inequality follows from the Cauchy--Schwarz inequality.

\medskip
\textbf{Case 2:} $\|\vA\vy\|_2^2 \leq \|\vA^\top \vx\|_2^2$.  
We can set $\|\vA^\top \vx\|_2 = \|\vx\|_2$, and we have:
\begin{align}
\|\vA^\top \vx\|_2^2 + \|\vA^\top \vy\|_2^2 
&= \|\vx\|_2^2 + \|\vA\vy\|_2^2 \nonumber \\
&\geq \tfrac{1}{2}\big(\|\vx\|_2^2 + \|\vA\vy\|_2^2\big)^2 \nonumber \\
&\geq 2\|\vx\|_2^2 \, \|\vA\vy\|_2^2 \nonumber \\
&\geq 2|\bar{\vx}^\top \vA\vy|^2, \label{eq:case2}
\end{align}
where the last inequality follows from the Cauchy--Schwarz inequality.

\medskip
The case $\|\vy\|_2 \leq \|\vx\|_2 \leq 1$ can be shown analogously.

Hence \ref{eq:stability-condition} holds and the system is stable. 
\end{proof}

However, we do not get any insights on choosing the values of the hyperparameters to ensure convergence. We discuss this further in the next subsection.

\subsection{Convergence Proof of LookAhead for Bilinear Game (Frequency Domain)}

We first present the proof of Theorem \ref{theo:2} by taking the inverse Laplace Transform (\ref{inverse-laplace}) of the frequency dual of \ref{eq:la-hrde}. We then perform a similar analysis on the poles of the frequency dual as described in \ref{app:gda_convergence_analysis} to find the convergence criteria for LookAhead.
\begin{theoremrep}[Restatement of Theorem 4]
    Consider the~\ref{eq:bilinear_game} (with matrix $\mA$). Let $U$ be the orthogonal matrix from eigen decomposition of $\mA$ i.e., $\mA = \mU \mLambda \mU^\intercal$, where $\mLambda = \mathrm{diag}(\lambda_i)$. The trajectory of the individual players $\vx(t)$ and $\vy(t)$ of the~\eqref{eq:la-hrde} continuous time dynamics with parameters $(k, \alpha)$ and Gradient Descent with step size $\gamma$ as the base optimizer, are as follows:
\begin{align}
    \vx(t) &= -\frac{2k\alpha}{\gamma} (\mG \ast \vy)(t) + \mU \mD_x(t) \mU^\intercal \tag{$x$-Sol} \\
    \vy(t) &= \frac{2k\alpha}{\gamma} (\mG * \vx)(t) + \mU \mD_y(t) \mU^\intercal  \tag{$y$-Sol}
\end{align}
where for a player $\mathbf{q} \in \{\vx, \vy\}$, $\mD_q(t)$ is a diagonal matrix whose $i$-th diagonal element is:
\vspace{-.5em}
\begin{align*}
    \big(\mD_q(t)\big)_{ii} = e^{-\frac{t}{\gamma}} \bigg[ \cosh(\omega_i t) &\mathbf{q}_i(0) \\
    &\hspace{-2.5em} + \frac{\sinh(\omega_i t)}{\omega_i} \Big( \dot{\mathbf{q}}_i(0) + \frac{\mathbf{q}_i(0)}{\gamma} \Big) \bigg] \,.
    \vspace{-.5em}
\end{align*}
Here, $*$ is the convolution operator, and we define
\vspace{-.5em}
\begin{align*}
    \omega_i &= \sqrt{\frac{1}{\gamma^2} - \alpha k(k-1) \lambda_i}, \\
    \mG(t) &= \mU \, \mathrm{diag}\left( \frac{e^{-t/\gamma} \sinh(\omega_i t)}{\omega_i} \right) \mU^\intercal \,,
    \vspace{-.5em}
\end{align*}
where $(\vx(0), \vy(0))$ and $(\dot{\vx}(0), \dot{\vy}(0))$ are the initial positions and momenta, respectively.
\end{theoremrep}

\begin{proof}[Proof of Theorem \ref{theo:2}.]
We compute similarly, the gradient field and the Jacobian of the bilinear game \eqref{eq:bilinear_game}. 
\begin{equation}
    \vz(t) = \begin{bmatrix}
        \vx(t) \\
        \vy(t)
    \end{bmatrix}
\end{equation}

\begin{equation}
    F(\vz(t)) = \begin{bmatrix}
    \vA \vy(t) \\
    -\vA \vx(t)\end{bmatrix}
\end{equation}
\begin{equation}
\label{eq:operator}
    JF(\vz(t))  = \begin{bmatrix}
        0 && \vA  \\
        -\vA && 0
    \end{bmatrix}
\end{equation}
\begin{equation}
\label{eq:product}
    JF(\vz(t)) \cdot F(\vz(t)) = \begin{bmatrix}
       -\vA^2 \vy(t)\\
         -\vA^2 \vx(t)
    \end{bmatrix} \,.
\end{equation}

We rewrite the  \ref{eq:la-hrde} for the 2 players using the equations \ref{eq:operator} and \ref{eq:product}
\begin{equation}
\begin{split}
\label{eq:substituted-hrde-x}
        \ddot{\vx}(t) &= -\frac{2}{\gamma} \dot{\vx}(t) - \frac{2k\alpha}{\gamma} \vA \vy(t) - \alpha k(k-1) \vA^2 \vx(t)  
\end{split}
\end{equation}
\begin{equation}
\begin{split}
\label{eq:substituted-hrde-y}
    \ddot{\vy}(t) &= -\frac{2}{\gamma} \dot{\vy}(t) + \frac{2k\alpha}{\gamma} \vA \vx(t) - 2\alpha k(k-1) \vA^2 \vy(t)
\end{split} \,.
\end{equation}

Taking the Laplace transform of \ref{eq:substituted-hrde-x} and \ref{eq:substituted-hrde-y} yields
\begin{equation}
       \left(s^2 \vI + \frac{2}{\gamma} s \vI + \alpha k(k-1)  \vA^2\right) \vX(s) + \frac{2k\alpha}{\gamma} \vA \vY(s) = \frac{2}{\gamma} \vx(0) + s\vx(0) + \dot{\vx}(0)
        \label{eq:freq_X}
\end{equation}
\begin{equation}
         -\frac{2k\alpha}{\gamma} \vA \vX(s) + \left(s^2 \vI + \frac{2}{\gamma} s \vI + \alpha k(k-1)  \vA^2\right) \vY(s) = \frac{2}{\gamma} \vy(0) + s\vy(0) + \dot{\vy}(0) \,.
        \label{eq:freq_Y} 
\end{equation}

Taking the Inverse Laplace transform (\ref{inverse-laplace}) of the above yields the solution equations in \ref{eq:trajectory_1}.

We solve the system \eqref{eq:freq_X}--\eqref{eq:freq_Y} by applying the eigen-decomposition $\vA = \mU \mLambda \mU^\intercal$ with $\mLambda = \mathrm{diag}(\lambda_i)$ and orthogonal $\mU$. Left-multiplying by $\mU^\intercal$ and defining $\vx(s) = \mU \, \vx_{\text{eig}}(s)$ and $\vy(s) = \mU \, \vy_{\text{eig}}(s)$, the system decouples along each eigendirection $i$:
\begin{align}
\left(s^2 + \frac{2}{\gamma} s + \alpha k(k{-}1) \lambda_i^2 \right) x_i(s) + \frac{2k\alpha}{\gamma} \lambda_i y_i(s) &= s x_i(0) + \dot{x}_i(0) + \frac{2}{\gamma} x_i(0), \\
-\frac{2k\alpha}{\gamma} \lambda_i x_i(s) + \left(s^2 + \frac{2}{\gamma} s + \alpha k(k{-}1) \lambda_i^2 \right) y_i(s) &= s y_i(0) + \dot{y}_i(0) + \frac{2}{\gamma} y_i(0)\,.
\end{align}

Solving this $2 \times 2$ system gives expressions of the form:
\[
x_i(s),\ y_i(s) = \frac{\text{linear in } s}{(s + \frac{1}{\gamma})^2 - \omega_i^2}, \quad \text{where } \omega_i^2 = \alpha k(k{-}1)\lambda_i^2 - \frac{1}{\gamma^2}\,.
\]

Applying the standard Laplace inverses:
\begin{align*}
\mathcal{L}^{-1} \left\{ \frac{1}{(s + \frac{1}{\gamma})^2 - \omega_i^2} \right\} &= e^{-t/\gamma} \frac{\sinh(\omega_i t)}{\omega_i}, \\
\mathcal{L}^{-1} \left\{ \frac{s + \frac{1}{\gamma}}{(s + \frac{1}{\gamma})^2 - \omega_i^2} \right\} &= e^{-t/\gamma} \cosh(\omega_i t) \,.
\end{align*}

We get the following time-domain solutions:
\[
\vx_i(t) = e^{-t/\gamma} \left[ \cosh(\omega_i t)\, \vx_i(0) + \frac{1}{\omega_i} \sinh(\omega_i t) \left( \dot{\vx}_i(0) + \frac{1}{\gamma} \vx_i(0) \right) \right] - \frac{2k\alpha}{\gamma} (g_i * \vy_i)(t) \,,
\]
\[
\vy_i(t) = e^{-t/\gamma} \left[ \cosh(\omega_i t)\, \vy_i(0) + \frac{1}{\omega_i} \sinh(\omega_i t) \left( \dot{\vy}_i(0) + \frac{1}{\gamma} \vy_i(0) \right) \right] + \frac{2k\alpha}{\gamma} (g_i * \vx_i)(t) \,,
\]
where $g_i(t) = e^{-t/\gamma} \frac{\sinh(\omega_i t)}{\omega_i}$, $\mG(t) = \mU\, \mathrm{diag}(g_i(t))\, \mU^\intercal$, and $*$ denotes convolution. 

Finally, transforming back to the original coordinates gives the full vector solution:
\begin{align}
\vx(t) &= -\frac{2k\alpha}{\gamma} (\mG * \vy)(t) + \mU\, \mathrm{diag}\left( e^{-t/\gamma} \left[ \cosh(\omega_i t)\, \vx_i(0) + \frac{1}{\omega_i} \sinh(\omega_i t) \left( \dot{\vx}_i(0) + \frac{1}{\gamma} \vx_i(0) \right) \right] \right) \mU^\intercal \,, \\
\vy(t) &= \frac{2k\alpha}{\gamma} (\mG * \vx)(t) + \mU\, \mathrm{diag}\left( e^{-t/\gamma} \left[ \cosh(\omega_i t)\, \vy_i(0) + \frac{1}{\omega_i} \sinh(\omega_i t) \left( \dot{\vy}_i(0) + \frac{1}{\gamma} \vy_i(0) \right) \right] \right) \mU^\intercal \,.
\end{align}
Setting $\big(\mD_q(t)\big)_{ii} = e^{-\frac{t}{\gamma}} \bigg[ \cosh(\omega_i t) \mathbf{q}_i(0) + \frac{\sinh(\omega_i t)}{\omega_i} \Big( \dot{\mathbf{q}}_i(0) + \frac{\mathbf{q}_i(0)}{\gamma} \Big) \bigg] $, where $\mathbf{D}_q (t)$ is a diagonal matrix and $\mathbf{q} \in \{\vx, \vy\}$ is the corresponding player,  yields the statement of the theorem.
\end{proof}

We remark that, for a purely potential game, the terms of $\vY(s)$ cancel out and simplify the general \ref{eq:trajectory_1} to:
\begin{equation*}
     \vx(t) = \hspace{2pt}  \mU \, \mathrm{diag}\left( 
  e^{-\frac{t}{\gamma}} \left[ 
  \cosh(\omega_i t) \, \vx_i(0) + \frac{1}{\omega_i} \sinh(\omega_i t) \left( \dot{\vx}_i(0) + \frac{1}{\gamma} \vx_i(0)
  \right)
  \right]\right) \mU^\intercal  \,.
\end{equation*}
The above affirms that in pure minimization settings, the solution $\vx(t)$ does not depend on $\vy(t)$.

\paragraph{Convergence analysis.}
We substitute \ref{eq:freq_Y} in \ref{eq:freq_X} to get the joint Laplace transform for $\vx(t)$ as follows 
\begin{equation*}
\label{eq:freq_hrde}
\tag{$X$--TF}
\begin{aligned}
\mathbf{X}(s) =\; & \dfrac{-2k\gamma\alpha \vy(0)\, (s+ \frac{2}{\gamma})\vA - 2k\alpha\gamma\, \dot{\vy}(0)\vA}{\gamma^2 \left(s^2 + \frac{2s}{\gamma} + \alpha k(k{-}1) \vA^2\right)^2 + 4k^2\alpha^2} \\
& + \dfrac{\dot{\vx}(0) + \left(s+\frac{2}{\gamma}\right)\vx(0)}{\gamma^2 \left(s^2 + \frac{2s}{\gamma} + \alpha k(k{-}1) \vA^2\right)^2 + 4k^2\alpha^2} \cdot \gamma^2\left(s^2 + \frac{2s}{\gamma} + \alpha k(k{-}1)\right)
\end{aligned}
\end{equation*}
\begin{equation*} \label{eq:freq_hrde-y} \tag{$Y$--TF}
\begin{aligned}
    \mathbf{Y}(s) =\; & \dfrac{2k\gamma\alpha \vx(0) (s+ \frac{2}{\gamma}) \vA +2k\alpha\gamma\dot{\vx}(0) \vA}{\gamma^2 (s^2 +\frac{2s}{\gamma}+\alpha k(k-1) \vA^2)^2 + 4k^2\alpha^2} +\\
& \dfrac{\dot{\vy}(0) + (s+\frac{2}{\gamma})\vy(0)}{\gamma^2 (s^2 +\frac{2s}{\gamma}+ \alpha k(k-1) \vA^2 )^2 + 4k^2\alpha^2} \gamma^2(s^2 +\frac{2s}{\gamma}+\alpha k(k-1)) \,.
\end{aligned} 
\end{equation*}

The characteristic equation is then: 
\begin{equation}
    \left( s^2 + \frac{2}{\gamma}s + \alpha k(k-1) \vA^2 \right)^2 + \left( \frac{2k\alpha}{\gamma} \vA \right)^2 = 0 \,.
\end{equation}
We use the Routh-Hurwitz criterion to analyze the coefficients of the characteristic equations. We arrive at the following convergence condition (using similar steps as in \ref{app:gda_convergence_analysis})
\begin{equation}
    \label{eq:convergence-condition}
    \tag{BG-Cond}
    \alpha < \frac{k-1}{k} \,.
\end{equation}

\paragraph{Discussion.}
The convergence condition for the continuous-time formulation of LookAhead (\ref{eq:la-hrde}) matches that derived in the discrete case (Lemma~\ref{lem:conv-cond}, Appendix~\ref{app:proof-lemma-2}). The recurring threshold $$\alpha<\tfrac{k-1}{k}$$ 
across both analyses is revealing: it captures the same transition near marginal stability.

In the discrete setting, a Taylor expansion of the response term
$$ \bigl|(1-\alpha)+\alpha(1-ic)^k\bigr|$$ for small $c=\gamma\omega$ shows that stability hinges on controlling the first nontrivial curvature term \eqref{eq:app-conv-cond}. 

The same condition reappears through a $Z$-transform stability analysis \eqref{eq:alpha-max}) for the bilinear game \ref{eq:bilinear_game}.  In continuous time, applying the Routh–Hurwitz criterion to \eqref{eq:convergence-condition} imposes an equivalent sign constraint on the leading coefficients. The two viewpoints are linked by the sampling correspondence $z=e^{s\Delta}\approx 1+s\Delta$: a root crossing the imaginary axis in the $s$–plane corresponds to one crossing the unit circle in the $z$–plane.

Although both perspectives identify the same stability boundary, they highlight complementary aspects. The continuous-time analysis is conceptually simpler: rescaling time absorbs the step size $\gamma$, and stability reduces to positioning all roots in the left half-plane, governed solely by $(k,\alpha)$. 

This view clarifies qualitative behavior—for instance, why larger $k$ allows $\alpha$ up to $1-\tfrac1k$—and makes the damping–oscillation trade-off geometrically visible in the root geometry. However, this approach is scale-free: it does not specify how the step size $\gamma$ should relate to curvature or sampling frequency.

The discrete-time analysis provides exactly that missing link. Here, the same threshold emerges as the leading condition for maintaining a non-expansive discrete frequency response, coupling the algorithmic parameters with the problem’s scale through
$$
\gamma\,L \;\le\; \Gamma_k^\star(\alpha)\,
$$ 
for an explicit margin $\Gamma_k^\star(\alpha)$. This yields a practical tuning rule: choose $(k,\alpha)$ within the admissible region (often with $\alpha$ just below $1-\tfrac1k$ to retain damping) and then choose 
$$
\gamma \;\le\; \Gamma_k^\star(\alpha)/L \,,
$$ 
To summarize, the continuous analysis identifies \emph{which} damping profiles are structurally stabilizing; while the discrete analysis prescribes \emph{how} to choose parameters that respect the problem’s scale.

 Finally, note that in continuous time, rescaling by $\gamma$ merely reparametrizes time, and preconditioning preserves eigenstructure; hence, the threshold is invariant. In discrete time, high-frequency content ($c=\gamma\omega$ large) is the bottleneck: choices of $(k,\alpha)$ that are harmless in the HRDE can excite grid-scale oscillations if $\gamma$ is too aggressive relative to $L$. Thus, increasing $k$—although it permits larger $\alpha$—does not automatically widen the safe range for $\gamma$; the passband narrows and the worst-case $c$ shifts, and the margin $\Gamma_k^\star(\alpha)$ remains the decisive stability indicator.

\clearpage

\section{Detailed Pseudocode}
In this section, we provide an annotated, more detailed description of the hyperparameter–selection routine in Alg.~\ref{alg:choose-dominant-clear}. 

\begin{algorithm}[!htb]
\DontPrintSemicolon
\caption{ChooseModalParams (called from Alg.~\ref{alg:mola-general})}
\label{alg:choose-dominant-extended}
\SetKwInOut{KwIn}{Input}\SetKwInOut{KwOut}{Output}
\KwIn{Eigenvalues $\Lambda$ of $\nabla F$ (or local linearization); scalars $k_{\min},k_{\max},\gamma$; $\alpha\_\text{grid}\subset(0,1)$}
\KwOut{Chosen $(k^\star,\alpha^\star)$}
\BlankLine
\textbf{Step 1: Per-mode base multipliers.}\;
$T_{\mathrm{all}} \leftarrow 1 - \gamma \cdot \Lambda$ \tcp*{Eigen-values after one base step of GD}\label{line:evc}
\BlankLine
\textbf{Step 2: Dominant mode.}\;
$i_{\mathrm{dom}} \leftarrow \arg\max_{i} \big|T_{\mathrm{all},i}\big|$\;
$T_{\mathrm{dom}} \leftarrow \{\,T_{\mathrm{all},i_{\mathrm{dom}}}\,\}$
\BlankLine
\textbf{Step 3: Initialize ``best'' pair.}\;
$(k^\star,\alpha^\star,\rho^\star) \leftarrow (k_{\min},\ 0.5,\ +\infty)$
\BlankLine
\textbf{Step 4: Search over $\alpha$ and $k$.}\;
\For{each $\alpha \in \alpha\_\text{grid}$}{
  \If{$\neg(0<\alpha<1)$}{\textbf{continue}}
  \tcp{Candidate $k$ values compatible with the dominant mode for current $\alpha$}
  $\mathcal{K} \leftarrow \texttt{KCandidatesForAlpha}(T_{\mathrm{dom}},\,\alpha,\,k_{\min},\,k_{\max})$\;
  \If{$\mathcal{K}=\varnothing$}{\textbf{continue}}
  \For{each $k \in \mathcal{K}$}{
    \tcp{max stable $\alpha$ for this $k$ (here: w.r.t.\ the dominant mode)}
    $\alpha_{\max} \leftarrow \texttt{AlphaCap}(T_{\mathrm{dom}},\,k)$\;
    \If{$\alpha \le \alpha_{\max}$}{
      \tcp{Worst-mode spectral radius under LookAhead $(k,\alpha)$}
      $\rho \leftarrow \max\limits_{\tau \in T_{\mathrm{dom}}} \left| (1-\alpha) + \alpha\, \tau^{\,k} \right|$\;
      \If{$\rho < \rho^\star$}{
        $(k^\star,\alpha^\star,\rho^\star) \leftarrow (k,\alpha,\rho)$\;
      }
    }
  }
}
\BlankLine
\textbf{Step 5: Fallback (no feasible pair).}\;
\If{$\rho^\star = +\infty$}{
  $k^\star \leftarrow k_{\min}$\;
  $\alpha^\star \leftarrow \min\!\big\{0.5,\ \texttt{AlphaCap}(T_{\mathrm{dom}},\,k_{\min})\big\}$\;
}
\KwRet $(k^\star,\alpha^\star)$\;
\BlankLine\BlankLine
\SetKwProg{Fn}{Function}{:}{end}
\Fn{\texttt{AlphaCap}$(T_{\mathrm{stab}},\,k)$}{
  \tcp{Return $\max\{\alpha\in(0,1):\ \max_{\tau\in T_{\mathrm{stab}}}|(1-\alpha)+\alpha\,\tau^k|\le 1\}$}
  \KwRet $\alpha_{\max}(k)$
}
\BlankLine
\Fn{\texttt{KCandidatesForAlpha}$(T_{\mathrm{dom}},\,\alpha,\,k_{\min},\,k_{\max})$}{
  \tcp{Return integer horizons $k\in[k_{\min},k_{\max}]$ that are admissible for the dominant multiplier in $T_{\mathrm{dom}}$ (by the LookAhead-cycle geometry/derivation).}
  \KwRet $\mathcal{K}$
}
\end{algorithm}

Given an operator \(F\) with Jacobian \(\nabla F\), we pass the eigenvalues \(\Lambda\) (exact or local) to the method. We first compute the one–step gradient-descent multipliers \(T_{\mathrm{all}}=1-\gamma\,\Lambda\) (line~\ref{line:evc}). Among these, we identify the dominant multiplier (defined in \ref{eq:dom_mode}) \(T_{\mathrm{dom}}=\{T_{\mathrm{all},i_{\mathrm{dom}}}\}\), as this mode controls convergence (cf. Eq.~\ref{eq:stability}). 

We then initialize the incumbent choice \((k^\star,\alpha^\star,\rho^\star)=(k_{\min},\,0.5,\,\infty)\) and sweep over a user-specified grid \(\alpha_{\text{grid}}\cap(0,1)\). For each \(\alpha\) in this grid, we form the set of admissible horizons \(\mathcal{K}=\texttt{KCandidatesForAlpha}(T_{\mathrm{dom}},\alpha,k_{\min},k_{\max})\) implied by the LookAhead-cycle geometry; if \(\mathcal{K}=\varnothing\), we continue. 
For every \(k\in\mathcal{K}\), we compute the dominant-mode stability cap \(\alpha_{\max}=\texttt{AlphaCap}(T_{\mathrm{dom}},k)\), namely the largest \(\alpha\in(0,1)\) such that \(\max_{\tau\in T_{\mathrm{dom}}}\bigl|(1-\alpha)+\alpha\,\tau^{k}\bigr|\le 1\). If the current \(\alpha\le\alpha_{\max}\), we evaluate the worst-case spectral radius \(\rho=\max_{\tau\in T_{\mathrm{dom}}}\bigl|(1-\alpha)+\alpha\,\tau^{k}\bigr|\) and update \((k^\star,\alpha^\star,\rho^\star)\leftarrow(k,\alpha,\rho)\) whenever \(\rho\) decreases. If no feasible pair is found (i.e., \(\rho^\star=\infty\)), we fall back to \(k^\star=k_{\min}\) and \(\alpha^\star=\min\{0.5,\texttt{AlphaCap}(T_{\mathrm{dom}},k_{\min})\}\). The routine then returns \((k^\star,\alpha^\star)\). 

Here, \texttt{AlphaCap}\((T_{\mathrm{stab}},k)\) returns the dominant-mode stability threshold \(\alpha_{\max}(k)\), and \texttt{KCandidatesForAlpha} enumerates admissible integer horizons \(k\in[k_{\min},k_{\max}]\).

\clearpage
\section{Experimental Details and Supplementary Results}\label{app:experiments}

This section first presents our experimental setup in detail, including the methods' setup, the considered games for the experiments, and the used hyperparameters. It then presents results of additional experiments as well as ablations that complement the main text.

\subsection{Experimental Setup Details}\label{app:exp_setup}

This section provides the details on the experimental setup used to evaluate our proposed method (\textbf{MoLA}) against several baselines. We consider two settings: the \emph{Bilinear Game}~\eqref{eq:bilinear_game} and a \emph{strongly convex–strongly concave} (SC–SC) game.

\paragraph{Methods' setup. } We compare \textbf{MoLA} with standard and widely used VI algorithms, previously defined in Sections~\ref{sec:prelim} and~\ref{app:vi_mds}. The methods and their implementation details are as follows.
\begin{itemize}
    \item \textbf{Gradient Descent (GD):} baseline method; uses simultaneous updates.

    \item \textbf{Extragradient (EG)} and \textbf{Optimistic Gradient (OGD)}: their implementation follows \eqref{eq:extragradient} and \eqref{eq:ogda}, respectively; defined in Appendix~\ref{app:vi_mds}.

    \item \textbf{LookAhead (LA)} with user-specified \(k\) and averaging factor \(\alpha\) (we report results mainly for \(k=40\), \(\alpha=0.5\)). Unless otherwise specified, we use GD as the base optimizer.

    \item \textbf{Modal LookAhead (MoLA):} implemented following Algorithm~\ref{alg:mola-general}. The search range for the hyperparameters is \(k\in[5,2000]\) and $ \alpha \in [0.02, 0.98]$. As in Algorithm~\ref{alg:mola-general}, the selected pair $(k, \alpha)$ is the one that minimizes the modal contraction factor the most. 

    \item \textbf{Adam}: we use two independent PyTorch Adam optimizers, one minimizing in \(\vx\) and one maximizing in \(\vy\). \textbf{LA--Adam} wraps Adam with the same LookAhead (LA) method \((k,\alpha)\) as above.
\end{itemize}
All experiments are run with a fixed random seed, and unless stated otherwise, all methods use the same step size $\gamma$ for fair comparison.

\noindent \textbf{Bilinear Game. } For each run, we sample a random bilinear game matrix
\[ 
    A=\tfrac{\beta}{\sqrt{d}}\,G \,, \quad \text{where} \quad G_{ij}\sim\mathcal{N}(0,1) \,,
\] 
so that  \(\beta\) controls the rotation scale. 
We initialize \(\vx_0,\vy_0\sim\mathcal{N}(0,I_d) * c\) where $c$ is positive integer (shared
across all methods). 
Each method is run for \(T\) \emph{base} iterations, and we log at every base step the Euclidean distance to the saddle point, as well as the wall-clock time in terms of processing time.

\noindent \textbf{Convex-concave Game. } We consider the strongly-convex--strongly-concave (SC--SC) game
\[
f(x,y)=\tfrac{1}{2}\,\vx^\top(\eta_x\, \vI\,)\vx \;-\; \tfrac{1}{2}\,\vy^\top(\eta_y \, \vI\,)\vy \;+\; \vx^\top \vA \vy,
\]
with curvature parameters \(\eta_x>0\) and \(\eta_y>0\). Each run, we generate \(A\) with a controlled spectrum via a prescribed SVD: draw \(U,V\) as orthogonal matrices and set
\(
A=U\,\mathrm{diag}(\sigma_1,\ldots,\sigma_d)\,V^\top
\)
with singular values linearly spaced in \([\sigma_{\min},\sigma_{\max}]\). This construction ensures consistent conditioning while preserving random singular directions. All other experimental details follow the same setup as in the bilinear game.

\subsection{Additional Results}

\paragraph{SCSC balanced instance. } Figure~\ref{fig:scsc_pot} presents an SC–SC instance with a balanced rotation–potential mix. In contrast to Figure~\ref{fig:scsc_rot} (in main part), the reduced rotational component leads to LA lagging behind other first-order baselines. MoLA outperforms all methods, confirming its effectiveness when potentials play a larger role.

\begin{minipage}{ .6\linewidth}
\centering
\makebox[\linewidth]{
  \includegraphics[width=\linewidth]{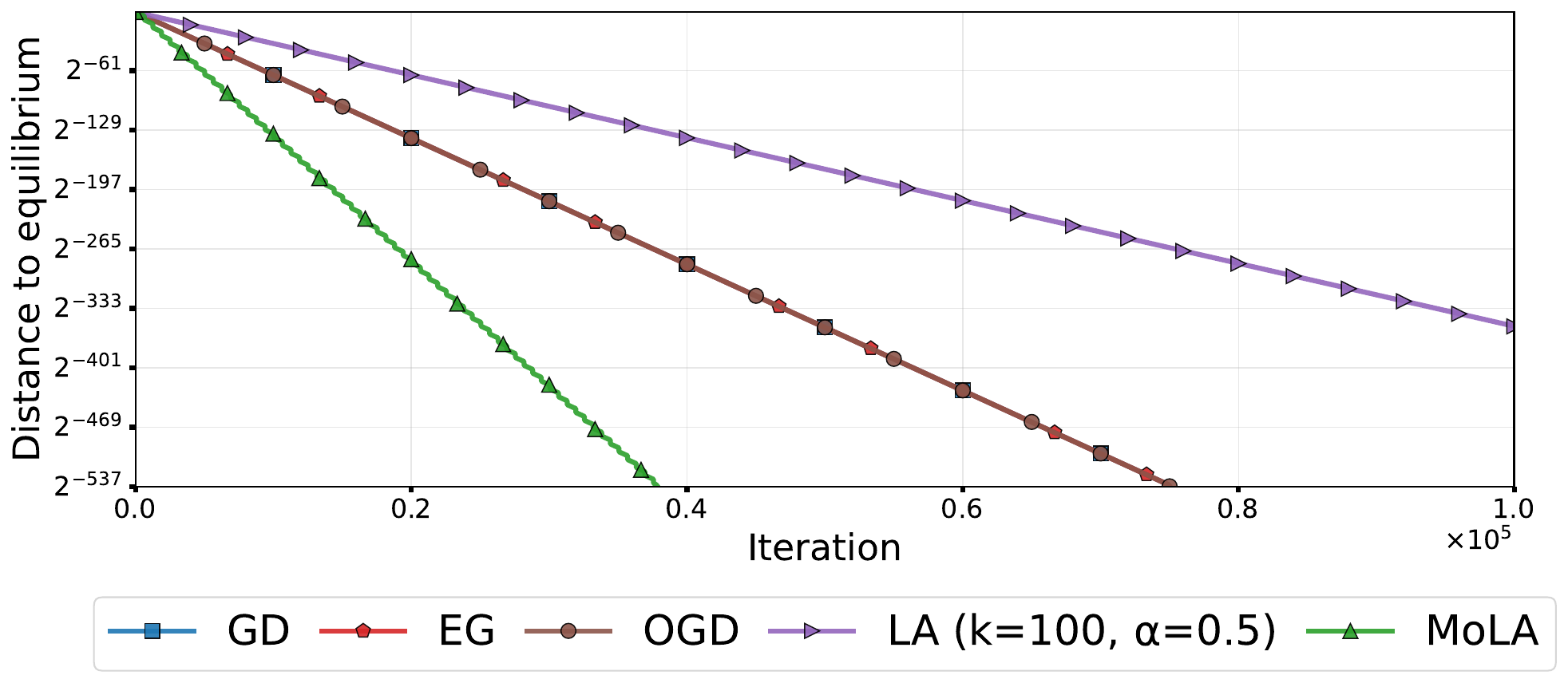}}
\end{minipage}
\begin{minipage}{ .37\linewidth}
\captionof{figure}{Distance to equilibrium vs. iterations for \emph{GD, EG, OGD, LA, and MoLA} in a balanced rotational potential setting of SC-SC game with $d=100, \gamma=0.01$. The GD, EG, and OGD methods overlap. The $x$-axis depicts iteration count, while the $y$-axis depicts the Euclidean distance to the Nash equilibrium. }\label{fig:scsc_pot}
\end{minipage}

\paragraph{GD divergence. } The main results compare MoLA against multiple baselines; gradient descent is omitted because it diverges. For completeness, Figure~\ref{fig:bg_withgd} illustrates this.

\begin{figure}[h]
    \centering
   \begin{subfigure}[t]{.47\linewidth}
    \includegraphics[width=\linewidth]{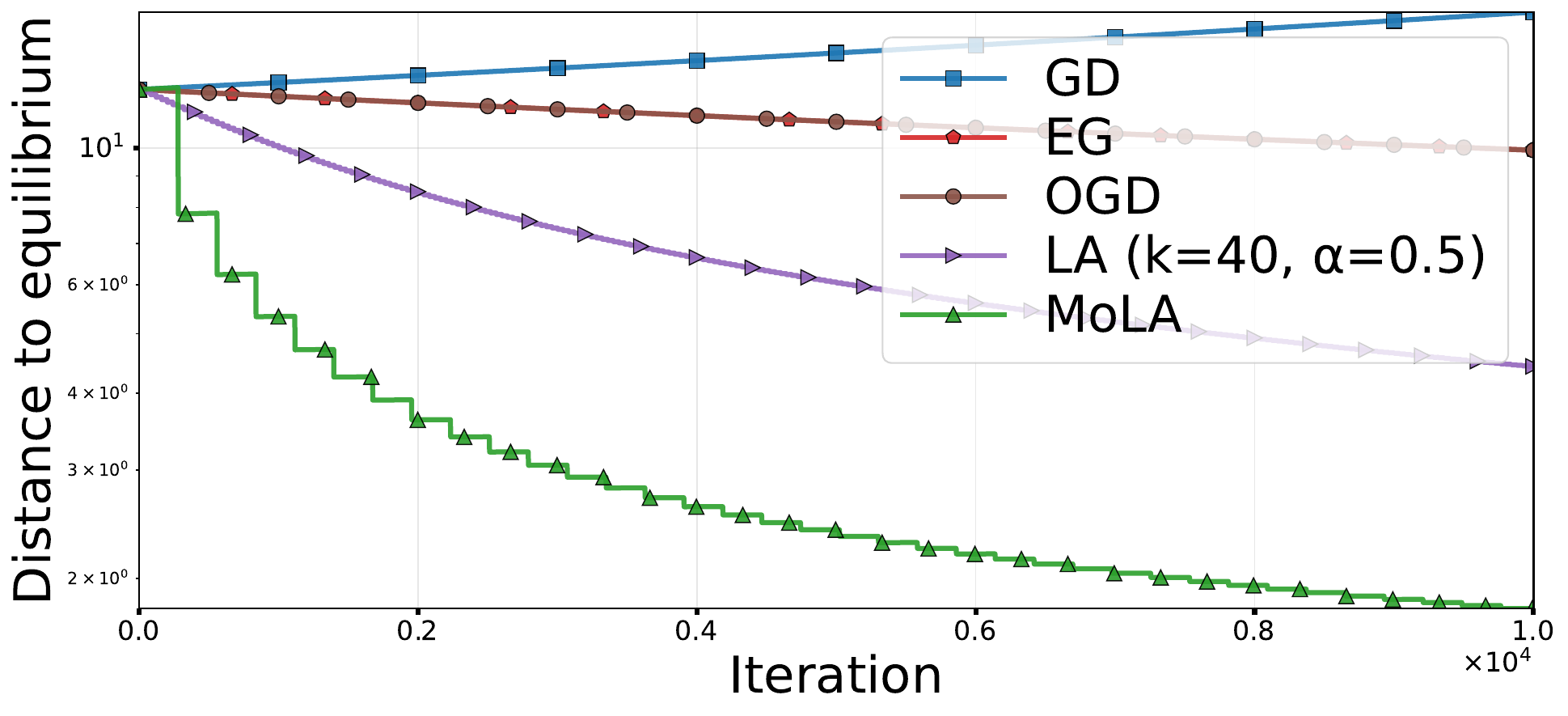}
    \label{fig:sgdplot_iter}
   \end{subfigure}
   \begin{subfigure}[t]{.47\linewidth}
        \includegraphics[width=\linewidth]{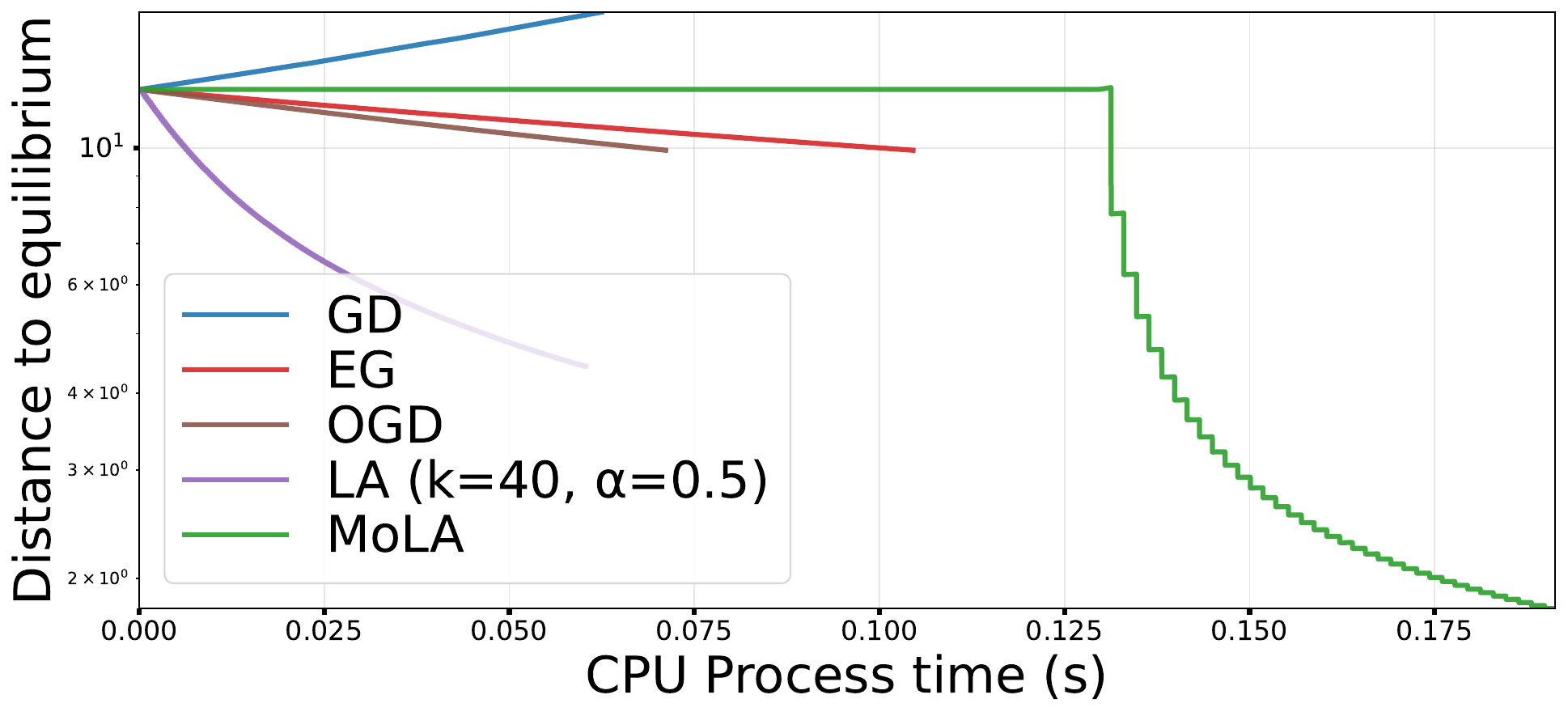}
      \label{fig:sgdplot_wc}
  \end{subfigure}
  \vspace{-.3em}
  \caption{  \textbf{Comparison between \emph{GD, EG, OGD, LA} and \emph{MoLA}, on Bilinear Game.} Figures \textbf{(left)} and \textbf{(right)} show (log-scale) Euclidean distance to equilibrium against iterations and CPU time, respectively.}
  \label{fig:bg_withgd}
\end{figure}

\paragraph{Rotation ablation. } Following the main-text ablation and the Eq.~\eqref{eq:guad_game} setup, Figure~\ref{fig:alpha_values_rot} shows how the LookAhead parameter $\alpha$ chosen by MoLA varies with the amount of rotation. With mild to no rotations, MoLA sets $\alpha=1$, which accelerates convergence by extrapolating to the final iterate; with stronger rotations, it reduces $\alpha$, anchoring updates toward the snapshot to average/contract and damp rotational dynamics. LA
is largely stable for any choice of $\alpha$ for $\beta \in [0, 0.8]$, while decreasing thereafter as the rotation factor increases. See Figure \ref{fig:k_values_rot} of the complementary plot for the trend between $k$ and $\beta$.

\begin{figure}[ht!]
    \centering \includegraphics[width=0.4\linewidth]{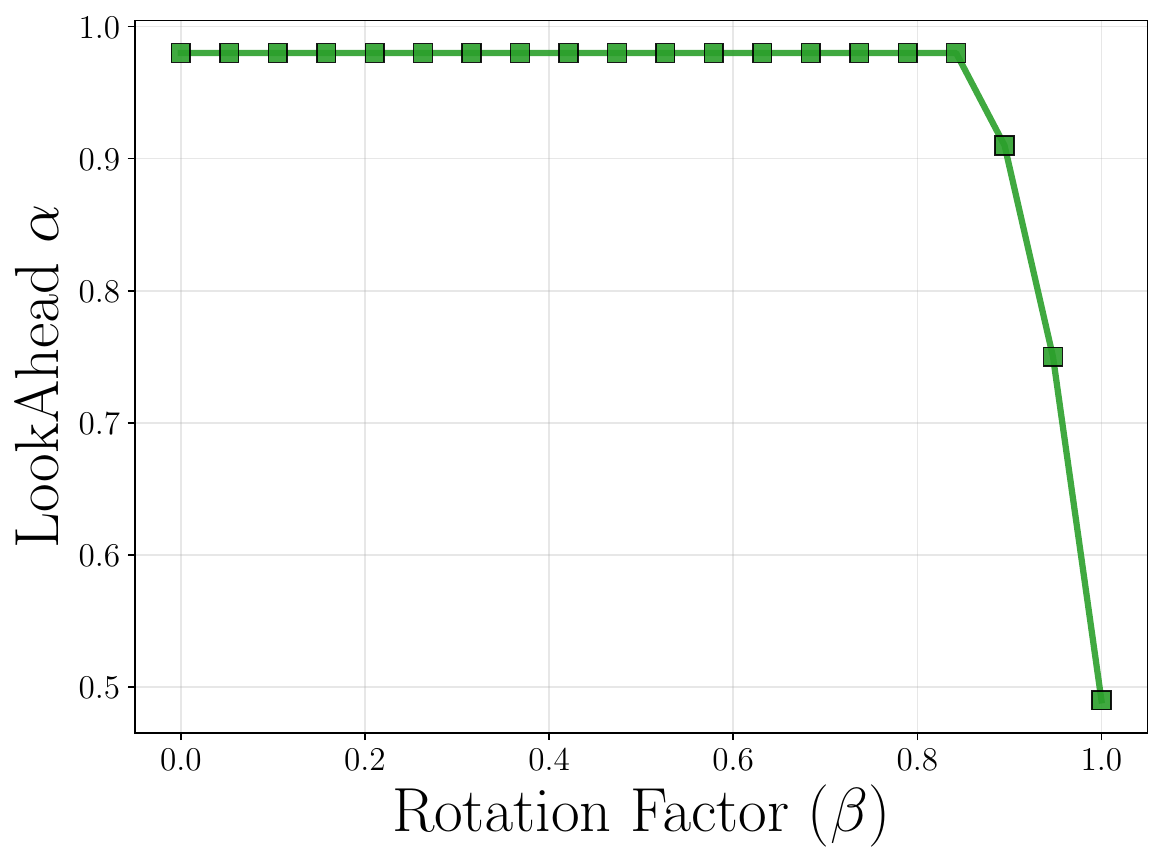}
    \vspace{-.5em}
    \caption{
    \textbf{Optimal LA averaging factor ($\alpha$ vs. rotation factor $\beta$) for the \eqref{eq:guad_game} game.}}
    \label{fig:alpha_values_rot}
    \vspace*{-0.5cm}
\end{figure}

\paragraph{Comparison with Adam.} We compare MoLA against Adam and LA-Adam (LA with Adam as the base optimizer) for the bilinear game. As shown in Figure~\ref{fig:adamplot}, Adam diverges—much like GD—and its added noise further degrades LA-Adam relative to GD-based LookAhead. MoLA outperforms both methods.

\begin{figure}[ht!]
    \centering
   \begin{subfigure}[t]{.47\linewidth}
    \includegraphics[width=\linewidth]{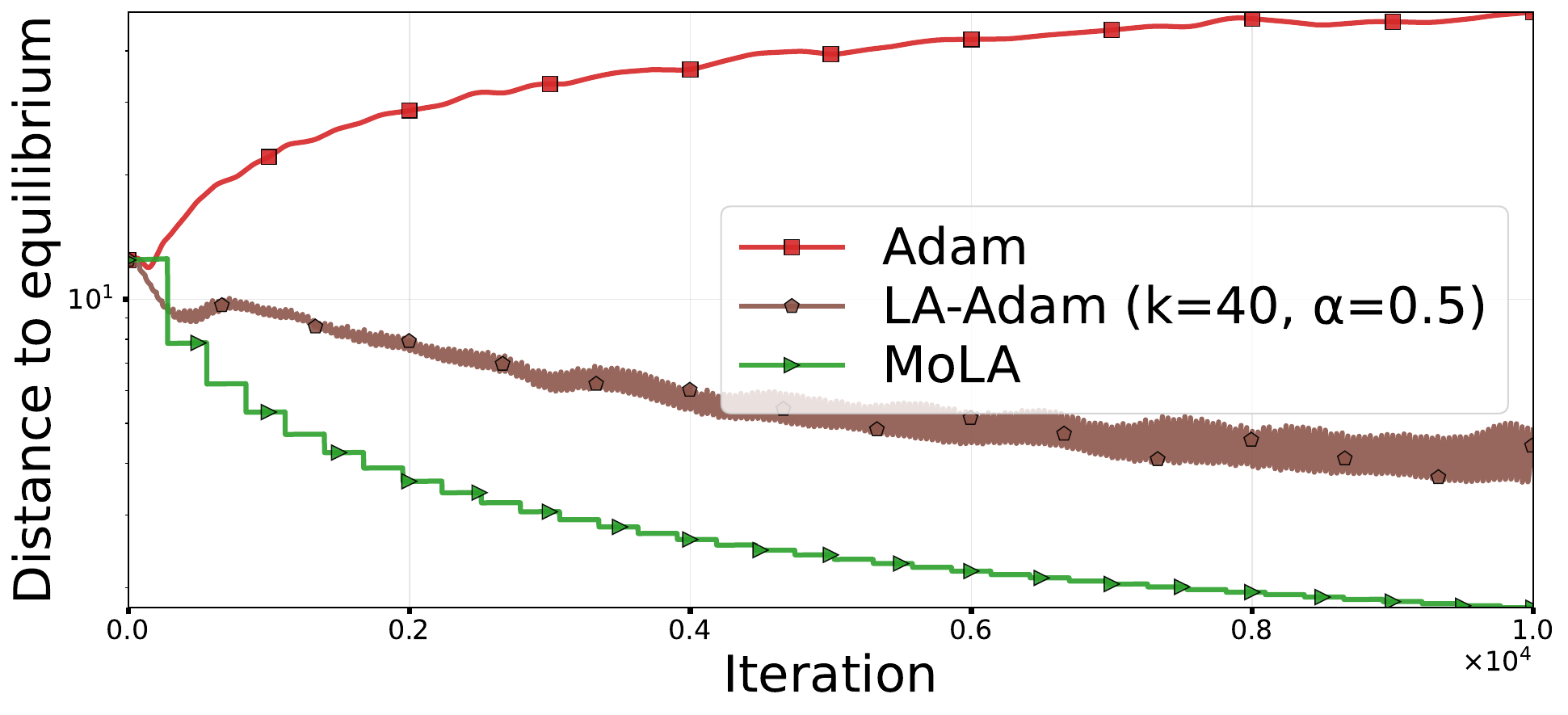}
    \label{fig:adamplot_iter}
   \end{subfigure}
   \begin{subfigure}[t]{.47\linewidth}
        \includegraphics[width=\linewidth]{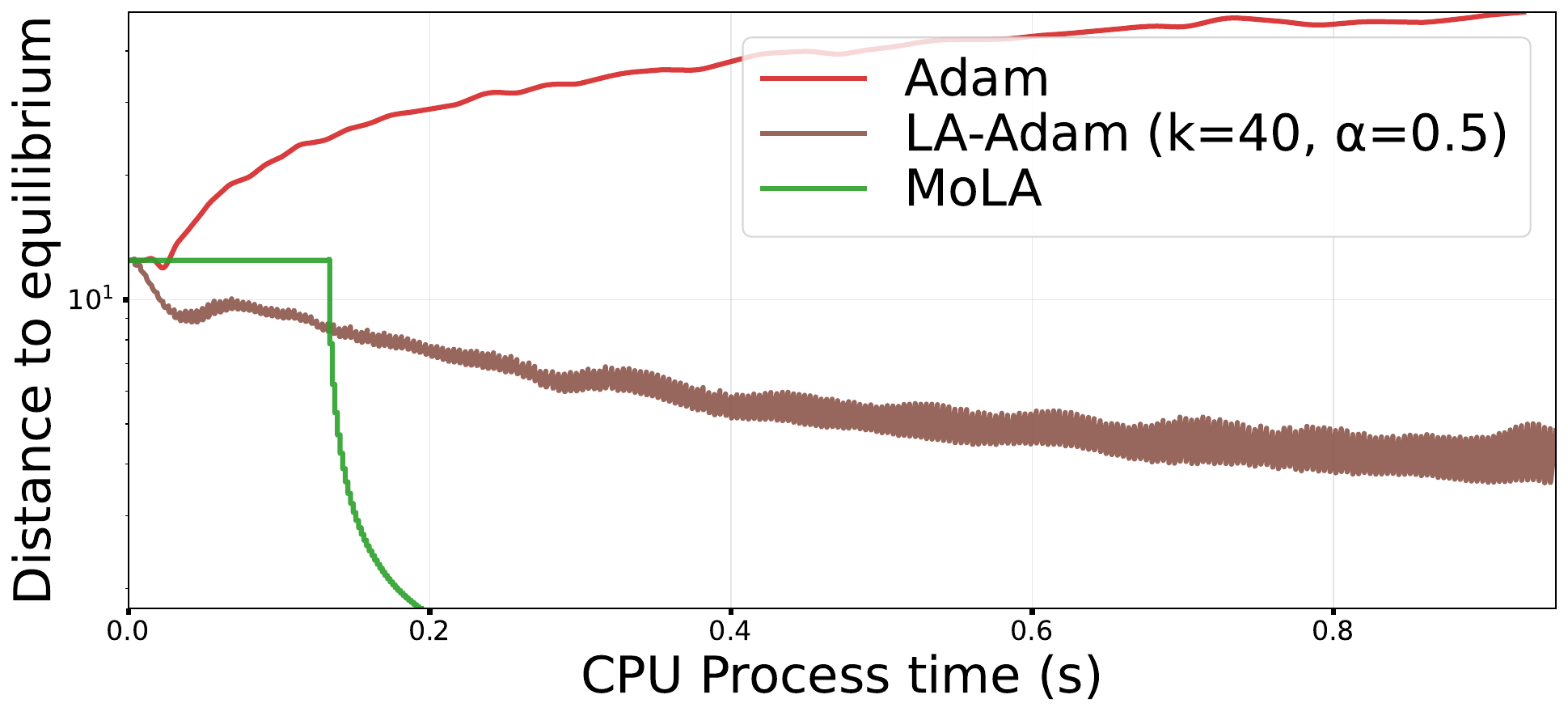}
        \label{fig:adamplot_wc}
  \end{subfigure}
  \vspace{-.3em}
  \caption{  \textbf{Comparison between \emph{Adam, LA-Adam} and \emph{MoLA}, on Bilinear Game.}  Figures \textbf{(left)} and \textbf{(right}) show (log-scale) Euclidean distance to equilibrium against iterations and CPU time, respectively.
}
\label{fig:adamplot}
\end{figure}

\paragraph{Exclusion of non-convergent modes.} As shown in Appendix~\ref{sec:modal-geom}, \textbf{MoLA} automatically selects hyperparameters that keep all dominant modes within the convergent region of the spectrum, thereby avoiding trajectories that drift into the non-convergent half-plane ($Re(z) > 1$). Figure~\ref{fig:modes_illustration} visualizes this effect by contrasting the modal stability of \textbf{LookAhead (LA)} with randomly chosen $k$ values against that of \textbf{MoLA}. Unlike arbitrary configurations, MoLA consistently excludes unstable modes, ensuring that all leading eigenmodes remain within the stability boundary.

\paragraph{Faster convergence of MoLA.} Figure~\ref{fig:traj_illustration} illustrates the trajectories of \textbf{LA} and \textbf{MoLA} on the bilinear game~\eqref{eq:bilinear_game}. MoLA adaptively selects the optimal averaging parameters $(k, \alpha)$ to align with the problem dominant mode, enabling rapid attenuation of oscillations and convergence to the equilibrium. In contrast, LA with arbitrary $(k,\alpha)$ values often exhibits slower convergence due to oscillatory behavior, highlighting MoLA's advantage in dynamically guiding the dynamics toward the saddle point.

\begin{figure}[h]
    \centering
    
   \begin{subfigure}[t]{.49\linewidth}
    \includegraphics[width=\linewidth]{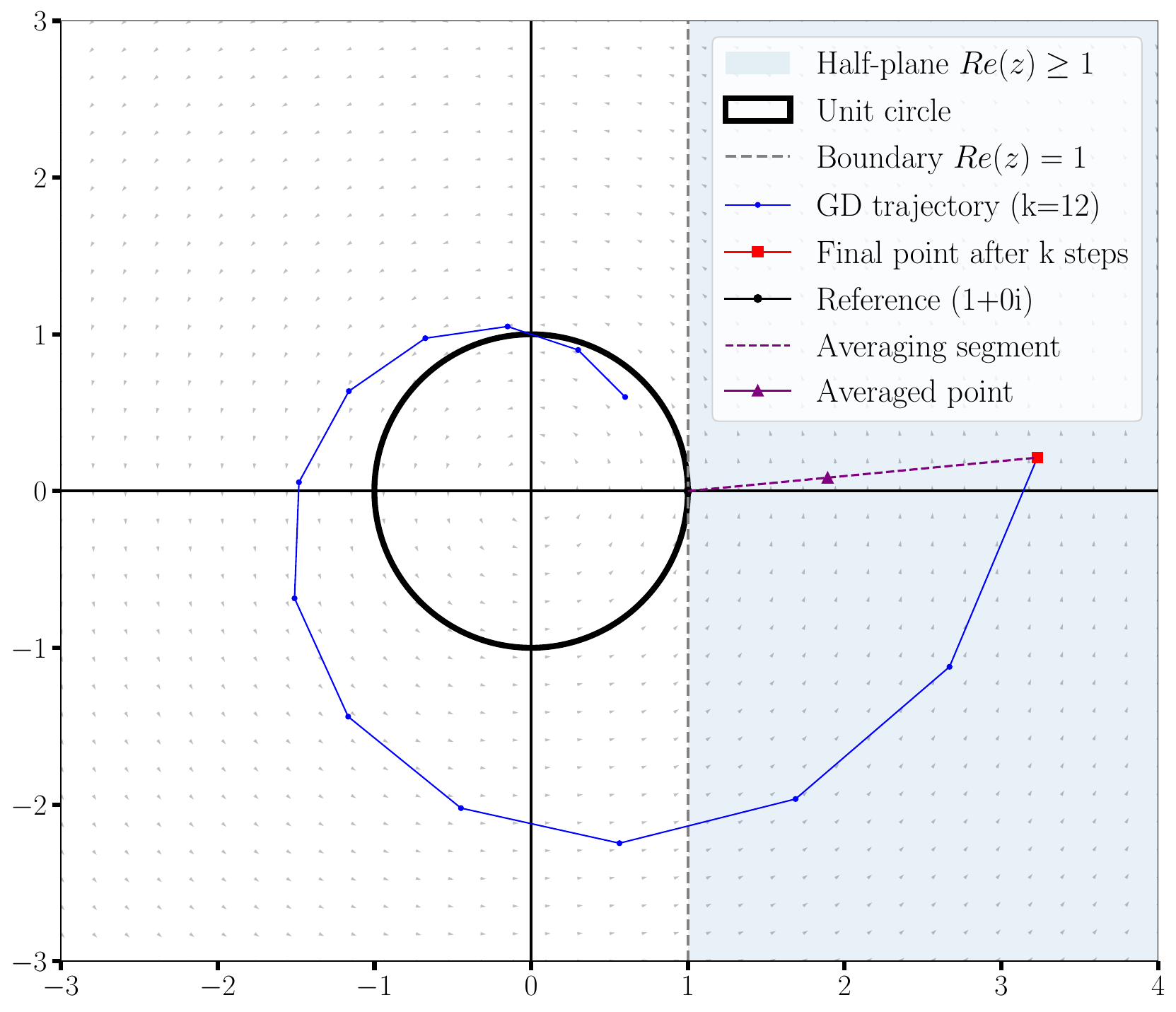}
    \caption{Lookahead}
    \label{fig:illust_bad}
   \end{subfigure}
   \begin{subfigure}[t]{.49\linewidth}
        \includegraphics[width=\linewidth]{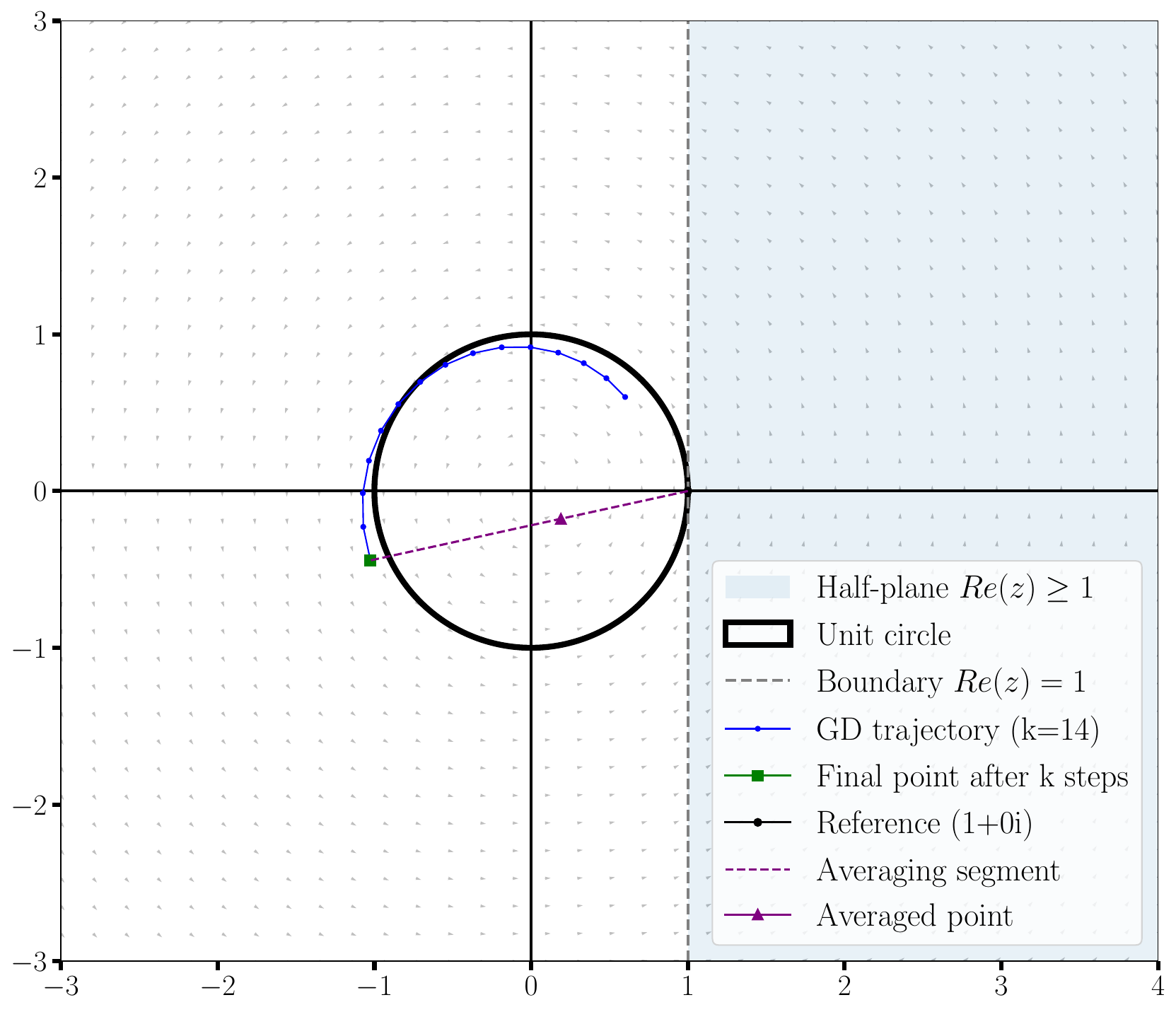}
        \caption{MoLA.}
        \label{fig:modes_safevsbad}
  \end{subfigure}
  \vspace{-.3em}
  \caption{ \textbf{Left.} LA (with randomly chosen hyperparameters) can over-rotate the complex vector (i.e., choosing $k$ gradient steps in a certain range) in a way that leads the final iterate into the complex half-plane $\Re (z) \geq 1$. This is to be avoided, since no choice of $\alpha$ can bring the average iterate within the unit ball. Hence, LA is not stable in general. \textbf{Right.} Since MoLA chooses parameters in line with Lemma \ref{lem:conv-cond}, it avoids this phenomenon. See also Appendix~\ref {sec:modal-geom}, where we show this formally.
}
\label{fig:modes_illustration}
\end{figure}

\begin{figure}
    \centering
    \includegraphics[width=0.5\linewidth]{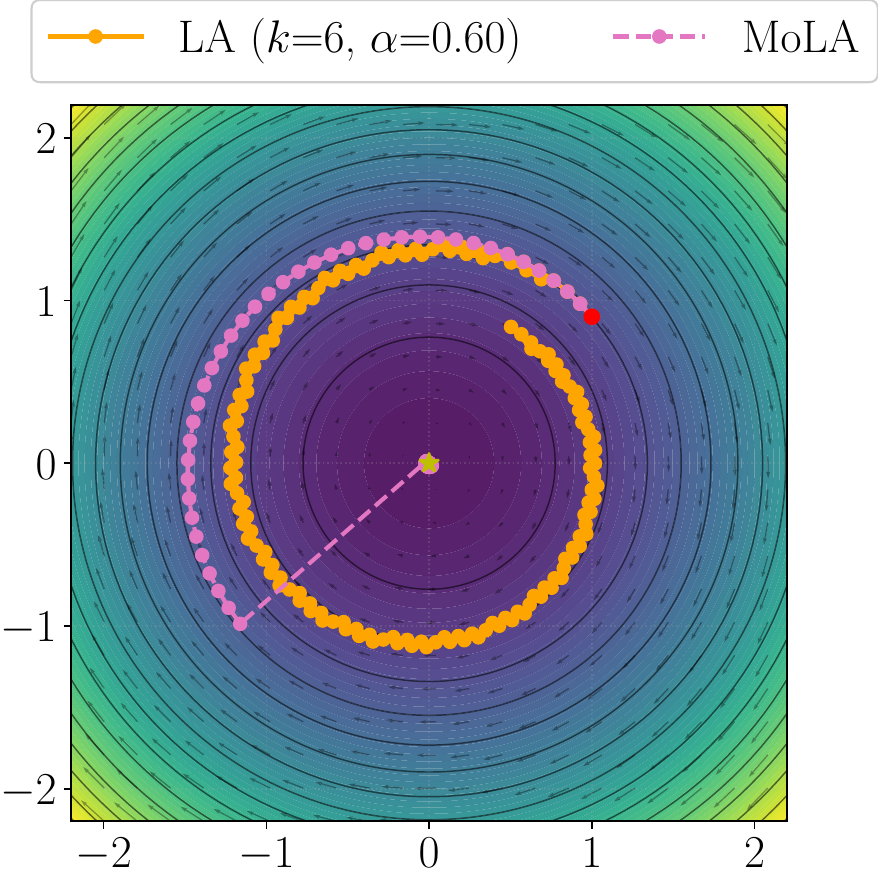}
    \caption{Trajectories of LA and MoLA on~\eqref{eq:bilinear_game}. While LA averages after $k$ steps (where $k$ is randomly selected), MoLA chooses the optimal $k$ steps to maximize contraction. Hence, the final averaged iterate of MoLA converges considerably faster than LA with randomly selected $k$.}
    \label{fig:traj_illustration}
\end{figure}

\end{document}